\RequirePackage{fix-cm}
\documentclass[smallcondensed]{svjour3}     
\usepackage{graphicx}
\usepackage{mathptmx}      
\usepackage{microtype}
\usepackage{natbib}

\usepackage{amsmath}
\usepackage{amssymb}

\usepackage{amsthm}
\usepackage{caption}
\usepackage{hyperref}
\usepackage{url}

\usepackage{array}
\newcolumntype{?}{!{\vrule width 1pt}}

\newtheorem{thm}{Theorem}
\newtheorem{clm}{Claim}
\newtheorem{lem}{Lemma}
\newtheorem{cor}[lem]{Corollary}

\theoremstyle{definition}
\newtheorem{notation}{Notation}
\newtheorem{defn}[notation]{Definition}

\theoremstyle{remark}
\newtheorem{rem}{Remark}

\DeclareMathOperator*{\argmax}{arg\,max}

\DeclareMathOperator{\var}{Var}

\DeclareMathOperator{\D}{d\!}
\DeclareMathOperator{\COMP}{COMP}

\journalname{Machine Learning}
\begin{document}

\title{The Mechanism of Additive Composition
}


\author{Ran Tian         \and
        Naoaki Okazaki  \and
        Kentaro Inui 
}


\institute{
Communication Science Laboratory\\
   Tohoku University\\
   6-6-05 Sendai, Miyagi 980-8579, Japan\\
              Tel.: +81-22-795-7091\\
              Fax: +81-22-795-4285\\
              \email{\{tianran, okazaki, inui\}@ecei.tohoku.ac.jp}           
}

\date{Received: date / Accepted: date}

\maketitle

\begin{abstract}
Additive composition~\citep{foltz98,landauer97,mitchell10} is a 
widely used method for computing meanings of 
phrases, which takes the average of vector representations of the constituent words. 
In this article, we prove an upper bound for the bias of additive composition, which 
is the first theoretical analysis on compositional frameworks from a machine learning point of view. 
The bound is written in terms of collocation strength; we prove that the more exclusively two successive 
words tend to occur together, the more accurate one can guarantee their additive composition as an 
approximation to the natural phrase vector. 
Our proof relies on properties of natural language data that are empirically verified, and can be 
theoretically derived from an assumption that the data is generated from a Hierarchical Pitman-Yor Process. 
The theory endorses additive composition as a reasonable operation for calculating meanings of phrases, 
and suggests ways to improve additive compositionality, including: 
transforming entries of distributional word vectors by a function that meets a specific condition, 
constructing a novel type of vector representations to make additive composition sensitive to word order, 
and utilizing singular value decomposition to train word vectors. 
\keywords{
Compositional Distributional Semantics \and 
bias and variance \and 
approximation error bounds \and 
natural language data \and 
Hierarchical Pitman-Yor Process
}
\end{abstract}

\section{Introduction}
\label{sec:introduction}

The decomposition of generalization errors into bias and variance~\citep{Geman:1992} 
is one of the most profound insights of 
learning theory. Bias is caused by low capacity of models when the training samples are assumed to be 
infinite, whereas variance is caused by overfitting to finite samples. In this article, we apply the analysis 
to a new set of problems in Compositional Distributional Semantics, which studies the calculation of meanings 
of natural language phrases by vector representations of their constituent words. We prove an upper 
bound for the bias of a widely used compositional framework, the 
additive composition~\citep{foltz98,landauer97,mitchell10}. 

Calculations of meanings are fundamental problems in Natural Language Processing (NLP). In recent years, 
vector representations have seen great success at conveying meanings of individual words \citep{levyTACL}. 
These vectors are constructed from statistics of contexts surrounding the words, based on the 
Distributional Hypothesis that words occurring in similar contexts tend to have similar meanings~\citep{harris54}. 
For example, given a target word $t$, one can consider its context as close neighbors of $t$ in a 
corpus, and assess the probability $p^t_{i}$ of the $i$-th word (in a fixed lexicon) occurring in the context of $t$. 
Then, the word $t$ is represented by a vector $\bigl(F(p^t_{i})\bigr)_i$ (where $F$ is some function), 
and words with similar meanings to $t$ will have similar vectors \citep{miller91}. 

Beyond the word level, a naturally following challenge is to represent meanings of phrases or 
even sentences. 
Based on the Distributional Hypothesis, it is generally believed that vectors should be constructed from surrounding contexts, at least for phrases observed in a corpus 
\citep{boleda-et-al:2013:IWCS2013}. 
However, 
a main obstacle here is that phrases are far more sparse than individual words. For example, 
in the British National Corpus (BNC) \citep{bnc}, which consists of 100M word tokens, 
a total of 16K lemmatized words are observed more than 200 times, but there are only 
46K such bigrams, far less than the $16000^2$ possibilities for two-word combinations. 
Be it a larger corpus, one might only observe more rare words due to Zipf's Law, so 
most of the two-word combinations will always be rare or unseen. 
Therefore, 
a direct estimation of the surrounding contexts of a phrase can have large sampling error. 
This partially fuels the motivation to construct phrase vectors from combining word vectors 
\citep{mitchell10}, 
which also bases on the linguistic intuition that meanings of phrases are ``composed'' from 
meanings of their 
constituent words. In view of machine learning, word vectors 
have smaller sampling errors, or lower variance since words are more abundant than phrases. 
Then, a compositional framework which calculates meanings from word vectors 
will be favorable if its bias is also small. 

Here, ``bias'' is the distance between two types of phrase vectors, one calculated from composing 
the vectors of constituent words (composed vector), 
and the other assessed from context statistics where the phrase is treated as a target (natural vector). 
The statistics is assessed from an infinitely large \emph{ideal corpus}, so that the natural vector of the phrase 
can be 
reliably estimated without sampling error, hence conveying the meaning of the phrase by Distributional 
Hypothesis. If the distance between the two vectors is small, the composed vector can be viewed as a 
reasonable approximation of 
the natural vector, hence an approximation of meaning; moreover the composed vector can be more reliably 
estimated from finite \emph{real corpora} because words are more abundant than phrases. 
Therefore, an upper bound for the bias will provide a learning-theoretic support for the composition operation. 

A number of compositional frameworks have been proposed in the literature 
\citep{baroni-zamparelli10,grefenstette-sadrzadeh:2011:EMNLP,socher12,paperno-pham-baroni14,hashimoto-EtAl:2014:EMNLP2014}. Some are complicated methods based on linguistic intuitions~\citep{coecke10}, 
and others 
are compared to human judgments for evaluation~\citep{mitchell10}. However, none of them has been 
previously analyzed regarding their bias\footnote{Unlike natural vectors which always lie in the same space as 
word vectors, some compositional frameworks construct meanings of phrases in different spaces. Nevertheless, 
we argue that even in such cases it is reasonable to require some mappings to a common space, 
because humans can usually compare meanings of a word and a phrase. 
Then, by considering distances between mapped images of composed vectors and natural vectors, 
we can define bias and call for theoretical analysis.}. 
The most widely used framework is the 
additive composition~\citep{foltz98,landauer97}, in which the composed vector is calculated by 
averaging word vectors. Yet, it was unknown if this average is by any means related to 
statistics of contexts surrounding the corresponding phrases. 

In this article, we prove an upper bound for the bias of additive composition of two-word phrases, 
and demonstrate several applications of the theory. An overview is given in Figure~\ref{fig:posting}; 
we summarize as follows.



\begin{figure}[t]
\centering
\includegraphics[scale=0.45,bb=0 0 940 340,clip]{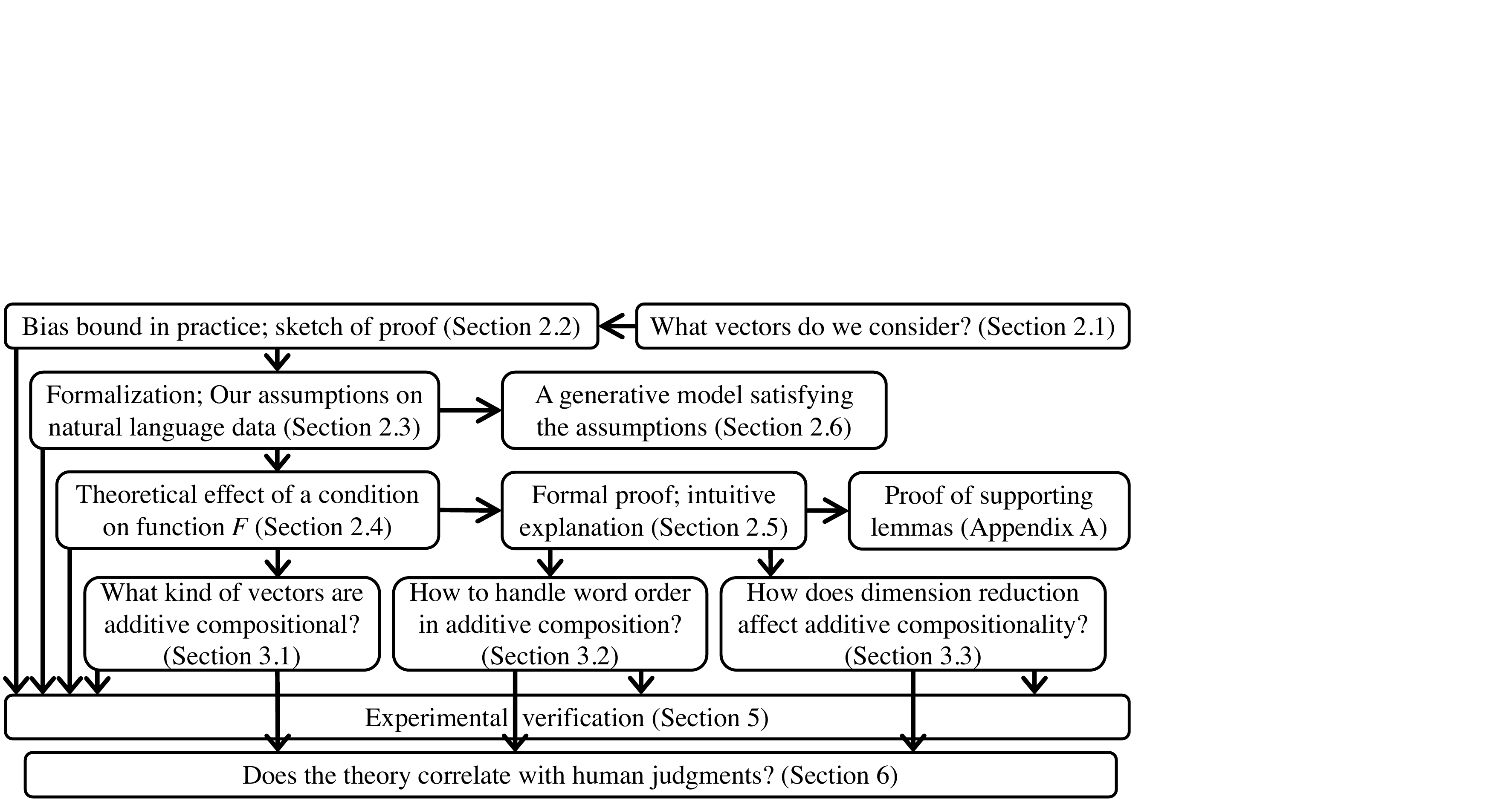}
\caption{An overview of this article. Arrows show dependencies between sections.}
\label{fig:posting}
\end{figure}


In Section~\ref{sec:notation}, we introduce notations and define the vectors we consider in this work. 
Roughly speaking, we use $p^\Upsilon_{i}$ to denote the probability of the $i$-th word in a fixed lexicon 
occurring within a context window of a target (i.e.~a word or phrase) $\Upsilon$, and define the $i$-th 
entry of the \emph{natural vector} as $\mathbf{w}^\Upsilon:=\bigl(c\cdot w^\Upsilon_{i}\bigr)$ and 
$$
w^\Upsilon_{i}:=F(p^\Upsilon_{i}\!+\!1/n)-a^\Upsilon-b_{i}. 
$$
Here, $n$ is the lexicon size, $a^\Upsilon$, $b_{i}$ and $c$ are real numbers and $F$ is a function. 
We note that the formalization is general enough to be compatible with several previous research. 

In Section~\ref{sec:biaspractical}, we describe our bias bound for additive composition, sketch its proof, 
and emphasize its 
practical consequences that can be tested on a natural language corpus. Briefly, we show that 
the more exclusively two successive 
words tend to occur together, the more accurate one can guarantee their additive composition as an 
approximation to the natural phrase vector; but this guarantee comes with one condition that 
$F$ should be a function that decreases steeply around $0$ and grows slowly at $\infty$; and when 
such condition is satisfied, one can derive an additional property that all natural vectors have approximately 
the same norm. 
These consequences are all experimentally verified in Section~\ref{sec:expfunctionF}. 

In Section~\ref{sec:formalization}, we give a formalized version of the bias bound (Theorem~\ref{thm:main}), 
with our assumptions on natural language data clarified. These assumptions include the well-know Zipf's Law, 
a similar law applied to word co-occurrences which we call \emph{the Generalized Zipf's Law}, 
and some intuitively acceptable conditions. The assumptions are experimentally tested in 
Section~\ref{sec:indeptest} and Section~\ref{sec:genzipf}. 
Moreover, we show that the Generalized Zipf's Law can be drived from 
a widely used generative model for natural language (Section~\ref{sec:pitman}). 

In Section~\ref{sec:effectF}, we prove some key lemmas regarding the aforementioned condition on 
function $F$; in Section~\ref{sec:sketch} we formally prove the bias bound (with some supporting lemmas 
proven in  Appendix~\ref{app:fullproof}), 
and further give an intuitive explanation 
for the strength of additive composition: namely, with two words given, the vector of each can be 
decomposed into two parts, one encoding the contexts shared by both words, and the other encoding contexts 
not shared; 
when the two word vectors are added up, the non-shared part of each of them tend to cancel out, 
because non-shared parts have nearly independent distributions; as a result, the shared part gets 
reinforced, which is coincidentally encoded by the natural phrase vector.

Empirically, we demonstrate three applications of our theory:
\begin{enumerate}
\item The condition required to be satisfied by $F$ provides a unified explanation on why some recently 
proposed word vectors are good at additive composition (Section~\ref{sec:functionF}). 
Our experiments 
also verify that the condition drastically affects additive compositionality and other properties of 
vector representations (Section~\ref{sec:expfunctionF}, Section~\ref{sec:exteval}). 
\item Our intuitive explanation inspires a novel method for 
making vectors recognize word order, which was long thought as an issue for additive composition. 
Briefly speaking, since additive composition cancels out non-shared parts of word vectors and reinforces the 
shared one, we show that one can use labels on context words to control what is shared. In this case, 
we propose the \emph{Near-far Context} in which the contexts of \emph{ordered} bigrams are shared 
(Section~\ref{sec:wordorder}). 
Our experiments show that the resulting vectors can indeed assess meaning similarities between ordered 
bigrams (Section~\ref{sec:expnearfar}), and demonstrate strong performance on phrase similarity tasks 
(Section~\ref{sec:phrasesim}). 
Unlike previous approaches, Near-far Context still composes vectors by taking average, 
retaining the merits of being parameter-free and having a bias bound. 
\item Our theory suggests that Singular Value Decomposition (SVD) is suitable for 
preserving additive compositionality in dimension reduction of word vectors 
(Section~\ref{sec:dimensionreduction}). 
Experiments also show that SVD might perform better than other models in additive composition 
(Section~\ref{sec:expdimred}, Section~\ref{sec:exteval}). 
\end{enumerate}

We discuss related works in Section~\ref{sec:relwork} and conclude in Section~\ref{sec:conclusion}. 

\section{Theory}
\label{sec:theory}

In this section, we discuss vector representations constructed from an ideal natural language corpus, and 
establish a mathematical framework for analyzing additive composition. 
Our analysis makes several assumptions on the ideal corpus, which might be approximations or 
oversimplifications of real data. In 
Section~\ref{sec:expveri}, we will test these assumptions on a real corpus and verify that the 
theory still makes reasonable predictions. 

\subsection{Notation and Vector Representation}
\label{sec:notation}

A natural language corpus is a sequence of words. Ideally, we assume that the sequence is infinitely long 
and contains an infinite number of distinct words. 

\begin{notation}
\label{nota:inp}
We consider a finite sample of the infinite ideal corpus. In this sample, we denote the number of 
distinct words by $n$, and use the $n$ words as a lexicon to construct vector representations. 
From the sample, we assess the count $C_i$ of the $i$-th word in the lexicon, and assume that index 
$1\leq i\leq n$ is taken such that $C_{i}\geq C_{i+1}$. Let $C:=\sum_{i=1}^{n}C_{i}$ be the total 
count, and denote $p_{i,n}:=C_i/C$. 
\end{notation}

\begin{table}[t]
\centering
\textit{\ldots as a percentage of your income, your \textbf{tax rate} is generally less than that \ldots}
\setlength{\tabcolsep}{12pt}
\begin{tabular}{| c | l |}
\hline
 target & \multicolumn{1}{ c |}{words in context} \\
\hline
 \textit{tax} & \textit{percentage, of, your, income, your, rate, is, generally, less, than} \\
 \textit{rate} & \textit{of, your, income, your, tax, is, generally, less, than, that} \\
 \textit{tax\_rate} & \textit{of, your, income, your, is, generally, less, than} \\
\hline
\end{tabular}
\caption{Contexts are taken as the closest five words to each side for the targets ``\textit{tax}'' and ``\textit{rate}'', 
and four for the target ``\textit{tax\_rate}''.} 
\label{tab:context}
\end{table}

With a sample corpus given, we can construct vector representations for \emph{targets}, which are either words or 
phrases. To define the vectors one starts from specifying a \emph{context} for each target, which is usually taken as 
words surrounding the target in corpus. As an example, Table~\ref{tab:context} shows a word sequence, 
a phrase target and two word targets; contexts are taken as the closest four or five words to the targets. 

\begin{notation}
\label{nota:targetgen}
We use $s$, $t$ to denote word targets, and $st$ a phrase target consisting 
of two consecutive words $s$ and $t$. When the word order is ignored (i.e., either $st$ or 
$ts$), we denote the target by $\{st\}$. A general target is denoted by $\Upsilon$. 
Later in this article, we will consider other types of targets as well, and a full list of target types is shown 
in Table~\ref{tab:targettypes}.
\end{notation}

\begin{notation}
\label{nota:conp}
Let $C(\Upsilon)$ be the count of target $\Upsilon$, and $C_i^\Upsilon$ the count of $i$-th word co-occurring 
in the context of $\Upsilon$. Denote $p^\Upsilon_{i,n}:=C_i^\Upsilon/C(\Upsilon)$. 
\end{notation}

In order to approximate the ideal corpus, we will take a sample larger and larger, then consider the 
limit. Under this limit, it is obvious that $n\rightarrow\infty$ and $C\rightarrow\infty$. 
Further, we will assume some limit properties on $p_{i,n}$ and $p^\Upsilon_{i,n}$ as specified in 
Section~\ref{sec:formalization}. These properties capture our idealization of an infinitely large natural 
language corpus. In Section~\ref{sec:pitman}, we will show that such properties can be derived from a 
Hierarchical Pitman-Yor Process, a widely used generative model for natural language data. 

\begin{defn}
\label{defn:wvec}
We construct a \emph{natural vector} for $\Upsilon$ from the statistics $p^\Upsilon_{i,n}$ as follows: 
$$
\mathbf{w}^\Upsilon_{n}:=\bigl(c_n\cdot w^\Upsilon_{i,n}\bigr)_{1\leq i\leq n} \quad\text{and}\quad w^\Upsilon_{i,n}:=F(p^\Upsilon_{i,n}\!+\!1/n)-a^\Upsilon_n-b_{i,n}. 
$$
Here, $a^\Upsilon_n$, $b_{i,n}$ and $c_n$ are real numbers and $F$ is a smooth function on $(0,\infty)$. 
The subscript $n$ emphasizes that the vector will change if $n$ becomes larger (i.e.~a larger sample 
corpus is taken). 
The scalar $c_n$ is for normalizing scales of vectors.
In Section~\ref{sec:biaspractical}, we will further specify some conditions on $a^\Upsilon_n$, $b_{i,n}$, $c_n$ 
and $F$, but without much loss of generality. 
\end{defn}

To consider $F(p^\Upsilon_{i,n}\!+\!1/n)$ instead of $F(p^\Upsilon_{i,n})$ can be viewed as a smoothing 
scheme that guarantees 
$F(x)$ being applied to $x>0$. We will consider $F$ that is not continuous at $0$, such as $F(x):=\ln{x}$; 
yet, $w^\Upsilon_{i,n}$ has to be well-defined even if $p^\Upsilon_{i,n}=0$. In practice, 
the $p^\Upsilon_{i,n}$ estimated from a finite corpus can often be $0$; theoretically, 
the smoothing scheme plays a role in our proof as well. 

The definition of $\mathbf{w}^\Upsilon_{n}$ is general enough to cover a wide range of previously proposed 
distributional word vectors. 
For example, if $F(x)=\ln{x}$, $a^\Upsilon_n=0$ and $b_{i,n}=\ln{p_{i,n}}$, then $w^\Upsilon_{i,n}$ is the 
Point-wise Mutual Information (PMI) value that has been widely adopted in 
NLP~\citep{Church:1990,Dagan:1994,Turney:2001,Turney:2010}. More recently, 
the Skip-Gram with Negative Sampling (SGNS) model \citep{word2vecNIPS} is shown to be a matrix factorization 
of the PMI matrix \citep{levyNIPS}; and the more general form of $a^\Upsilon_n$ and $b_{i,n}$ is 
explicitly introduced 
by the GloVe model \citep{pennington-socher-manning14}. Regarding other forms of $F$, it has been reported in 
\citet{HellingerPCA:EACL} and \citet{stratos-collins-hsu:2015:ACL-IJCNLP} that empirically 
$F(x):=\sqrt{x}$ outperforms $F(x):=x$. We will discuss function $F$ further in Section~\ref{sec:functionF}, 
and review some other distributional vectors in Section~\ref{sec:relwork}.

We finish this section by pointing to Table~\ref{tab:notations} for a list of frequently used notations.

\subsection{Practical Meaning of the Bias Bound}
\label{sec:biaspractical}

\begin{table}[t]
\centering
\renewcommand{\arraystretch}{1.2}
\begin{tabular}{ c | c  p{8.2cm} }
\hline
Notation~\ref{nota:targetgen} & $\Upsilon$ & a general target can denote either of the following: \\
\hline
Notation~\ref{nota:targetgen} & $s,t$ & word targets \\
\hline
Notation~\ref{nota:targetgen} & $st$ & two-word phrase target \\
\hline
Notation~\ref{nota:targetgen} & $\{st\}$ & two-word phrase target with word order ignored \\
\hline
Definition~\ref{defn:collo} & $s/t\backslash s$ & a token of word $t$ not next to the word $s$ in corpus \\
\hline
Definition~\ref{defn:nflabel} & $s\bullet,\bullet t$ & words in the context of $s\bullet$ (resp.~$\bullet t$) are 
assigned the left-hand-side (resp.~right-hand-side) Near-far labels \\
\hline
Definition~\ref{defn:nfdecomp} & ${s\bullet}\backslash t,s/{\bullet t}$ & a target $s\bullet$ (resp.~$\bullet t$) not at the left 
(resp.~right) of word $t$ (resp.~$s$) \\
\hline
Theorem~\ref{thm:main} & $S,T$ & random word targets \\
\hline
General & \multicolumn{2}{ c }{random word targets can form different types such as $\{ST\}$ and 
$S/T\backslash S$} \\
\hline
\end{tabular}
\caption{List of target types} 
\label{tab:targettypes}
\end{table}

A compositional framework combines vectors $\mathbf{w}^s_{n}$ and $\mathbf{w}^t_{n}$ to represent the 
meaning of phrase ``\textit{s t}''. In this work, we study relations between this \emph{composed vector} and 
the natural vector $\mathbf{w}^{\{st\}}_{n}$ of the phrase target\footnote{Or it should be $\mathbf{w}^{st}_{n}$ 
if one cares about word order, which we will discuss in Section~\ref{sec:wordorder}.}. More precisely, we study 
the Euclidean distance 
$$\lim_{n\rightarrow\infty} \lVert \mathbf{w}^{\{st\}}_{n}-\COMP(\mathbf{w}^s_{n}, \mathbf{w}^t_{n})\rVert$$
where $\COMP(\cdot,\cdot)$ is the composition operation. If a sample corpus is taken larger and larger, 
we have limit $n\rightarrow\infty$, and $\mathbf{w}^{\{st\}}_{n}$ will be well estimated to represent the 
meaning of ``\textit{s t} or \textit{t s}''.
Therefore, the above distance can be viewed as the bias of approximating 
$\mathbf{w}^{\{st\}}_{n}$ by the composed vector $\COMP(\mathbf{w}^s_{n}, \mathbf{w}^t_{n})$. 
In practice, especially when $\COMP$ is a complicated operation with parameters, 
it has been a widely adopted approach to learn the parameters 
by minimizing the same distances for phrases observed in 
corpus~\citep{dinu-pham-baroni:2013:CVSC,baroni-zamparelli10,guevara:2010:GEMS}. 
These practices further motivate our study on the bias. 

\begin{defn}
\label{defn:error}
We consider \emph{additive composition}, where 
$\COMP(\mathbf{w}^s_{n}, \mathbf{w}^t_{n}):=\frac{1}{2}(\mathbf{w}^s_{n}+\mathbf{w}^t_{n})$ 
is a parameter-free composition operator. 
We define 
$$\mathcal{B}^{\{st\}}_{n}:=\lVert \mathbf{w}^{\{st\}}_{n}-\frac{1}{2}(\mathbf{w}^s_{n}+\mathbf{w}^t_{n})\rVert.$$
\end{defn}

\begin{table}[t]
\centering
\setlength{\tabcolsep}{5pt}
\renewcommand{\arraystretch}{1.6}
\begin{tabular}{ c | c  p{8cm} }
\hline
Notation~\ref{nota:inp} & $i$, $n$ & index $1\leq i \leq n$, where $n$ is the lexicon size \\
\hline
Notation~\ref{nota:inp} & $p_{i,n}$ & empirical probability of the $i$-th word, $p_{i,n}\geq p_{i+1,n}$ \\
\hline\\[-14pt]
Notation~\ref{nota:conp} & $p^\Upsilon_{i,n}$ & {probability of the $i$-th word co-occurring in context of 
$\Upsilon$; defined as $p^\Upsilon_{i,n}:=C_i^\Upsilon/C(\Upsilon)$.} 
\\
\hline\\[-14pt]
Definition~\ref{defn:wvec} & \multicolumn{2}{ c }{$
\mathbf{w}^\Upsilon_{n}:=\bigl(c_n\cdot w^\Upsilon_{i,n}\bigr)_{1\leq i\leq n}, \quad\text{where}\quad w^\Upsilon_{i,n}=F(p^\Upsilon_{i,n}\!+\!1/n)-a^\Upsilon_n-b_{i,n}
$} \\[2pt]
\hline\\[-14pt]
Definition~\ref{defn:error} & \multicolumn{2}{ c }{$\mathcal{B}^{\{st\}}_{n}:=\lVert \mathbf{w}^{\{st\}}_{n}-\frac{1}{2}(\mathbf{w}^s_{n}+\mathbf{w}^t_{n})\rVert$} \\
\hline
Definition~\ref{defn:collo} & $\pi_{s/t\backslash s}$ & probability for an occurrence of $t$ being 
non-neighbor of $s$ \\
\hline
Definition~\ref{defn:cdef} & $\Lambda_n$ & set of observed two-word phrases, word order ignored \\
\hline
General & $\mathbb{E}[\cdot],\var[\cdot]$ & expected value and variance of a random variable \\
\hline
General & $I_{\mathcal{H}}$ & indicator; $I_{\mathcal{H}}=1$ if condition $\mathcal{H}$ is true, 0 otherwise \\
\hline
General & $\mathbb{P}(\mathcal{H})$ & probability of $\mathcal{H}$ being true; 
$\mathbb{P}(\mathcal{H})=\mathbb{E}[I_{\mathcal{H}}]$ \\
\hline
General & $\lambda,\beta,\xi,\ldots$ & lowercase Greek letters denote real constants \\
\hline\\[-14pt]
Theorem~\ref{thm:main} & \multicolumn{2}{ c }{$X:=p^\Upsilon_{i,n}/p_{i,n}$ where $\Upsilon:=$ $\{ST\}$, 
$S/T\backslash S$, or $T/S\backslash T$.} \\
\hline\\[-14pt]
Lemma~\ref{lem:ycalc} & \multicolumn{2}{ c }{$Y_{i,n}:=F(p_{i,n}X\!+\!1/n)-F(p_{i,n}\beta\!+\!1/n)$} \\
\hline\\[-14pt]
Lemma~\ref{lem:ycalc} & \multicolumn{2}{ c }{$\varphi_{i,n}:=(p_{i,n})^{2\lambda}\bigl(1+(\beta np_{i,n})^{-1}\bigr)^{-1+2\lambda}$  } \\
\hline
\end{tabular}
\caption{Frequently used notations and general conventions in this article} 
\label{tab:notations}
\end{table}

Our analysis starts from the observation that, every word in the context of $\{st\}$ also occurs in the contexts of 
$s$ and $t$: as illustrated in Table~\ref{tab:context}, 
if a word token $t$ (e.g.~``\textit{rate}'') comes from a phrase $\{st\}$ (e.g.~``\textit{tax rate}''), 
and if the context window size is not too small, the 
context for this token of $t$ is almost the same as the context of 
$\{st\}$. This motivates us to decompose the context of 
$t$ into two parts, one coming from $\{st\}$ and the other not. 
\begin{defn}
\label{defn:collo}
Define target $s/t\backslash s$ as the tokens of word $t$ which do not occur next to word $s$ in corpus. 
We use $\pi_{s/t\backslash s}$ to denote the probability of $t$ not occurring next to $s$, conditioned 
on a token of word $t$. Practically, $(1-\pi_{s/t\backslash s})$ can be estimated by the count ratio 
$C(\{st\})/C(t)$. 
Then, we have the following equation 
\begin{equation}
\label{eq:collo}
p^t_{i,n} =\pi_{s/t\backslash s}p^{s/t\backslash s}_{i,n} + (1 - \pi_{s/t\backslash s})p^{\{st\}}_{i,n} 
\quad\text{for all $i,n$} 
\end{equation}
because a word in the context of $t$ occurs in the context of either $\{st\}$ or $s/t\backslash s$. 
\end{defn}

We can view $\pi_{s/t\backslash s}$ and $\pi_{t/s\backslash t}$ as indicating how weak the 
``collocation'' between $s$ and $t$ is. 
When $\pi_{s/t\backslash s}$ and $\pi_{t/s\backslash t}$ are small, 
$s$ and $t$ tend to occur next to each other exclusively, so $\mathbf{w}^s_{n}$ and $\mathbf{w}^t_{n}$ are likely 
to correlate with $\mathbf{w}^{\{st\}}_{n}$, making $\mathcal{B}^{\{st\}}_{n}$ small. This is the fundamental idea 
of our bias bound, which estimates 
$\mathcal{B}^{\{st\}}_{n}$ in terms of $\pi_{s/t\backslash s}$ and $\pi_{t/s\backslash t}$. We give a 
detailed sketch below. 
First, by Triangle Inequality one immediately has 
$$
\mathcal{B}^{\{st\}}_{n}\leq\lVert \mathbf{w}^{\{st\}}_{n}\rVert +\frac{1}{2}\bigl(\lVert\mathbf{w}^s_{n}\rVert+\lVert\mathbf{w}^t_{n}\rVert\bigr).
$$
Then, we note that both $\mathcal{B}^{\{st\}}_{n}$ and $\lVert \mathbf{w}^{\Upsilon}_{n}\rVert$ scale with 
$c_n$. Without loss of generality, we can assume that $c_n$ is normalized such that the \emph{average} 
norm of $\mathbf{w}^{\Upsilon}_{n}$ equals $1$. Thus, if we can prove that 
$\lVert \mathbf{w}^{\Upsilon}_{n}\rVert=1$ for \emph{every} target $\Upsilon$, we will have an upper bound 
\begin{equation}
\label{eq:firstbound}
\mathcal{B}^{\{st\}}_{n}\leq 2. 
\end{equation}
This is intuitively obvious because if all vectors lie on the unit sphere, distances between them will be less than 
the diameter 2. 
In this article, we will show that roughly speaking, it is indeed that 
$\lVert \mathbf{w}^{\Upsilon}_{n}\rVert=1$ for ``every'' $\Upsilon$, and the above ``upper bound'' can 
further be strengthened using Equation \eqref{eq:collo}. 

More precisely, we will prove that if a target phrase $ST$ is \emph{randomly} chosen, then 
$\lim\limits_{n\rightarrow\infty}\lVert \mathbf{w}^{\{ST\}}_{n}\rVert$ converges to $1$ in probability. 
The argument is sketched as follows. First, when $ST$ is random, $p^{\{ST\}}_{i,n}$ and $w^{\{ST\}}_{i,n}$ 
become random variables. We assume that for each $i\neq j$, $p^{\{ST\}}_{i,n}$ and $p^{\{ST\}}_{j,n}$ are 
\emph{independent} random variables. Note this assumption in contrast to the fact that 
$p_{i,n}\geq p_{j,n}$ for $i>j$; nonetheless, we assume that 
$p^{\{ST\}}_{i,n}$ is random enough so that when $i$ changes, no obvious relation exists between 
$p^{\{ST\}}_{i,n}$ and $p^{\{ST\}}_{j,n}$. Thus, $\bigl(w^{\{ST\}}_{i,n}\bigr)^2$'s ($1\leq i\leq n$, $n$ fixed) are 
independent and we can apply the Law of Large Numbers:
\begin{equation}
\label{eq:roughlln}
\lim_{n\rightarrow\infty}\frac{\sum_{i=1}^n \bigl(w^{\{ST\}}_{i,n}\bigr)^2}{\sum_{i=1}^n \mathbb{E}\bigl[\bigl(w^{\{ST\}}_{i,n}\bigr)^2\bigr]}=1 \quad\text{in probability}.
\end{equation}
In words, the fluctuations of $\bigl(w^{\{ST\}}_{i,n}\bigr)^2$'s cancel out each other and their sum 
converges to expectation. However, Equation \eqref{eq:roughlln} requires a stronger statement than the ordinary 
Law of Large Numbers; namely, we do \emph{not} assume $p^{\{ST\}}_{i,n}$ and $p^{\{ST\}}_{j,n}$ are 
identically distributed\footnote{This is reasonable, because $p^{\Upsilon}_{i,n}$ is likely to be at the same 
scale as $p_{i,n}$, whereas $p_{i,n}$ varies for different $i$.}. 
For this generalized Law of Large Numbers we need some technical 
conditions. One necessary condition is $\lim\limits_{n\rightarrow\infty}\sum_{i=1}^n \mathbb{E}\bigl[\bigl(w^{\{ST\}}_{i,n}\bigr)^2\bigr]=\infty$, which we prove by explicitly calculating 
$\mathbb{E}\bigl[\bigl(w^{\{ST\}}_{i,n}\bigr)^2\bigr]$; another requirement is that the fluctuations of 
$\bigl(w^{\{ST\}}_{i,n}\bigr)^2$ must be at comparable scales so they indeed cancel out. 
This is formalized as a uniform integrability condition, and we will show it 
imposes a non-trivial constraint on the function 
$F$ in definition of word vectors. Finally if Equation \eqref{eq:roughlln} holds, by setting 
$c_n:=\sum_{i=1}^n \mathbb{E}\bigl[\bigl(w^{\{ST\}}_{i,n}\bigr)^2\bigr]$ we are done. 

We further strengthen upper bound \eqref{eq:firstbound} as follows. First, 
Equation \eqref{eq:collo} suggests: 
$$
F(p^t_{i,n}\!+\!1/n) \approx \pi_{s/t\backslash s}F(p^{s/t\backslash s}_{i,n}\!+\!1/n) + 
(1 - \pi_{s/t\backslash s})F(p^{\{st\}}_{i,n}\!+\!1/n). 
$$
Since $F$ is smooth, this equation can be justified as long as $p^{s/t\backslash s}_{i,n}$ and 
$p^{\{st\}}_{i,n}$ are small compared to $1/n$. 
Then, we will rigorously prove that, when $n$ is sufficiently large, the total error in the above 
approximation becomes infinitesimal:
$$
\mathbf{w}^t_{n} \risingdotseq \pi_{s/t\backslash s}\mathbf{w}^{s/t\backslash s}_{n} + 
(1 - \pi_{s/t\backslash s})\mathbf{w}^{\{st\}}_{n}. 
$$
So we can replace $\mathbf{w}^t_{n}$ and $\mathbf{w}^s_{n}$ in definition of 
$\mathcal{B}^{\{st\}}_{n}$: 
\begin{equation}
\label{eq:infinitesimalB}
\mathcal{B}^{\{st\}}_{n}\risingdotseq\frac{1}{2}
\lVert(\pi_{s/t\backslash s}+\pi_{t/s\backslash t})\mathbf{w}^{\{st\}}_{n}
-\pi_{s/t\backslash s}\mathbf{w}^{s/t\backslash s}_{n}
-\pi_{t/s\backslash t}\mathbf{w}^{t/s\backslash t}_{n}\rVert. 
\end{equation}
With arguments similar to the previous paragraph, we have 
$\lim\limits_{n\rightarrow\infty}\lVert \mathbf{w}^{S/T\backslash S}_{n}\rVert$ and 
$\lim\limits_{n\rightarrow\infty}\lVert \mathbf{w}^{T/S\backslash T}_{n}\rVert$ converge to $1$ in probability. 
Therefore, by Triangle Inequality we get 
$$
\eqref{eq:infinitesimalB}\leq\frac{1}{2}\Bigl(
(\pi_{s/t\backslash s}+\pi_{t/s\backslash t})\lVert\mathbf{w}^{\{st\}}_{n}\rVert
+\pi_{s/t\backslash s}\lVert\mathbf{w}^{s/t\backslash s}_{n}\rVert
+\pi_{t/s\backslash t}\lVert\mathbf{w}^{t/s\backslash t}_{n}\rVert\Bigr) = \pi_{s/t\backslash s}+\pi_{t/s\backslash t}, 
$$
a better bound than \eqref{eq:firstbound}. 
However, our bias bound is even stronger than this. Our further argument goes to the intuition 
that, $\mathbf{w}^{s/t\backslash s}_{n}$ should have ``positive 
correlation'' with $\mathbf{w}^{\{st\}}_{n}$ because as targets, both $s/t\backslash s$ 
and $\{st\}$ contain word $t$; on the other hand, $\mathbf{w}^{s/t\backslash s}_{n}$ and 
$\mathbf{w}^{t/s\backslash t}_{n}$ should be ``independent'' because targets $s/t\backslash s$ and 
$t/s\backslash t$ cover disjoint tokens of different words. With this intuition, we will derive the 
following bias bound: 
$$
\eqref{eq:infinitesimalB}\leq\frac{1}{2}\sqrt{(\pi_{s/t\backslash s}
+\pi_{t/s\backslash t})^2+\pi_{s/t\backslash s}^2+\pi_{t/s\backslash t}^2}=
\sqrt{\frac{1}{2}(\pi_{s/t\backslash s}^2+\pi_{t/s\backslash t}^2+\pi_{s/t\backslash s}\pi_{t/s\backslash t})}. 
$$
A brief explanation can be found in Section~\ref{sec:sketch}, after the formal proof. 
Our experiments suggest that this bound is remarkably tight (Section~\ref{sec:expfunctionF}). 
In addition, the intuitive explanation inspires a way to make additive composition 
aware of word order (Section~\ref{sec:wordorder}). 

In the rest of this section, we will formally normalize $c_n$, $a^\Upsilon_n$, $b_{i,n}$ and $F$ for simplicity of 
discussion. These are mild conditions and do not affect the generality of our results. Then, 
we will summarize our claim of the bias bound, focusing on its practical verifiability. 

\begin{defn}
\label{defn:cdef}
Let $\Lambda_n$ be the set of two-word phrases observed in a finite corpus, word order ignored. 
We normalize $c_n$ such that the average norm of natural phrase 
vectors becomes $1$: 
$$
\frac{1}{|\Lambda_n |}\sum_{\{st\}\in\Lambda_n}\lVert \mathbf{w}^{\{st\}}_{n} \rVert=1. 
$$
\end{defn}
\begin{defn}
\label{defn:bdef}
Since $b_{i,n}$ is canceled out in definition of $\mathcal{B}^{\{st\}}_{n}$, it does not 
affect the bias. It does affect $\mathbf{w}^{\{st\}}_{n}$; we 
set $b_{i,n}$ such that the centroid of natural phrase vectors becomes $\mathbf{0}$: 
\begin{equation}
\label{eq:bdef}
b_{i,n}:=\frac{1}{|\Lambda_n |}\sum_{\{st\}\in\Lambda_n}F(p^{\{st\}}_{i,n}\!+\!1/n)-a^{\{st\}}_{n}, 
\end{equation}
so that 
$$
\frac{c_n}{|\Lambda_n |}\sum_{\{st\}\in\Lambda_n}w^{\{st\}}_{i,n}=0\quad\text{for all $i$}. 
$$
\end{defn}

Note that, if the centroid of natural phrase vectors is far from $\mathbf{0}$, the normalization in 
Definition~\ref{defn:cdef} would cause all phrase vectors cluster around one point on the unit sphere. 
Then, the phrase vectors would not be able 
to distinguish different meanings of phrases. The choice of $b_{i,n}$ in Definition~\ref{defn:bdef} prevents 
such degenerated cases.

Next, if $c_n$ and $F$ are fixed, $\mathcal{B}^{\{st\}}_{n}$ is taking minimum at 
$$
a^{\{st\}}_{n}-\frac{a^s_n+a^t_n}{2}=\frac{1}{n}\sum_{i=1}^{n}F(p^{\{st\}}_{i,n}\!+\!1/n)-\frac{F(p^s_{i,n}\!+\!1/n)+F(p^t_{i,n}\!+\!1/n)}{2}. 
$$
Hence, it is favorable to have the above equality. 
We can achieve it by adjusting $a^\Upsilon_n$ such that the entries 
of each vector average to $0$. 
\begin{defn}
\label{defn:adef}
We set 
\begin{equation}
\label{eq:adef}
a^\Upsilon_n:=\frac{1}{n}\sum_{i=1}^n F(p^\Upsilon_{i,n}\!+\!1/n)-b_{i,n}, 
\end{equation}
so that 
$$
\frac{c_n}{n}\sum_{i=1}^n w^\Upsilon_{i,n}=0 \quad\text{for all $\Upsilon$}. 
$$
\end{defn}

Practically, one can calculate $a^\Upsilon_n$ and $b_{i,n}$ by first assuming $b_{i,n}=0$ in \eqref{eq:adef} to 
obtain $a^\Upsilon_n$, and then substitute $a^{\{st\}}_n$ in \eqref{eq:bdef} to obtain the actual $b_{i,n}$. 
The value of $a^\Upsilon_n$ will not change because if all vectors have average entry $0$, so dose their centroid. 
In Section~\ref{sec:effectF}, we will derive asymptotic values of $a^\Upsilon_n$, $b_{i,n}$ and $c_n$ 
theoretically. 

\begin{defn}
\label{defn:fdef} We assume there is a $\lambda$ such that $F'(x)=x^{-1+\lambda}$. 
So $F(x)$ can be either $x^\lambda/\lambda$ (if $\lambda\neq 0$) or $\ln x$ (if $\lambda=0$). 
This assumption is mainly for simplicity; 
intuitively, behavior of $F(x)$ only matters at $x\approx 0$, 
because $F$ is applied to probability values which are close to $0$. 
Indeed, our results can be generalized to $G(x)$ such that 
$$
\lim_{x\rightarrow 0}G'(x)x^{1-\lambda}=1\quad\text{ and }\quad G'(x)x^{1-\lambda}\leq M \text{ 
for some constant $M$}. 
$$
See Remark~\ref{rem:extendG} in Section~\ref{sec:formalization} for further discussion. 
The exponent $\lambda$ turns out to be a crucial factor in our theory; we require $\lambda<0.5$, which 
imposes a non-trivial constraint on $F$. 
%
%
%
%
\end{defn}

Our bias bound is summarized as follows.
\begin{clm}
\label{claim:biasbound}
Assume $\lambda< 0.5$, the factors $a^\Upsilon_n$, $b_{i,n}$ and $c_n$ are normalized as above, and 
distributional vectors are constructed from an ideal natural language corpus. Then: 
$$
\lim_{n\rightarrow\infty}\mathcal{B}^{\{st\}}_{n}\leq
\sqrt{\frac{1}{2}(\pi_{s/t\backslash s}^2+\pi_{t/s\backslash t}^2+\pi_{s/t\backslash s}\pi_{t/s\backslash t})}. 
$$
\end{clm}

As we expected, for more ``collocational'' phrases, since $\pi_{s/t\backslash s}$ and $\pi_{t/s\backslash t}$ are 
smaller, the bias bound becomes stronger. Claim~\ref{claim:biasbound} 
states a prediction that can be empirically tested on a real large corpus; namely, one can 
estimate $p^\Upsilon_{i,n}$ 
from the corpus and construct $\mathbf{w}^\Upsilon_{n}$ for a fixed $n$, then check if the inequality 
holds approximately while omitting the limit. 
In Section~\ref{sec:expfunctionF}, we conduct the experiment 
and verify the prediction. Our theoretical assumptions 
on the ``ideal natural language corpus'' will be specified in Section~\ref{sec:formalization}. 

Besides it being empirically verified for phrases observed in a real corpus, 
the true value of Claim~\ref{claim:biasbound} is that the upper bound holds for an arbitrarily large ideal corpus. 
We can assume any plausible two-word phrase to occur sufficiently many times in the ideal corpus, even 
when it is unseen in the real one. In that case, a natural vector for the phrase can only be 
reliably estimated from the ideal corpus, but Claim~\ref{claim:biasbound} suggests that 
additive composition of word vectors provides a reasonable approximation for that unseen natural vector. 
Meanwhile, since word vectors \emph{can} be reliably estimated from the real corpus, Claim~\ref{claim:biasbound} 
endorses additive composition as a reasonable meaning representation for unseen or rare phrases. 
On the other hand, it endorses additive composition for frequent 
phrases as well, because such phrases usually have strong collocations and Claim~\ref{claim:biasbound} says 
that the bias in this case is small. 

The condition $\lambda< 0.5$ on function $F$ is crucial; we discuss its empirical implications in 
Section~\ref{sec:functionF}. 

Further, the following is a by-product of 
Theorem~\ref{thm:normalize} in Section~\ref{sec:effectF}, which corresponds to the previous 
Equation \eqref{eq:roughlln} in our sketch of proof. 
\begin{clm}
\label{claim:norm}
Under the same conditions in Claim~\ref{claim:biasbound}, we have 
$\lim\limits_{n\rightarrow\infty}\lVert\mathbf{w}^{\{st\}}_{n}\rVert=1$ for all $\{st\}$. 
\end{clm}
Thus, all natural phrase vectors approximately lie on the unit sphere. 
This claim is also empirically verified in Section~\ref{sec:expfunctionF}. 
It enables a link between the 
Euclidean distance $\mathcal{B}^{\{st\}}_{n}$ and the cosine similarity, which is the most widely used similarity 
measure in practice.

\subsection{Formalization and Assumptions on Natural Language Data}
\label{sec:formalization}

Claim~\ref{claim:biasbound} is formalized as Theorem~\ref{thm:main} in the following. 

 \begin{thm}
\label{thm:main}
For an ideal natural language corpus, we assume that:
\begin{enumerate}
\item[\rm (A)] $\lim\limits_{n\rightarrow\infty}p_{i,n}\cdot i\ln n=1$. 
\item[\rm (B)] Let $S$, $T$ be randomly chosen word targets. If $\Upsilon:=$ $\{ST\}$, 
$S/T\backslash S$ or $T/S\backslash T$, then: 
\begin{enumerate}
\item[\rm (B1)] For $n$ fixed, $p^{\Upsilon}_{i,n}$'s $(1\leq i\leq n)$ can be viewed as independent 
random variables. 
\item[\rm (B2)] Put $X:=p^{\Upsilon}_{i,n}/p_{i,n}$. 
There exist $\xi,\beta$ such that $\mathbb{P}(x\leq X)=\xi/x$ for $x\geq\beta$.
\end{enumerate}
\item[\rm (C)] For each $i$ and $n$, the random variables $p^{S/T\backslash S}_{i,n}$ and 
$p^{T/S\backslash T}_{i,n}$ 
are independent, whereas 
$F(p^{S/T\backslash S}_{i,n}\!+\!1/n)$ and $F(p^{\{ST\}}_{i,n}\!+\!1/n)$ have positive correlation. 
\end{enumerate}
Then, if $\mathbb{E}\bigl[F(X)^2\bigr]<\infty$, we have
$$
\lim_{n\rightarrow\infty}\mathcal{B}^{\{ST\}}_{n}\leq
\sqrt{\frac{1}{2}(\pi_{S/T\backslash S}^2+\pi_{T/S\backslash T}^2+\pi_{S/T\backslash S}\pi_{T/S\backslash T})}\quad\text{in probability}. 
$$
 \end{thm}

We explain the assumptions of Theorem~\ref{thm:main} in details below. 

\begin{rem}
Assumption (A) is the Zipf's Law \citep{zipf35}, which states that the frequency of the $i$-th word is 
inversely proportional to $i$. So $p_{i,n}$ is proportional to $i^{-1}$, and the factor $\ln n$ comes from 
equations 
$\sum_{i=1}^np_{i,n}=1$ and $\sum_{i=1}^ni^{-1}\approx\ln n$. One immediate implication of Zipf's Law 
is that one can make $np_{i,n}$ arbitrarily small by choosing sufficiently large $n$ and $i$. More precisely, 
for any $\delta>0$, we have 
\begin{equation}
\label{eq:indexnp}
np_{i,n}\leq\delta\Leftrightarrow\frac{n}{\delta\ln n}\leq i\leq n, 
\end{equation}
so as long as $n$ is large enough that $\ln n\geq 1/\delta$, there is an $i$ in \eqref{eq:indexnp} 
such that $np_{i,n}\leq\delta$. 
The limit $np_{i,n}\rightarrow 0$ will be extensively explored in our theory. 
\end{rem}

Empirically, Zipf's Law has been thoroughly 
tested under several settings \citep{Montemurro01,Ha:2002:EZL,powerlaw,corral15}. 

\begin{rem}
When a target $\Upsilon$ is randomly chosen, (B1) assumes that the probability value $p^\Upsilon_{i,n}$ 
is random enough that, when $i\neq j$, there is no obvious relation between $p^\Upsilon_{i,n}$ and 
$p^\Upsilon_{j,n}$ (i.e.~they are independent). We test this assumption in Section~\ref{sec:indeptest}. 
Assumption (B2) suggests that $p^\Upsilon_{i,n}$ is 
at the same scale as $p_{i,n}$, and the random variable $X:=p^{\Upsilon}_{i,n}/p_{i,n}$ has a 
power law tail\footnote{The assumption can further be relaxed to 
$\lim\limits_{x\rightarrow\infty}x\mathbb{P}(x\leq X)=\xi$. We only consider (B2) for simplicity.} 
of index $1$. 
We regard (B2) as the \emph{Generalized Zipf's Law}, analogous to Zipf's Law because 
$p_{i,n}$'s ($1\leq i\leq n$, $n$ fixed) can also be viewed as i.i.d. samples drawn from 
a power law of index $1$. In Section~\ref{sec:pitman}, we show that Assumption (B) is closely 
related to a Hierarchical Pitman-Yor Process; and in Section~\ref{sec:genzipf} we empirically verify this assumption. 
\end{rem}

\begin{rem} 
Assumption (C) is based on an intuition that, since $S/T\backslash S$ and $T/S\backslash T$ are 
different word targets and $p^{S/T\backslash S}_{i,n}$ and $p^{T/S\backslash T}_{i,n}$ are assessed from 
disjoint parts of corpus, the two random variables should be independent. On the other hand, 
the targets $S/T\backslash S$ and $\{ST\}$ both contain a word $T$, 
so we expect $F(p^{S/T\backslash S}_{i,n}\!+\!1/n)$ and $F(p^{\{ST\}}_{i,n}\!+\!1/n)$ to have positive correlation. 
This assumption is also empirically tested in Section~\ref{sec:indeptest}. 
\end{rem}

\begin{rem}
\label{rem:lambdacondition}
Since $X$ has a power law tail of index 1, the probability density $-\D\mathbb{P}(x\leq X)$ is a multiple of 
$x^{-2}\D x$ for sufficiently large $x$. Wherein, $\mathbb{E}\bigl[F(X)^2\bigr]$ becomes an integral of 
$F(x)^2x^{-2}\D x$, so $\lambda<0.5$ is a necessary condition for $\mathbb{E}\bigl[F(X)^2\bigr]<\infty$. 
\end{rem}

Conversely, $\lambda<0.5$ is usually a sufficient condition for 
$\mathbb{E}\bigl[F(X)^2\bigr]<\infty$, for instance, if $X$ follows the Pareto Distribution 
(i.e.~$\xi=\beta$) or Inverse-Gamma Distribution. Another example 
will be given in Section~\ref{sec:pitman}. 

\begin{lem}
\label{lem:ycalc}
Define $Y_{i,n}:=F(p^\Upsilon_{i,n}\!+\!1/n)-F(p_{i,n}\beta\!+\!1/n)=F(p_{i,n}X\!+\!1/n)-F(p_{i,n}\beta\!+\!1/n)$. 
Put $e_{i,n}:=\mathbb{E}\bigl[Y_{i,n}\bigr]$, $v_{i,n}:=\var\bigl[Y_{i,n}\bigr]$, and 
$\varphi_{i,n}:=(p_{i,n})^{2\lambda}\bigl(1+(\beta np_{i,n})^{-1}\bigr)^{-1+2\lambda}$. Then, 
\begin{enumerate}
\item[\rm (a)] There exists $\chi$ such that 
$
\dfrac{|e_{i,n}|}{\sqrt{\varphi_{i,n}}}\leq\chi
$
for all $i,n$.
\item[\rm (b)] 
$
\lim\limits_{np_{i,n}\rightarrow 0}\dfrac{e_{i,n}}{\sqrt{\varphi_{i,n}}}=0. 
$
\item[\rm (c)] 
The set of random variables $\bigl\{Y_{i,n}^2/\varphi_{i,n}\bigr\}$ is uniformly integrable; i.e., for any 
$\varepsilon>0$, there exists $N$ such that $\mathbb{E}\bigl[Y_{i,n}^2I_{Y_{i,n}^2>N \varphi_{i,n}}\bigr]<\varepsilon\varphi_{i,n}$ for all $i,n$. 
\item[\rm (d)]
$\lim\limits_{np_{i,n}\rightarrow 0}\dfrac{v_{i,n}}{\varphi_{i,n}}=\eta\neq 0$, where 
$
\eta=\displaystyle\int_0^\infty\bigl(F(z+\beta)-F(\beta)\bigr)^2\cdot\frac{\xi\D z}{z^2}
$.
\end{enumerate}
\end{lem}

\begin{rem}
\label{rem:taildom}
As sketched in Section~\ref{sec:biaspractical}, our proof requires calculation of 
$\mathbb{E}\bigl[\bigl(w^{\Upsilon}_{i,n}\bigr)^2\bigr]$; this is done by applying Lemma~\ref{lem:ycalc} above. 
The lemma calculates the first and second moments of $Y_{i,n}$; note that 
$Y_{i,n}$ differs from $w^\Upsilon_{i,n}$ only by some constant 
shift\footnote{Namely, the constant $F(p_{i,n}\beta\!+\!1/n)-a^\Upsilon_n -b_{i,n}$. 
As a further clue, in the upcoming Theorem~\ref{thm:normalize} we will 
prove that $a^\Upsilon_n$ can be taken as $0$, and $b_{i,n}$ as 
$\mathbb{E}\bigl[F(p^{\{ST\}}_{i,n}\!+\!1/n)\bigr]$ which is in the same scale as $F(p_{i,n}\beta\!+\!1/n)$.}. 
As the lemma shows, 
when $i$ and $n$ vary, the squared first moment $e_{i,n}^2$ and the variance 
$v_{i,n}$ scale with the constant $\varphi_{i,n}$. At the limit $np_{i,n}\rightarrow 0$, 
Lemma~\ref{lem:ycalc}(b)(d) suggests that 
$e_{i,n}/\sqrt{\varphi_{i,n}}$ and $v_{i,n}/\varphi_{i,n}$ converge, which is where the 
power law tail of $X$ mostly affects the behavior of $Y_{i,n}$. 
\end{rem}

\begin{rem}
\label{rem:extendG}
The function $F(x)$ in Lemma~\ref{lem:ycalc} can be generalized to function $G(x)$ as 
mentioned in Definition~\ref{defn:fdef}. Because by Cauchy's Mean Value Theorem, 
$$
\frac{G(p_{i,n}x\!+\!1/n)-G(p_{i,n}\beta\!+\!1/n)}{F(p_{i,n}x\!+\!1/n)-F(p_{i,n}\beta\!+\!1/n)}=G'(\zeta)\zeta^{1-\lambda} \quad
\text{for some $p_{i,n}\beta\!+\!1/n\leq\zeta\leq p_{i,n}x\!+\!1/n$}, 
$$
so the random variable $G(p_{i,n}X\!+\!1/n)-G(p_{i,n}\beta\!+\!1/n)$ is dominated by 
$MY_{i,n}$ and converges pointwisely to $Y_{i,n}$ as $n\rightarrow\infty$. 
Then, by Lebesgue's Dominated Convergence Theorem, 
we can generalize Lemma~\ref{lem:ycalc} to $G(x)$, and in turn generalize our bias bound. 
\end{rem}

%
%
%

\begin{lem}
\label{lem:phicalc}
Regarding the asymptotic behavior of $\varphi_{i,n}$, we have 
\begin{enumerate}
\item[\rm (a)] $\lim\limits_{np_{i,n}\rightarrow 0}\dfrac{n^{2\lambda}\cdot\varphi_{i,n}}{np_{i,n}}=\beta^{1-2\lambda}$. 
\item[\rm (b)] $n^{-1+2\lambda}\ln n\cdot\varphi_{i,n}\leq\beta^{1-2\lambda}/i$. 
\item[\rm (c)] For any $\delta>0$, there exists $M_\delta$ such that $n^{-1+2\lambda}\ln n\sum_{i=1}^{\frac{n}{\delta\ln n}}\varphi_{i,n} \leq M_\delta$ for all $n$. 
\item[\rm (d)] For any $\delta>0$, we have $\lim\limits_{n\rightarrow\infty}n^{-1+2\lambda}\ln n\sum_{\frac{n}{\delta\ln n}\leq i}^n\varphi_{i,n}=\infty$. 
\end{enumerate}
\end{lem}

%

Lemma~\ref{lem:ycalc} is derived from Assumption (B) and the condition 
$\mathbb{E}\bigl[F(X)^2\bigr]<\infty$. 
Lemma~\ref{lem:phicalc} is derived from Assumption (A). The proofs are 
found in Appendix~\ref{app:fullproof}. 

\subsection{Why is $\lambda<0.5$ important?}
\label{sec:effectF}

As we note in Remark~\ref{rem:lambdacondition}, the condition $\lambda<0.5$ is necessary for 
the existence of $\mathbb{E}\bigl[F(X)^2\bigr]$. 
This existence is important because, briefly speaking, the Law of Large Numbers only holds when 
expected values exist. More precisely, we use the following lemma to prove convergence in 
probability in Theorem~\ref{thm:main}, and in particular Equation \eqref{eq:roughlln} as discussed in 
Section~\ref{sec:biaspractical}. If $\mathbb{E}\bigl[F(X)^2\bigr]=\infty$, the 
required uniform integrability is not satisfied, which means the fluctuations of random variables may 
have too different scales to completely cancel out, so their 
weighted averages as we consider will not converge. 

\begin{lem}
\label{lem:lln}
Assume $U_{i,n}$'s ($1\leq i\leq n$, $n$ fixed) are independent random variables and 
$\bigl\{U_{i,n}/\varphi_{i,n}\bigr\}$ is uniformly integrable. Assume 
$\lim\limits_{np_{i,n}\rightarrow 0}\mathbb{E}[U_{i,n}]/\varphi_{i,n}=\ell$. Then, 
$$
\lim_{n\rightarrow\infty}\frac{\sum_{i=1}^n U_{i,n}}{\sum_{i=1}^n \varphi_{i,n}}=\ell \quad\text{in probability}.
$$
\end{lem}
\begin{proof}
This lemma is a combination of the Law of Large Numbers and the Stolz-Ces\`{a}ro Theorem. We prove it 
in two steps. 

First step, we prove 
$$
\lim_{n\rightarrow\infty}\frac{\sum_{i=1}^n U_{i,n}-\mathbb{E}[U_{i,n}]}{\sum_{i=1}^n \varphi_{i,n}}=0 
\quad\text{in probability}.
$$
This is a generalized version of the Law of Large Numbers, saying that the weighted average of $U_{i,n}$ 
converges in probability to the weighted average of $\mathbb{E}[U_{i,n}]$. Since 
$\bigl\{U_{i,n}/\varphi_{i,n}\bigr\}$ is uniformly integrable, for any $\varepsilon>0$ there exists $N$ such that 
$\mathbb{E}\bigl[\lvert U_{i,n}\rvert I_{\lvert U_{i,n} \rvert >N\varphi_{i,n}}\bigr]<\varepsilon^2 \varphi_{i,n}$ for 
all $i,n$. Our strategy is to divide the average of $U_{i,n}$ into two parts, namely 
$$
\frac{\sum_{i=1}^n U_{i,n}}{\sum_{i=1}^n \varphi_{i,n}}=
\frac{\sum_{i=1}^n U_{i,n}I_{\lvert U_{i,n} \rvert \leq N\varphi_{i,n}}}{\sum_{i=1}^n \varphi_{i,n}} + 
\frac{\sum_{i=1}^n U_{i,n}I_{\lvert U_{i,n} \rvert >N\varphi_{i,n}}}{\sum_{i=1}^n \varphi_{i,n}}, 
$$
and show that each part is close to its expectation. For the $\lvert U_{i,n} \rvert >N\varphi_{i,n}$ part, we have 
$$
\left\lvert\mathbb{E}\left[\frac{\sum_{i=1}^n U_{i,n}I_{\lvert U_{i,n} \rvert >N\varphi_{i,n}}}{\sum_{i=1}^n \varphi_{i,n}}
\right]\right\rvert<\varepsilon^2\quad\text{for all $n$} 
$$
by definition, so it has negligible expectation and can be bounded by Markov's Inequality: 

$$
\mathbb{P}\left(\left\lvert\frac{\sum_{i=1}^n U_{i,n}I_{\lvert U_{i,n} \rvert >N\varphi_{i,n}}}{\sum_{i=1}^n \varphi_{i,n}} - 
\mathbb{E}\left[\frac{\sum_{i=1}^n U_{i,n}I_{\lvert U_{i,n} \rvert >N\varphi_{i,n}}}{\sum_{i=1}^n \varphi_{i,n}}
\right] \right\rvert>\varepsilon\right)<2\varepsilon\quad\text{for all $n$}. 
$$
On the other hand, for the $\lvert U_{i,n} \rvert\leq N\varphi_{i,n}$ part we have 
$\var\bigl[U_{i,n}I_{\lvert U_{i,n} \rvert \leq N\varphi_{i,n}}\bigr]\leq N\varphi_{i,n}\mathbb{E}\bigl[\lvert U_{i,n}\rvert\bigr]$; 
and $\bigl\{U_{i,n}/\varphi_{i,n}\bigr\}$ being uniformly integrable implies 
$\mathbb{E}\bigl[\lvert U_{i,n}\rvert\bigr]\leq M\varphi_{i,n}$ for some $M$, so 
$$
\var\left[\frac{\sum_{i=1}^n U_{i,n}I_{\lvert U_{i,n} \rvert \leq N\varphi_{i,n}}}{\sum_{i=1}^n \varphi_{i,n}}\right]
\leq\frac{\sum_{i=1}^n \varphi_{i,n}^2}{\bigl(\sum_{i=1}^n \varphi_{i,n}\bigr)^2}NM. 
$$
By Lemma~\ref{lem:phicalc}(b), we have
$$
\bigl(n^{-1+2\lambda}\ln n\bigr)^2\sum_{i=1}^n \varphi_{i,n}^2\leq\beta^{2-4\lambda}\sum_{i=1}^{\infty}i^{-2}<\infty. 
$$
In contrast, by Lemma~\ref{lem:phicalc}(d) we have
$$
\lim_{n\rightarrow\infty}\bigl(n^{-1+2\lambda}\ln n\bigr)^2
\Bigl(\sum_{i=1}^n \varphi_{i,n}\Bigr)^2=\infty. 
$$
Therefore, 
$$
\lim_{n\rightarrow\infty}\frac{\sum_{i=1}^n \varphi_{i,n}^2}{\bigl(\sum_{i=1}^n \varphi_{i,n}\bigr)^2}=0\quad\text{and}\quad
\lim_{n\rightarrow\infty}\var\left[\frac{\sum_{i=1}^n U_{i,n}I_{\lvert U_{i,n} \rvert \leq N\varphi_{i,n}}}{\sum_{i=1}^n \varphi_{i,n}}\right]=0. 
$$
Thus, the $\lvert U_{i,n} \rvert\leq N\varphi_{i,n}$ part concentrates to its expectation by 
Chebyshev's Inequality. 
The first step is completed. 

Second step, we prove 
$$
\lim_{n\rightarrow\infty}\frac{\sum_{i=1}^n \mathbb{E}[U_{i,n}]}{\sum_{i=1}^n \varphi_{i,n}}=
\lim_{np_{i,n}\rightarrow 0}\frac{\mathbb{E}[U_{i,n}]}{\varphi_{i,n}}. 
$$
This is a generalized version of the Stolz-Ces\`{a}ro Theorem, saying that the limit ratio of two series equals 
the limit ratio of corresponding terms. By definition and Equation \eqref{eq:indexnp}, for any $\varepsilon>0$ there exists $\delta$ such that 
$$
\frac{n}{\delta\ln n}\leq i\leq n\Rightarrow \Bigl\lvert \frac{\mathbb{E}[U_{i,n}]}{\varphi_{i,n}} - \ell \Bigr\rvert < \varepsilon 
\quad\text{for all $n$}. 
$$
In addition, we can bound 
$$
\Bigl\lvert \frac{\mathbb{E}[U_{i,n}]}{\varphi_{i,n}} - \ell \Bigr\rvert\leq M\quad\text{for $1\leq i\leq\dfrac{n}{\delta\ln n}$}
$$
because $\bigl\{U_{i,n}/\varphi_{i,n}\bigr\}$ is uniformly integrable. The ratio 
$\bigl(\sum_{i=1}^n \mathbb{E}[U_{i,n}]\bigr)/\bigl(\sum_{i=1}^n \varphi_{i,n}\bigr)$ can be viewed as a weighted average 
of two parts, one from indices $\frac{n}{\delta\ln n}\leq i\leq n$ and the other from $1\leq i<\frac{n}{\delta\ln n}$. 
By Lemma~\ref{lem:phicalc}(d), the weight for the first part tends to infinity: 
$$
\lim_{n\rightarrow\infty}n^{-1+2\lambda}\ln n\sum_{\frac{n}{\delta\ln n}\leq i}^n \varphi_{i,n}=\infty; 
$$
whereas by Lemma~\ref{lem:phicalc}(c), the weight for the second part is finite: 
$$
n^{-1+2\lambda}\ln n\sum_{i=1}^{\frac{n}{\delta\ln n}} \varphi_{i,n}<\infty. 
$$
Therefore, the first part dominates, so 
$\lim\limits_{n\rightarrow\infty}\bigl\lvert \bigl(\sum_{i=1}^n\mathbb{E}[U_{i,n}]\bigr)/\bigl(\sum_{i=1}^n \varphi_{i,n}\bigr)-\ell \bigr\rvert < \varepsilon$. 
\end{proof}

Combining Lemma~\ref{lem:ycalc}(a)(c)(d) and Lemma~\ref{lem:lln}, 
we immediately obtain the following. This is almost Equation \eqref{eq:roughlln} we wanted in 
Section~\ref{sec:biaspractical}. 
\begin{cor}
\label{cor:normone}
Let $\Upsilon:=$ $\{ST\}$, $S/T\backslash S$ or $T/S\backslash T$. Then we have 
$$
\lim_{n\rightarrow\infty}\frac{\sum_{i=1}^n \Bigl(F(p^\Upsilon_{i,n}\!+\!1/n) - \mathbb{E}\bigl[F(p^\Upsilon_{i,n}\!+\!1/n)\bigr]\Bigr)^2}{\eta\sum_{i=1}^n \varphi_{i,n}}=1\quad\text{in probability}. 
$$
\end{cor}

Now, we can asymptotically derive the normalization of $a^\Upsilon_n$, $b_{i,n}$ 
and $c_n$, as defined in Section~\ref{sec:biaspractical}. A by-product is that the norms of 
natural phrase vectors converge to $1$. 

\begin{thm}
\label{thm:normalize}
If we put $a^{\{st\}}_n:=0$ for all $\{st\}$, and set $b_{i,n}:=\mathbb{E}\bigl[F(p^{\{ST\}}_{i,n}\!+\!1/n)\bigr]$, 
$c_n:=\bigl(\eta\sum_{i=1}^n \varphi_{i,n}\bigr)^{-1/2}$, then 
$$
\lim_{n\rightarrow\infty}\frac{c_n}{|\Lambda_n|}\sum_{\{st\}\in\Lambda_n}w^{\{st\}}_{i,n}=0, \quad
\lim_{n\rightarrow\infty}\frac{c_n}{n}\sum_{i=1}^n w^{\{ST\}}_{i,n}=0 \quad\text{and}\quad 
\lim_{n\rightarrow\infty}\lVert\mathbf{w}^{\{ST\}}_{n}\rVert=1
$$
in probability. 
\end{thm}
\begin{proof}
By the assumptions on $a^{\{st\}}_n$ and $b_{i,n}$, we have 
$$
w^{\{st\}}_{i,n}=F(p^{\{st\}}_{i,n}\!+\!1/n) - \mathbb{E}\bigl[F(p^{\{ST\}}_{i,n}\!+\!1/n)\bigr]. 
$$
Then, 
$$
\lVert\mathbf{w}^{\{ST\}}_{n}\rVert^2
=\frac{\sum_{i=1}^n \Bigl(F(p^\Upsilon_{i,n}\!+\!1/n) - \mathbb{E}\bigl[F(p^\Upsilon_{i,n}\!+\!1/n)\bigr]\Bigr)^2}{\eta\sum_{i=1}^n \varphi_{i,n}}, 
$$
so Corollary~\ref{cor:normone} implies $\lim\limits_{n\rightarrow\infty}\lVert\mathbf{w}^{\{ST\}}_{n}\rVert=1$. 

Next, Lemma~\ref{lem:ycalc}(c) implies that there exists $M$ such that 
$$\var\bigl[ F(p^{\{ST\}}_{i,n}\!+\!1/n)\bigr]\leq M\varphi_{i,n}\quad\text{for all $i,n$},$$
hence  
$$
\mathbb{E}\left[ \left( \frac{c_n}{n}\sum_{i=1}^n w^{\{ST\}}_{i,n} \right)^2 \right]=\frac{1}{n^2}\frac{\sum_{i=1}^n \var\bigl[ F(p^{\{ST\}}_{i,n}\!+\!1/n)\bigr]}{\eta\sum_{i=1}^n \varphi_{i,n}}\leq\frac{1}{n^2}\frac{M}{\eta}\rightarrow 0\quad\text{(when $n\rightarrow\infty$)}. 
$$
Therefore, by Chebyshev's Inequality we have 
$
\lim\limits_{n\rightarrow\infty}\dfrac{c_n}{n}\displaystyle\sum_{i=1}^n w^{\{ST\}}_{i,n}=0
$ in probability. 

Finally, the ordinary version of Law of Large Numbers implies that 
$$
\lim_{n\rightarrow\infty}\frac{1}{|\Lambda_n|}\sum_{\{st\}\in\Lambda_n}\frac{F(p^{\{st\}}_{i,n}\!+\!1/n) - \mathbb{E}\bigl[ F(p^{\{ST\}}_{i,n}\!+\!1/n)\bigr]}{\sqrt{\var\bigl[ F(p^{\{ST\}}_{i,n}\!+\!1/n)\bigr]}}=0 \quad\text{in probability}; 
$$
so we immediately have 
$$
\lim_{n\rightarrow\infty}\frac{c_n}{|\Lambda_n|}\sum_{\{st\}\in\Lambda_n}w^{\{st\}}_{i,n} = 0\quad\text{for all $i$}.
$$
The theorem is proven. 
\end{proof}

Therefore, if we set $a^{\{st\}}_n$, $b_{i,n}$ and $c_n$ as in Theorem~\ref{thm:normalize}, all conditions in 
Definition~\ref{defn:cdef}, 
Definition~\ref{defn:bdef} and Definition~\ref{defn:adef} are asymptotically satisfied. 
In addition, we have obtained the result stated in Claim~\ref{claim:norm}. 

In view of Corollary~\ref{cor:normone}, if $\lambda<0.5$ is not satisfied, the norms of natural phrase vectors will not 
converge. This prediction is experimentally verified in Section~\ref{sec:expfunctionF}. 

\subsection{Proof of Theorem~\ref{thm:main} and an Intuitive Explanation}
\label{sec:sketch}

In this section, we start to use Equation \eqref{eq:collo} and derive our bias bound. Recall that 
Equation \eqref{eq:collo} decomposes $p^{t}_{i,n}$ into a linear combination of 
$p^{s/t\backslash s}_{i,n}$ and $p^{\{st\}}_{i,n}$; our first notice is that $F(p^{t}_{i,n}\!+\!1/n)$ 
can be decomposed similarly into a linear combination of $F(p^{s/t\backslash s}_{i,n}\!+\!1/n)$ 
and $F(p^{\{st\}}_{i,n}\!+\!1/n)$, as if the function $F$ has linearity. This is because $F$ is smooth, 
and when $np_{i,n}$ is sufficiently small, the probability value $p_{i,n}$ is small compared to $1/n$, 
so $F(x\!+\!1/n)$ can be linearly approximated as $F'(1/n)x+F(1/n)$, as long as $x$ is at the same 
scale as $p_{i,n}$. This is formalized as the following lemma. 


\begin{lem}
\label{lem:linear}
The set of random variables 
$$\Bigl\{ \bigl( F(p^{T}_{i,n}\!+\!1/n)
 - \pi_{S/T\backslash S}F(p^{S/T\backslash S}_{i,n}\!+\!1/n) 
 - (1-\pi_{S/T\backslash S})F(p^{\{ST\}}_{i,n}\!+\!1/n) \bigr)^2/\varphi_{i,n}\Bigr\}$$ 
is uniformly integrable, and 
$$
\lim_{np_{i,n}\rightarrow 0}\frac{\mathbb{E}\Bigl[ 
\bigl( F(p^{T}_{i,n}\!+\!1/n)
 - \pi_{S/T\backslash S}F(p^{S/T\backslash S}_{i,n}\!+\!1/n) 
 - (1-\pi_{S/T\backslash S})F(p^{\{ST\}}_{i,n}\!+\!1/n) \bigr)^2
\Bigr]}{\varphi_{i,n}}=0. 
$$
\end{lem}
\begin{proof}

For brevity, we set 
\begin{multline*}
\quad\quad 
P_1 :=p^{S/T\backslash S}_{i,n},\quad P_2 :=p^{\{ST\}}_{i,n},\quad \pi := \pi_{S/T\backslash S},\\
 \text{and}\quad \tilde{F}(x) := F(x\!+\!1/n) - F(p_{i,n}\beta\!+\!1/n). 
\end{multline*}
By Equation \eqref{eq:collo} we have $p^T_{i,n}=\pi P_1 + (1-\pi) P_2$, so 
$\tilde{F}(p^T_{i,n}) = \tilde{F}(\pi P_1 + (1-\pi) P_2)$ lies in between $\tilde{F}(P_1)$ and $\tilde{F}(P_2)$. 
Therefore, 
\[\begin{split}
\bigl( \tilde{F}(p^T_{i,n})-\pi \tilde{F}(P_1)-(1-\pi) \tilde{F}(P_2) \bigr)^2 
& \leq \bigl( \tilde{F}(P_1) - \tilde{F}(P_2) \bigr)^2 \\
& \leq 4\tilde{F}(P_1)^2I_{\tilde{F}(P_1)^2\geq \tilde{F}(P_2)^2} + 4\tilde{F}(P_2)^2I_{\tilde{F}(P_2)^2\geq \tilde{F}(P_1)^2}. 
\end{split}\]
By Lemma~\ref{lem:ycalc}(c), 
$\bigl\{ \tilde{F}(P_1)^2 / \varphi_{i,n} \bigr\}$ and $\bigl\{ \tilde{F}(P_2)^2 / \varphi_{i,n} \bigr\}$ are uniformly integrable. 
So for any $\varepsilon>0$, we have 
$\mathbb{E}\bigl[ \tilde{F}(P_1)^2I_{\tilde{F}(P_1)^2>N\varphi_{i,n}} \bigr]<\varepsilon \varphi_{i,n}$ and 
$\mathbb{E}\bigl[ \tilde{F}(P_2)^2I_{\tilde{F}(P_2)^2>N\varphi_{i,n}} \bigr]<\varepsilon \varphi_{i,n}$ for some $N$. 
Consider the condition 
$$
\mathcal{C}:=\text{``}\textit{Either } \tilde{F}(P_1)^2>N\varphi_{i,n} \textit{ or } \tilde{F}(P_2)^2>N\varphi_{i,n} \text{''}, 
$$
which is weaker than ``$\bigl( \tilde{F}(p^T_{i,n})-\pi \tilde{F}(P_1)-(1-\pi) \tilde{F}(P_2) \bigr)^2>4N\varphi_{i,n}$'', and we 
have 
\begin{multline}
\label{eq:truncated}
\mathbb{E}\Bigl[ \bigl( \tilde{F}(p^T_{i,n})-\pi \tilde{F}(P_1)-(1-\pi) \tilde{F}(P_2) \bigr)^2I_{\mathcal{C}} \Bigr] \\ 
\begin{split}
& \leq \mathbb{E}\bigl[ 4\tilde{F}(P_1)^2I_{\tilde{F}(P_1)^2>N\varphi_{i,n}} + 4\tilde{F}(P_2)^2I_{\tilde{F}(P_2)^2>N\varphi_{i,n}} \bigr] \\
& < 8\varepsilon \varphi_{i,n}. 
\end{split}
\end{multline}
So $\Bigl\{ \bigl( \tilde{F}(p^T_{i,n})-\pi \tilde{F}(P_1)-(1-\pi) \tilde{F}(P_2) \bigr)^2 / \varphi_{i,n} \Bigr\}$ is uniformly 
integrable. 

The previous argument also suggests that the case $\mathcal{C}$ being satisfied is negligible, 
because \eqref{eq:truncated} is arbitrarily small. 
Thus, we only have to consider the complement of $\mathcal{C}$, namely 
$$
\neg\mathcal{C}:=\text{``}\textit{Both } \tilde{F}(P_1)^2\leq N\varphi_{i,n} \textit{ and } \tilde{F}(P_2)^2\leq N\varphi_{i,n} \text{''}.
$$
Under this condition, intuitively $\tilde{F}(P_1)$ and $\tilde{F}(P_2)$ are restricted to a small range so a 
linear approximation of $F$ becomes valid. More precisely, we show that 
\begin{equation}
\label{eq:lemrest}
\lim_{np_{i,n}\rightarrow 0}\frac{\mathbb{E}\Bigl[ \bigl( \tilde{F}(p^T_{i,n})-\pi \tilde{F}(P_1)-(1-\pi) \tilde{F}(P_2) \bigr)^2I_{\neg\mathcal{C}} \Bigr]}{\varphi_{i,n}}=0, 
\end{equation}
which will complete the proof. 
For brevity, we set  
$$
\hat{F}(x):=F(x+1)-F(1),\quad U_1:=\hat{F}(nP_1),\quad U_2:=\hat{F}(nP_2). 
$$
Let $H$ be the inverse function of $\hat{F}$:
$$
H(\hat{F}(x))=x, 
$$
and put 
$$
J(u_1, u_2; \pi):=\hat{F}(\pi H(u_1)+(1-\pi) H(u_2))-\pi u_1-(1-\pi) u_2. 
$$
Note that the functions $\hat{F}$, $H$ and $J$ do not depend on $n$, $i$, $S$ or $T$. Now, 
we consider the limit $np_{i,n}\rightarrow 0$. 
By Lemma~\ref{lem:phicalc}(a), we can replace the $\varphi_{i,n}$ in \eqref{eq:lemrest} 
with $np_{i,n}\cdot n^{-2\lambda}$; and 
since 
$$
n^{\lambda} \tilde{F}(x)=F(nx\!+\!1) - F(np_{i,n}\beta\!+\!1)\rightarrow \hat{F}(nx) \quad\text{(when $np_{i,n}\rightarrow 0$)}, 
$$ 
we can replace $n^{\lambda} \tilde{F}(x)$ with $\hat{F}(nx)$. Thus, \eqref{eq:lemrest} is equivalent to  
\begin{equation}
\label{eq:lemrestre}
\lim_{np_{i,n}\rightarrow 0}\frac{\mathbb{E}\Bigl[ J(U_1, U_2; \pi)^2I_{\mathcal{D}} \Bigr]}{np_{i,n}}=0, 
\end{equation}
where $\mathcal{D}$ is the condition 
$$
\mathcal{D}:=\text{``}\textit{Both } U_1^2\leq Nnp_{i,n} \textit{ and } U_2^2\leq Nnp_{i,n} \text{''}.
$$
Now, since $\frac{\partial}{\partial u_1}J(0,0; \pi)=\frac{\partial}{\partial u_2}J(0,0; \pi)=0$, we have 
$$
\lim_{u_1^2+u_2^2\rightarrow 0}\frac{J(u_1, u_2; \pi)^2}{u_1^2+u_2^2}=0\quad\text{uniformly on $0\leq\pi\leq 1$}. 
$$
Therefore, when $np_{i,n}\rightarrow 0$ we have 
$$
\mathbb{E}\left[\frac{J(U_1, U_2; \pi)^2I_{\mathcal{D}}}{np_{i,n}}\right]
= \mathbb{E}\left[\frac{J(U_1, U_2; \pi)^2I_{\mathcal{D}}}{U_1^2 + U_2^2}\cdot\frac{U_1^2 + U_2^2}{np_{i,n}}\right]
\leq 2N\mathbb{E}\left[\frac{J(U_1, U_2; \pi)^2I_{\mathcal{D}}}{U_1^2 + U_2^2}\right]\rightarrow 0. 
$$
Equation \eqref{eq:lemrestre} is proven and we complete. 
\end{proof}

Now, we are ready to prove Theorem~\ref{thm:main}. An intuitive discussion is given 
after the proof. 

\begin{proof}[Proof of Theorem~\ref{thm:main}]
As in Theorem~\ref{thm:normalize}, we set $a^{\{st\}}_n:=0$, 
$b_{i,n}:=\mathbb{E}\bigl[F(p^{\{ST\}}_{i,n}\!+\!1/n)\bigr]$ and $c_n:=\bigl(\eta\sum_{i=1}^n \varphi_{i,n}\bigr)^{-1/2}$. 
Assume $a_n^{t}:=0$ for all $t$, then one can calculate that 
$
\lim\limits_{n\rightarrow\infty}\frac{c_n}{n}\sum_{i=1}^n w^{T}_{i,n}=0 
$
in probability, by using Lemma~\ref{lem:linear} and similar to 
the proof of Theorem~\ref{thm:normalize}. Thus, we set 
$a^{t}_n:=0$. Then, 
$$
\bigl(\mathcal{B}^{\{ST\}}_{n}\bigr)^2=\frac{\sum_{i=1}^{n}\Bigl( F(p^{\{ST\}}_{i,n}\!+\!1/n)-\frac{1}{2}\bigl(F(p^S_{i,n}\!+\!1/n)+F(p^T_{i,n}\!+\!1/n)\bigr) \Bigr)^2}{\eta\sum_{i=1}^n \varphi_{i,n}}. 
$$
Next, by Lemma~\ref{lem:linear}, Lemma~\ref{lem:lln} and Triangle Inequality, we can replace 
$F(p^T_{i,n}\!+\!1/n)$ with 
$$
\pi_{S/T\backslash S}F(p^{S/T\backslash S}_{i,n}\!+\!1/n) 
 + (1-\pi_{S/T\backslash S})F(p^{\{ST\}}_{i,n}\!+\!1/n), 
$$
and replace $F(p^S_{i,n}\!+\!1/n)$ with 
$$
\pi_{T/S\backslash T}F(p^{T/S\backslash T}_{i,n}\!+\!1/n) 
 + (1-\pi_{T/S\backslash T})F(p^{\{ST\}}_{i,n}\!+\!1/n). 
$$
For brevity, we put $\pi_1:=\pi_{S/T\backslash S}$, $\pi_2:=\pi_{T/S\backslash T}$, 
$\tilde{F}(x) := F(x\!+\!1/n) - F(p_{i,n}\beta\!+\!1/n)$ and 
$$
\tilde{w}^{\{ST\}}_{i,n}:=\tilde{F}(p^{\{ST\}}_{i,n}),\quad
\tilde{w}^{S/T\backslash S}_{i,n}:=\tilde{F}(p^{S/T\backslash S}_{i,n}),\quad
\tilde{w}^{T/S\backslash T}_{i,n}:=\tilde{F}(p^{T/S\backslash T}_{i,n}). 
$$
We use ``$\risingdotseq$'' to denote asymptotic equality at the limit $n\rightarrow\infty$. 
Then,  
$$
\bigl(\mathcal{B}^{\{ST\}}_{n}\bigr)^2\risingdotseq\frac{\sum_{i=1}^{n}\bigl( (\pi_1+\pi_2)\tilde{w}^{\{ST\}}_{i,n}-\pi_1 \tilde{w}^{S/T\backslash S}_{i,n}-\pi_2 \tilde{w}^{T/S\backslash T}_{i,n} \bigr)^2}{4\eta\sum_{i=1}^n \varphi_{i,n}}. 
$$
Again, by Lemma~\ref{lem:ycalc}(a)(b), Lemma~\ref{lem:lln} and Triangle Inequality, we can replace 
$\tilde{w}^\Upsilon_{i,n}$ with 
$\hat{w}^\Upsilon_{i,n}:=\tilde{w}^\Upsilon_{i,n}-\mathbb{E}[\tilde{w}^\Upsilon_{i,n}]$ 
(where $\Upsilon$ is either $\{ST\}$, $S/T\backslash S$ or $T/S\backslash T$). Hence, 
\[\begin{split}
\bigl(\mathcal{B}^{\{ST\}}_{n}\bigr)^2\risingdotseq & (\pi_1+\pi_2)^2\frac{\sum_{i=1}^n\bigl(\hat{w}^{\{ST\}}_{i,n}\bigr)^2}{4\eta\sum_{i=1}^n \varphi_{i,n}}+
\pi_1^2\frac{\sum_{i=1}^n\bigl(\hat{w}^{S/T\backslash S}_{i,n}\bigr)^2}{4\eta\sum_{i=1}^n \varphi_{i,n}}+
\pi_2^2\frac{\sum_{i=1}^n\bigl(\hat{w}^{T/S\backslash T}_{i,n}\bigr)^2}{4\eta\sum_{i=1}^n \varphi_{i,n}}\\[10pt]
& - 2\pi_1(\pi_1+\pi_2)\frac{\sum_{i=1}^n\hat{w}^{S/T\backslash S}_{i,n}\hat{w}^{\{ST\}}_{i,n}}{4\eta\sum_{i=1}^n \varphi_{i,n}} 
- 2\pi_2(\pi_1+\pi_2)\frac{\sum_{i=1}^n\hat{w}^{T/S\backslash T}_{i,n}\hat{w}^{\{ST\}}_{i,n}}{4\eta\sum_{i=1}^n \varphi_{i,n}} 
\\[10pt]
& + 2\pi_1\pi_2\frac{\sum_{i=1}^n\hat{w}^{S/T\backslash S}_{i,n}\hat{w}^{T/S\backslash T}_{i,n}}{4\eta\sum_{i=1}^n \varphi_{i,n}}.
\end{split}\]
By Corollary~\ref{cor:normone}, we have 
$$
\frac{\sum_{i=1}^n\bigl(\hat{w}^{\{ST\}}_{i,n}\bigr)^2}{4\eta\sum_{i=1}^n \varphi_{i,n}}\risingdotseq\frac{1}{4}, \quad 
\frac{\sum_{i=1}^n\bigl(\hat{w}^{S/T\backslash S}_{i,n}\bigr)^2}{4\eta\sum_{i=1}^n \varphi_{i,n}}\risingdotseq\frac{1}{4} \quad\text{and}\quad 
\frac{\sum_{i=1}^n\bigl(\hat{w}^{T/S\backslash T}_{i,n}\bigr)^2}{4\eta\sum_{i=1}^n \varphi_{i,n}}\risingdotseq\frac{1}{4}. 
$$
By Assumption (C), we have $\mathbb{E}\bigl[\hat{w}^{S/T\backslash S}_{i,n}\hat{w}^{T/S\backslash T}_{i,n}\bigr]=0$, so applying Lemma~\ref{lem:lln} we get 
$$
\frac{\sum_{i=1}^n\hat{w}^{S/T\backslash S}_{i,n}\hat{w}^{T/S\backslash T}_{i,n}}{4\eta\sum_{i=1}^n \varphi_{i,n}}\risingdotseq 
\lim_{np_{i,n}\rightarrow 0}\frac{\mathbb{E}\bigl[\hat{w}^{S/T\backslash S}_{i,n}\hat{w}^{T/S\backslash T}_{i,n}\bigr]}{4\eta\varphi_{i,n}}
=0.
$$
Also by Assumption (C), we have 
$\mathbb{E}\bigl[\hat{w}^{S/T\backslash S}_{i,n}\hat{w}^{\{ST\}}_{i,n}\bigr]\geq 0$ and 
$\mathbb{E}\bigl[\hat{w}^{T/S\backslash T}_{i,n}\hat{w}^{\{ST\}}_{i,n}\bigr]\geq 0$, 
so 
$$
\lim_{n\rightarrow\infty}\frac{\sum_{i=1}^n\hat{w}^{S/T\backslash S}_{i,n}\hat{w}^{\{ST\}}_{i,n}}{4\eta\sum_{i=1}^n \varphi_{i,n}}\geq 0 \quad\text{and}\quad 
\lim_{n\rightarrow\infty}\frac{\sum_{i=1}^n\hat{w}^{T/S\backslash T}_{i,n}\hat{w}^{\{ST\}}_{i,n}}{4\eta\sum_{i=1}^n \varphi_{i,n}}\geq 0. 
$$
Therefore, 
$\lim\limits_{n\rightarrow\infty}\bigl(\mathcal{B}^{\{ST\}}_{n}\bigr)^2\leq
\frac{1}{4}\bigl( (\pi_1+\pi_2)^2 + \pi_1^2 + \pi_2^2 \bigr)
=\frac{1}{2}(\pi_{1}^2+\pi_{2}^2+\pi_{1}\pi_{2})$. 
\end{proof}

Using notations in the proof of Theorem~\ref{thm:main}, from a high level it is as if we have the 
following decomposition: 
$$
w^{t}_{i,n}=\pi_1\hat{w}^{s/t\backslash s}_{i,n} + (1-\pi_1)\hat{w}^{\{st\}}_{i,n}, 
$$
which is in correspondence to the decomposition of $p^{t}_{i,n}$ in Equation \eqref{eq:collo}. 
Similarly, 
$$
w^{s}_{i,n}=\pi_2\hat{w}^{t/s\backslash t}_{i,n} + (1-\pi_2)\hat{w}^{\{st\}}_{i,n}, 
$$
and by definition $w^{\{st\}}_{i,n}=\hat{w}^{\{st\}}_{i,n}$. Thus, 
\[\begin{split}
\bigl(w^{\{st\}}_{i,n}-\frac{1}{2}(w^{s}_{i,n} + w^{t}_{i,n})\bigr)^2
& =\frac{1}{4}\bigl( (\pi_1 + \pi_2)\hat{w}^{\{st\}}_{i,n} - \pi_1\hat{w}^{s/t\backslash s}_{i,n} - \pi_2\hat{w}^{t/s\backslash t}_{i,n} \bigr)^2 \\
& =\frac{1}{4}\Bigl( (\pi_1 + \pi_2)^2\bigl(\hat{w}^{\{st\}}_{i,n}\bigr)^2 + \pi_1^2\bigl(\hat{w}^{s/t\backslash s}_{i,n}\bigr)^2 + \pi_2^2\bigl(\hat{w}^{t/s\backslash t}_{i,n}\bigr)^2 \\
& \quad - 2\pi_1(\pi_1 + \pi_2)\hat{w}^{s/t\backslash s}_{i,n}\hat{w}^{\{st\}}_{i,n} - 2\pi_2(\pi_1 + \pi_2)\hat{w}^{t/s\backslash t}_{i,n}\hat{w}^{\{st\}}_{i,n}\\
& \quad + 2\pi_1\pi_2\hat{w}^{s/t\backslash s}_{i,n}\hat{w}^{t/s\backslash t}_{i,n}\Bigr). 
\end{split}\]
By taking summation $c_n\sum_{i=1}^n$, term $\hat{w}^{s/t\backslash s}_{i,n}\hat{w}^{t/s\backslash t}_{i,n}$'s 
cancel out to $0$ because $\hat{w}^{S/T\backslash S}_{i,n}$ and $\hat{w}^{T/S\backslash T}_{i,n}$ 
are independent; meanwhile, $\hat{w}^{s/t\backslash s}_{i,n}\hat{w}^{\{st\}}_{i,n}$'s and 
$\hat{w}^{t/s\backslash t}_{i,n}\hat{w}^{\{st\}}_{i,n}$'s sum to positive because 
$\hat{w}^{S/T\backslash S}_{i,n}$ and $\hat{w}^{T/S\backslash T}_{i,n}$ are positively correlated 
to $\hat{w}^{\{ST\}}_{i,n}$. Therefore, the sum of the above is bounded by 
$\frac{1}{4}\bigl( (\pi_1+\pi_2)^2 + \pi_1^2 + \pi_2^2 \bigr)
=\frac{1}{2}(\pi_{1}^2+\pi_{2}^2+\pi_{1}\pi_{2})$.

In view of this explanation, the technical points of Theorem~\ref{thm:main} are as follows. 
First, the decomposition of $w^{t}_{i,n}$ into $\hat{w}^{s/t\backslash s}_{i,n}$ and $\hat{w}^{\{st\}}_{i,n}$ 
is not exact; there is difference between $\hat{w}^{s/t\backslash s}_{i,n}$ and 
$\tilde{w}^{s/t\backslash s}_{i,n}$ due to the expected value, 
and there is difference between 
$F(p^T_{i,n}\!+\!1/n)$ and the linear combination of  
$F(p^{S/T\backslash S}_{i,n}\!+\!1/n)$ and $F(p^{\{ST\}}_{i,n}\!+\!1/n)$. However, by Lemma~\ref{lem:ycalc}(b) 
the expected value converges to $0$, and by Lemma~\ref{lem:linear} the linear approximation holds 
asymptotically. So this first issue is settled. Second, the most importantly, term 
$\bigl(\hat{w}^{\{st\}}_{i,n}\bigr)^2$'s, $\bigl(\hat{w}^{s/t\backslash s}_{i,n}\bigr)^2$'s and 
$\bigl(\hat{w}^{t/s\backslash t}_{i,n}\bigr)^2$'s have to sum to constants independent of $s$ and $t$, 
otherwise they cannot be separated from 
$\pi_1$ and $\pi_2$ in the calculation of $\mathcal{B}^{\{st\}}_{n}$. This requires Equation 
\eqref{eq:roughlln} as we discussed in Section~\ref{sec:biaspractical}, and it is a generalized version 
of the Law of Large Numbers. For this law to hold, one needs conditions to guarantee that the 
fluctuations of random variables are at comparable scales to cancel out. This leads to 
the condition $\lambda<0.5$, which is a non-trivial constraint on function $F$. Formally, 
Equation \eqref{eq:roughlln} is proven as Corollary~\ref{cor:normone}. 

Insights brought by our theory lead to several applications. First, as we found that 
the power law tail of natural language data requires $\lambda<0.5$ for constructing additively 
compositional vectors, our theory 
provides important guidance for empirical research on Distributional Semantics 
(Section~\ref{sec:functionF}). Second, as we found that $w^{t}_{i,n}$ and $w^{s}_{i,n}$ have 
decompositions in which $\hat{w}^{\{st\}}_{i,n}$ is a common factor and survives averaging, 
but $\hat{w}^{s/t\backslash s}_{i,n}$ and $\hat{w}^{t/s\backslash t}_{i,n}$ cancel 
out each other, 
we come to the idea of 
harnessing additive composition by engineering what is common in the summands. Then, 
for example, we can make additive composition aware of word order (Section~\ref{sec:wordorder}). 
Third, as one can read from Lemma~\ref{lem:phicalc}(c)(d) and the proof of Lemma~\ref{lem:lln}, 
it is important to realize that the behavior of vector representations is dominated by entries at dimensions 
corresponding to low-frequency words, 
namely $w^\Upsilon_{i,n}$'s where $\frac{n}{\delta\ln n}\leq i\leq n$. 
This understanding has impact on dimension reduction (Section~\ref{sec:dimensionreduction}). 

\subsection{Hierarchical Pitman-Yor Process}
\label{sec:pitman}

In Assumptions (A)(B) of Theorem~\ref{thm:main} we have required several properties to be satisfied by 
the probability values $p_{i,n}$ and $p^\Upsilon_{i,n}$. Meanwhile, $p_{i,n}$'s and $p^\Upsilon_{i,n}$'s 
($1\leq i\leq n$, $n$ fixed) define distributions from which words can be generated. 
This setting is reminiscent of a Bayesian model 
where priors of word distributions are specified. 

Conversely, by the well-known de Finetti's Theorem, an exchangeable random sequence of words (i.e., 
given any sequence sample, all permutations of that sample occur with the same probability) can be seen 
as if the words are drawn i.i.d.~from a conditioned word distribution, where the distribution itself is drawn from 
a prior. A widely studied example is the Pitman-Yor Process \citep{pitman-yor97,pitman:book}; 
in this section, we use the process to define a generative model, from which Assumptions (A)(B) 
can be derived. 

\begin{defn}
A Pitman-Yor Process $PY(\alpha,\theta)$ $(0<\alpha<1, \theta>-\alpha)$ defines a prior for word 
distributions, which is the prior corresponding to the exchangeable random sequence generated by 
the following Chinese Restaurant Process: 
\begin{enumerate}
\item First, generate a new word. 
\item At each step, let $C(\varpi)$ be the count of word $\varpi$, and $C:=\sum_{\varpi}C(\varpi)$ 
the total count; let $N$ be the number of distinct words. Then:
\begin{enumerate}
\item[(2.1)] Generate a new word with probability $\dfrac{\theta + \alpha N}{\theta + C}$. 
\item[(2.2)] Or, generate a new copy of an existing word $\varpi$, with probability 
$\dfrac{C(\varpi) - \alpha}{\theta + C}$.
\end{enumerate}
\end{enumerate}
\end{defn}

\begin{defn}
In the above process $PY(\alpha,\theta)$, we define $p(\varpi):=\lim\dfrac{C(\varpi)}{C}$, where limit is taken at 
Step $\rightarrow\infty$. 
Fix a word index $i$ such that 
$p(\varpi_i)\geq p(\varpi_{i+1})$. Put $p_i:=p(\varpi_i)$. 
\end{defn}

\begin{thm}
\label{thm:ratioCN}
For a sequence generated by $PY(\alpha, \theta)$, 
we have $\lim C/N^{1/\alpha}=Z$ for some $Z$. 
\end{thm}
\begin{proof}
This is Theorem 3.8 in \citet{pitman:book}. 
\end{proof}

\begin{thm}
\label{thm:pitmanZipf}
We have $\lim\limits_{i\rightarrow\infty} p_i\cdot i^{1/\alpha}\Gamma(1-\alpha)^{1/\alpha}=Z$, where 
$Z$ is the same as in Theorem~\ref{thm:ratioCN}. 
\end{thm}
\begin{proof}
This is Lemma 3.11 in \citet{pitman:book}. 
\end{proof}

Theorem~\ref{thm:pitmanZipf} shows that, if words are generated by a Pitman-Yor Process $PY(\alpha, \theta)$, 
then $p_i$ has a power law tail of index $\alpha$. It is in the same form as Assumption (A), and when 
$\alpha\approx 1$, it approximates the Zipf's Law. 

For two sequences generated by $PY(\alpha, \theta)$, their corresponding 
$Z$ as in Theorem~\ref{thm:ratioCN} may differ 
(since the sequences are random), even if 
they are generated 
with the same hyper-parameters $\alpha$ and $\theta$. Nevertheless, the limit always exists, and $Z$ 
follows 
a statistical distribution.
The probability density of $Z^{-\alpha}$ is 
derived in \citet{pitman:book}, Theorem 3.8: 
\begin{equation}
\label{eq:distZ}
-\D\mathbb{P}(x\leq Z^{-\alpha})=\frac{\Gamma(\theta+1)}{\Gamma(\theta/\alpha+1)}x^{\theta/\alpha}g_\alpha(x)\D x \quad(x>0), 
\end{equation}
where $g_\alpha(x)$ is the Mittag-Leffler density function: 
$$
g_\alpha(x):=\frac{1}{\pi\alpha}\sum_{k=0}^{\infty}\frac{(-1)^{k+1}}{k!}\Gamma(\alpha k + 1)\sin(\pi\alpha k)x^{k-1}. 
$$
In this article, we only need the fact that $\lim\limits_{x\rightarrow 0}xg_\alpha(x)$ is a nonzero constant. 

Next, we consider the co-occurrence probability $p^\Upsilon(\varpi)$, conditioned on $\varpi$ being in the context of 
a target $\Upsilon$. One first notes that $p^\Upsilon(\varpi)$ is likely to be related 
to $p(\varpi)$; i.e., frequent words are likely to occur in every context, regardless of target. 
To model this intuition, the idea of Hierarchical Pitman-Yor Process \citep{teh:2006:COLACL} is to 
adapt $PY(\alpha,\theta)$ such that in each step, if a new word is to be generated, it is no longer generated 
brand new, but drawn from another Pitman-Yor Process instead. This second Pitman-Yor Process serves as 
a ``reference'' which controls how frequently a word is likely to occur. More precisely, a Hierarchical 
Pitman-Yor Process $H\!PY(\alpha_1,\theta_1;\alpha_2,\theta_2)$ generates sequences as follows. 

\begin{defn}
\label{defn:hpy}
In $H\!PY(\alpha_1,\theta_1;\alpha_2,\theta_2)$, instead of generating words directly, one generates a 
``reference'' at each step, where the reference can 
refer to new words or existing words. We use $\varrho$ to denote a reference and $\varpi^\varrho$ the word 
referred to by the reference. 
\begin{enumerate}
\item First step, generate a new reference which refers to a new word. 
\item At each step, let $C(\varrho)$ be the count of reference $\varrho$, and 
$C(\varpi):=\sum_{\varpi^\varrho=\varpi}C(\varrho)$ the count of all references referring to word $\varpi$; 
let $C:=\sum_{\varpi}C(\varpi)$ be the total count, $N_r(\varpi)$ the number of distinct references referring 
to $\varpi$, and $N_r:=\sum_{\varpi}N_r(\varpi)$ the total number of distinct references; finally, let 
$N_w$ be the number of distinct words. 
\begin{enumerate}
\item[(2.1)] Generate a new reference referring to a new word, with probability 
$$\frac{1}{\theta_1 + C}\cdot\frac{\theta_1+\alpha_1N_r}{\theta_2+N_r}\cdot(\theta_2+\alpha_2N_w).$$ 
\item[(2.2)] Generate a new reference referring to an existing word $\varpi$, with probability 
$$\frac{1}{\theta_1 + C}\cdot\frac{\theta_1+\alpha_1N_r}{\theta_2+N_r}\cdot(N_r(\varpi)-\alpha_2).$$ 
\item[(2.3)] Or, generate a new copy of an existing reference $\varrho$, with probability 
$$\frac{1}{\theta_1 + C}\cdot(C(\varrho)-\alpha_1).$$ 
\end{enumerate}
\end{enumerate}
\end{defn}

It is easy to see from definition that $H\!PY(\alpha_1,\theta_1;\alpha_2,\theta_2)$ generates an exchangeable 
word sequence; and if we focus on distinct references (i.e., ignoring (2.3), consider $N_r(\varpi)$ as 
``the count of word $\varpi$'' in the ordinary Pitman-Yor Process), then the process becomes 
$PY(\alpha_2,\theta_2)$. We assume this is the 
same process which defines word probability $p(\varpi)$, so 
$$
p(\varpi)=\lim\frac{N_r(\varpi)}{N_r}; 
$$
and we define the conditional probability $p^\Upsilon(\varpi)$ as: 
$$
p^\Upsilon(\varpi):=\lim\frac{C(\varpi)}{C}. 
$$
Thus, $H\!PY(\alpha_1,\theta_1;\alpha_2,\theta_2)$ indeed connects $p^\Upsilon(\varpi)$ to $p(\varpi)$. 
This connection between word probability and conditioned word probability has been explored in 
\citet{teh:2006:COLACL}; in which, it is used in an $n$-gram language model to connect 
the bigram probability $p(w|u)$ to unigram probability $p(w)$, for deriving a smoothing method. 

Unfortunately, a precise analysis on the above $p^\Upsilon(\varpi)$ is beyond the reach of the authors; 
instead, we consider a slightly modified process which is much simpler for our purpose. 

\begin{defn}
A \emph{Modified Hierarchical Pitman-Yor Process} 
$M\!H\!PY(\alpha_1,\theta_1;\alpha_2,\theta_2)$ generates sequences as follows. Using the same notation as in 
Definition~\ref{defn:hpy}: 
\begin{enumerate}
\item First step, generate a new reference which refers to a new word. 
\item At each step: 
\begin{enumerate}
\item[(2.1)] Generate a new reference referring to a new word, with probability 
$$\frac{1}{D}\cdot(\theta_2+\alpha_2N_w).$$ 
\item[(2.2)] Generate a new reference referring to an existing word $\varpi$, with probability 
$$\frac{1}{D}\cdot(N_r(\varpi)-\alpha_2).$$ 
\item[(2.3)] Or, generate a new copy of an existing reference $\varrho$, with probability 
$$\frac{1}{D}\cdot\frac{N_r(\varpi^\varrho)-\alpha_2}{\theta_1 + \alpha_1N_r(\varpi^\varrho)}\cdot(C(\varrho)-\alpha_1).$$ 
\end{enumerate}
In above, $D$ is a normalization factor that makes the probability values sum to 1: 
$$
D:=\theta_2+ \alpha_2N_w + \sum_\varpi \frac{N_r(\varpi)-\alpha_2}{\theta_1 + \alpha_1N_r(\varpi)}\cdot(C(\varpi)+\theta_1). 
$$
\end{enumerate}
\end{defn}

$M\!H\!PY(\alpha_1,\theta_1;\alpha_2,\theta_2)$ modifies $H\!PY(\alpha_1,\theta_1;\alpha_2,\theta_2)$ 
by canceling $(\theta_1+\alpha_1N_r)/(\theta_2+N_r)$ in (2.1) and (2.2), and scaling (2.3) by 
a $(N_r(\varpi^\varrho)-\alpha_2)/(\theta_1 + \alpha_1N_r(\varpi^\varrho))$ factor instead. It is noteworthy that, 
since $\lim N_r=\infty$ and $\lim N_r(\varpi^\varrho)=\infty$, we have 
$$
\lim\frac{\theta_1+\alpha_1N_r}{\theta_2+N_r}=\alpha_1\quad\text{and}\quad
\lim\frac{N_r(\varpi^\varrho)-\alpha_2}{\theta_1 + \alpha_1N_r(\varpi^\varrho)}=\frac{1}{\alpha_1}. 
$$
So the asymptotic behaviors of $M\!H\!PY(\alpha_1,\theta_1;\alpha_2,\theta_2)$ and $H\!PY(\alpha_1,\theta_1;\alpha_2,\theta_2)$ are similar. 

A favorable property of $M\!H\!PY(\alpha_1,\theta_1;\alpha_2,\theta_2)$ is that, like 
$H\!PY(\alpha_1,\theta_1;\alpha_2,\theta_2)$, it becomes $PY(\alpha_2,\theta_2)$ when one focuses on 
distinct references, so we have 
\begin{equation}
\label{eq:pNN}
p(\varpi)=\lim\frac{N_r(\varpi)}{N_r}
\end{equation}
as before; besides, if restricted to a specific word $\varpi$ (i.e., ignoring (2.1), only consider the 
references referring to $\varpi$, and regard references as ``words'', $C(\varrho)$ as ``the count of word 
$\varrho$'', and $C(\varpi):=\sum_{\varpi^\varrho=\varpi}C(\varrho)$ as ``the total count'', $N_r(\varpi)$ as 
``the number of distinct words'' in the ordinary Pitman-Yor Process), then the process becomes 
$PY(\alpha_1,\theta_1)$. Thus, by Theorem~\ref{thm:ratioCN} 
\begin{equation}
\label{eq:CNZ}
\lim\frac{C(\varpi)}{N_r(\varpi)^{1/\alpha_1}}=Z_\varpi\quad\text{for some $Z_\varpi$}. 
\end{equation}
Therefore, combining \eqref{eq:pNN} and \eqref{eq:CNZ} we have 
$$
p^\Upsilon(\varpi):=\lim\frac{C(\varpi)}{C}=p(\varpi)^{1/\alpha_1}Z_\varpi\lim\frac{N_r^{1/\alpha_1}}{C}. 
$$
So $p^\Upsilon(\varpi)/p(\varpi)^{1/\alpha_1}$ is a constant multiple of $Z_\varpi$, which follows a distribution 
specified in Equation \eqref{eq:distZ}; and it is easy to see that $Z_\varpi$'s for different $\varpi$ are 
mutually independent. Thus, we have obtained Assumption (B1). 

As for Assumption (B2), we assume $\theta_1=1$ and derive the distribution of $Z_\varpi$ from 
\eqref{eq:distZ}:
$$
-\D\mathbb{P}(z\leq Z_w)=\frac{\alpha_1}{\Gamma(1/\alpha_1+1)}\frac{z^{-\alpha_1}g_{\alpha_1}(z^{-\alpha_1})}{z^2}\D z \quad(z>0). 
$$
Since $\lim\limits_{x\rightarrow 0}xg_\alpha(x)$ is a nonzero constant, the above probability density 
is of order $o(z^{-2})$ when $z\rightarrow\infty$, so the random variable $Z_w$ has a power law tail of index $1$. 
Thus, Assumption (B2) is approximately satisfied when $\alpha_1\approx 1$ and $\theta_1=1$. 

\section{Applications}
\label{sec:applications}

In this section, we demonstrate three applications of our theory. 

\subsection{The Choice of Function $F$}
\label{sec:functionF}

The condition $\lambda<0.5$ specifies a nontrivial constraint on the function 
$F$. In Section~\ref{sec:effectF} we have shown that this is a necessary condition for the norms of natural phrase 
vectors to converge. The convergence of norms is an outstanding property that might affect not only 
additive composition but also the 
composition ability of vector representations in general. Specifically, we note that 
$F(x)=\ln{x}$ when $\lambda=0$, and $F(x)=\sqrt{x}$ when $\lambda=0.5$. It is straightforward to 
predict that these functions might perform better in composition tasks than functions that have larger 
$\lambda$, 
such as $F(x):=x$ or $F(x):=x\ln{x}$. In Section~\ref{sec:expfunctionF}, we show 
experiments that verify the necessity of $\lambda<0.5$ for our bias bound to hold, and in Section~\ref{sec:exteval} 
we show that $F$ indeed drastically affects additive compositionality as judged by human annotators; 
while $F(x):=\ln{x}$ and $F(x):=\sqrt{x}$ perform similarly well, $F(x):=x$ and $F(x):=x\ln{x}$ are much worse. 

Different settings of function $F$ have been considered in previous research, and 
speculations have been made about the reason of semantic additivity of some of the vector representations. 
In \citet{pennington-socher-manning14}, the authors noted that logarithm is a homomorphism from multiplication 
to addition, and used this property to justify $F(x):=\ln{x}$ for training semantically additive 
word vectors, but based on the unverified hypothesis that multiplications of co-occurrence probabilities 
have specialties in semantics. On the other hand, \citet{HellingerPCA:EACL} proposed to use $F(x):=\sqrt{x}$, 
which is motivated by the Hellinger distance between two probability distributions, and reported it being 
better than $F(x):=x$. \citet{stratos-collins-hsu:2015:ACL-IJCNLP} proposed a similar but more general 
and better-motivated model, which attributed $F(x):=\sqrt{x}$ to an optimal choice that stabilizes 
the variances of Poisson random variables. Based on the assumption that co-occurrence counts are generated 
by a Poisson Process, the authors pointed out that $F(x):=\sqrt{x}$ may have the effect of stabilizing the 
\emph{variance} 
in estimating word vectors. In contrast, our theory shows clearly that $F$ affects the \emph{bias} of additive 
composition, besides variance. All in all, none of the previous research can explain why $F(x):=\ln{x}$ 
and $F(x):=\sqrt{x}$ are \emph{both} good choices but $F(x):=x$ is not. 

Intuitively, the condition $\lambda<0.5$ requires $F(x)$ to decrease steeply as $x$ tends to $0$. 
The steep slope has effect of ``amplifying'' the fluctuations of lower co-occurrence probabilities, 
and ``suppressing'' higher ones as a result. 
Formally, this can be read from Lemma~\ref{lem:ycalc}, which 
shows that $\var[F(p^\Upsilon_{i,n}\!+\!1/n)]$ 
scales with $\varphi_{i,n}=p_{i,n}\bigl(p_{i,n}+(\beta n)^{-1}\bigr)^{-1+2\lambda}$. When $\lambda<0.5$, 
the $\bigl(p_{i,n}+(\beta n)^{-1}\bigr)^{-1+2\lambda}$ factor decreases as $p_{i,n}$ increases, and the 
decrease becomes faster when $\lambda$ is smaller. Thus, in the vector representations we consider, 
higher co-occurrence probabilities are ``suppressed'' more when $\lambda$ is smaller. 


\subsection{Handling Word Order in Additive Composition}
\label{sec:wordorder}

By considering the vector representation $\mathbf{w}^{\{st\}}_n$ we have ignored word order and 
conflated the phrases ``$s$ $t$'' and ``$t$ $s$''. Though the meanings of the two might be 
related somehow, to treat a compositional framework as approximating $\mathbf{w}^{\{st\}}_n$ instead of 
$\mathbf{w}^{st}_n$ would certainly be troublesome, especially when one tries to extend our theory to longer 
phrases or even sentences. As the following example~\citep{Landauer97howwell} demonstrates, 
meanings of sentences 
may differ greatly as word order changes. 
\begin{enumerate}
\item[a.] \textit{It was not the sales manager who hit the bottle that day, but the office worker with 
the serious drinking problem.}
\item[b.] \textit{That day the office manager, who was drinking, hit the problem sales worker with a bottle, 
but it was not serious.}
\end{enumerate}
Thus, it is necessary to handle the changes of meaning brought by different word order. Traditionally, 
additive composition is considered unsuitable for this purpose, because one always has 
$\mathbf{w}^s_{n}+\mathbf{w}^t_{n}=\mathbf{w}^t_{n}+\mathbf{w}^s_{n}$. 
However, the commutativity can be broken by defining different contexts for ``left-hand-side'' words 
and ``right-hand-side'' words, denoted by $t\bullet$ and $\bullet t$, respectively. Then, 
the co-occurrence probabilities 
$p^{t\bullet}_{i,n}$ and $p^{\bullet t}_{i,n}$ will be different, so 
$\frac{1}{2}(\mathbf{w}^{s\bullet}_n+\mathbf{w}^{\bullet t}_n)$ and 
$\frac{1}{2}(\mathbf{w}^{t\bullet}_n+\mathbf{w}^{\bullet s}_n)$ are different vectors. In this section, we 
propose the \emph{Near-far Context}, which specifies contexts for $s\bullet$ and $\bullet t$ such that 
the additive composition $\frac{1}{2}(\mathbf{w}^{s\bullet}_n+\mathbf{w}^{\bullet t}_n)$ approximates 
the natural vector $\mathbf{w}^{st}_n$ for \emph{ordered} phrase ``$s$ $t$''. 

\begin{figure}
\centering
\begin{minipage}{.48\textwidth}
  \centering
  \includegraphics[scale=0.38,bb=0 0 450 120,clip]{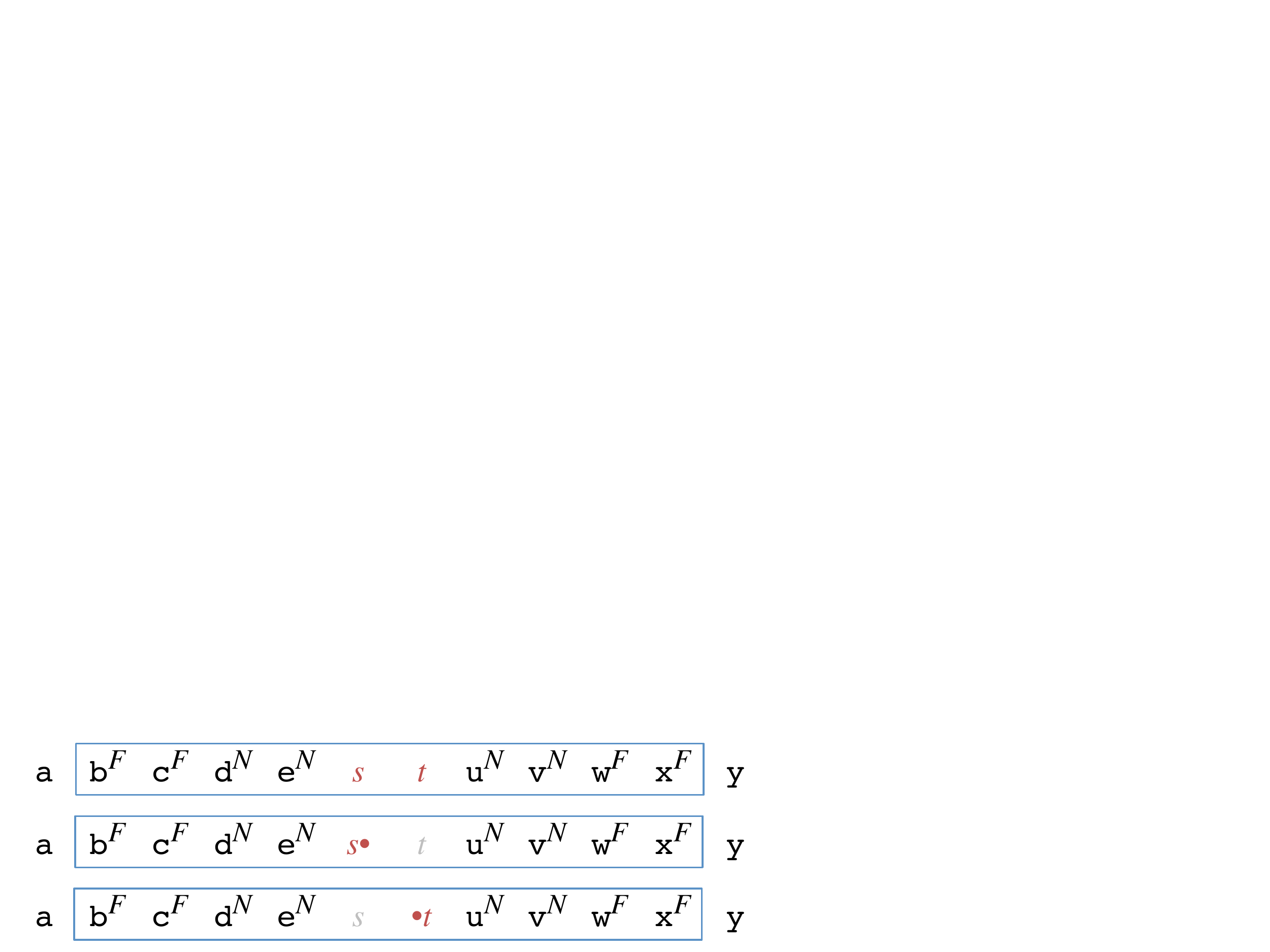}
  \captionof{figure}{Surrounding two-word phrase ``$s$ $t$'', the Near-far Contexts assigned to 
  $s\bullet$, $\bullet t$ and $st$ are the same.}
  \label{fig:nearfar}
\end{minipage}%
\begin{minipage}{.04\textwidth}
  ~
\end{minipage}%
\begin{minipage}{.48\textwidth}
  \centering
  \includegraphics[scale=0.38,bb=0 0 450 120,clip]{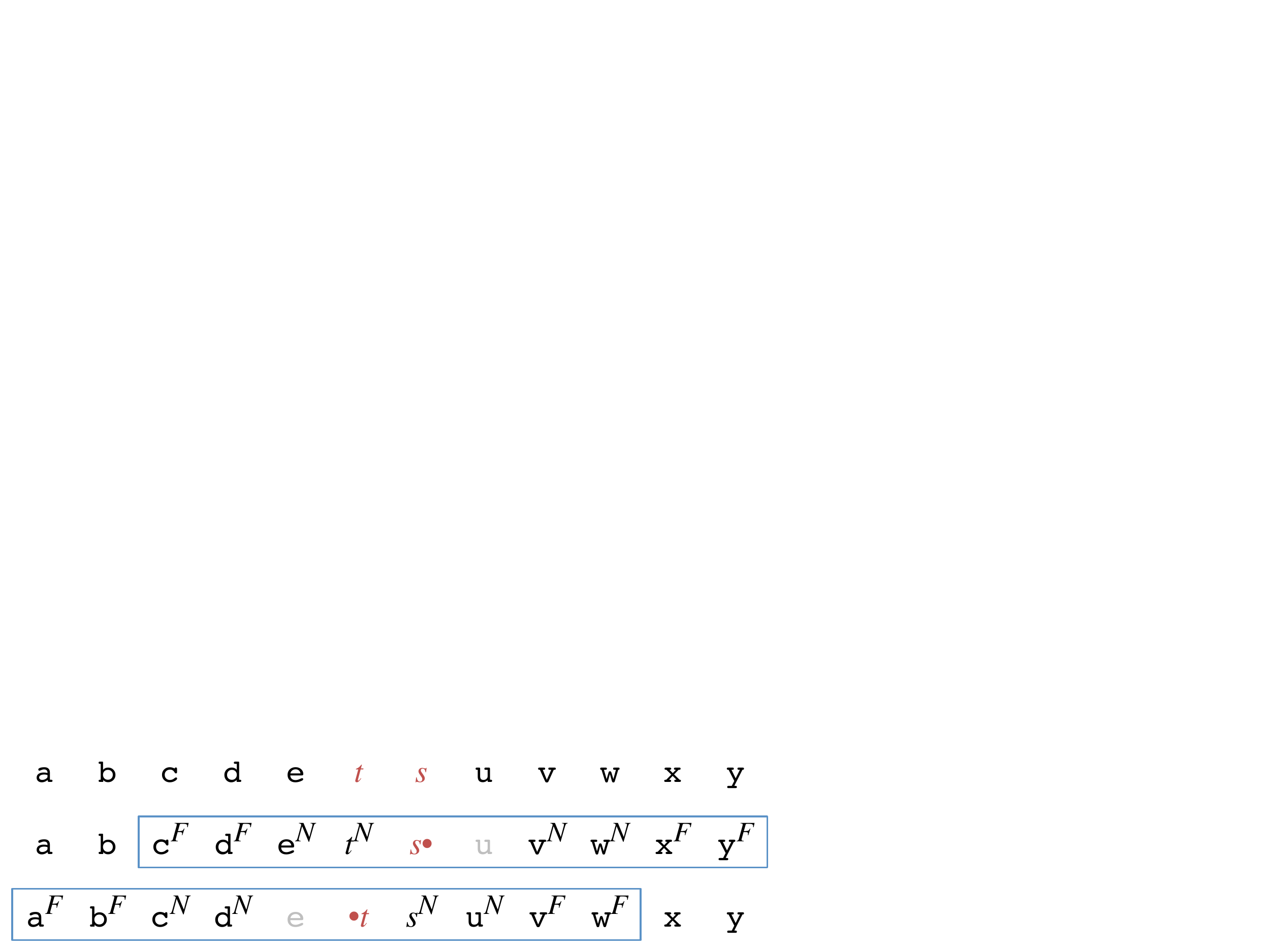}
  \captionof{figure}{Surrounding phrase ``$t$ $s$'', word order reversed, the Near-far Contexts assigned to 
  $s\bullet$ and $\bullet t$ differ in their \textbf{\textit{N}}-\textbf{\textit{F}} labels.}
  \label{fig:nearfar_reverse}
\end{minipage}
\end{figure}



\begin{defn}
\label{defn:nflabel}
In \emph{Near-far Context}, context words are assigned labels, either \textbf{\textit{N}} or \textbf{\textit{F}}. 
For constructing vector representations, we use a lexicon of  \textbf{\textit{N}}-\textbf{\textit{F}} labeled 
words, and regard words with different labels as different entries in the lexicon. For any target, we 
label the nearer two words to each side by \textbf{\textit{N}}, and the farther two words to each side 
by \textbf{\textit{F}}. Except that, for the ``left-hand-side'' word $s\bullet$ we skip one word adjacent to 
the right; and similarly, for the ``right-hand-side'' word $\bullet t$ we skip one word adjacent to the left 
(Figure~\ref{fig:nearfar}). 
\end{defn}

The idea behind Near-far Context is that, in the context of phrase ``$s$ $t$'', each word is assigned an 
\textbf{\textit{N}}-\textbf{\textit{F}} label the same as in the context of $s\bullet$ and 
$\bullet t$ (Figure~\ref{fig:nearfar}). 
On the other hand, for targets $s$ and $t$ occurring in the order-reversed phrase ``$t$ $s$'', context words 
are labeled differently for $s\bullet$ and $\bullet t$ (Figure~\ref{fig:nearfar_reverse}). As we discussed 
in Section~\ref{sec:biaspractical}, the key fact about additive composition is 
that if a word token $t$ comes from phrase ``$s$ $t$'' or ``$t$ $s$'', the context for this token of $t$ is 
almost the same as 
the context of ``$s$ $t$'' or ``$t$ $s$''. By introducing different labels for context words of $t\bullet$ and 
$\bullet t$, we 
are able to distinguish ``$s$ $t$'' from ``$t$ $s$''. 
More precisely, similar to our discussion in Section~\ref{sec:sketch}, the common component of 
$w^{s\bullet}_{i,n}$ and $w^{\bullet t}_{i,n}$ will survive in the average 
$\frac{1}{2}(w^{s\bullet}_{i,n}+w^{\bullet t}_{i,n})$, whereas independent ones 
will cancel out each other. Thus, the additive composition 
$\frac{1}{2}(\mathbf{w}^{s\bullet}_n+\mathbf{w}^{\bullet t}_n)$ will become closer to $\mathbf{w}^{st}_n$ rather than 
$\mathbf{w}^{ts}_n$, because $s\bullet$ and $\bullet t$ share context surrounding 
``$s$ $t$'' but not ``$t$ $s$''.

\begin{defn}
\label{defn:nfdecomp}
Formally, as analogue to Definition~\ref{defn:collo}, we define target 
${s\bullet}\backslash t$ which counts every $s\bullet$ not at the left of word $t$. We denote 
$\pi_{{s\bullet}\backslash t}$ the probability of $s$ 
being not at the left of $t$, conditioned on its occurrence. 
Practically, one can estimate $(1-\pi_{{s\bullet}\backslash t})$ by $C(st)/C(s)$. 
Similarly, we define $s/{\bullet t}$ and $\pi_{s/{\bullet t}}$. Then, we have equations 
\begin{align*}
p^{s\bullet}_{i,n} & = \pi_{{s\bullet}\backslash t}p^{{s\bullet}\backslash t}_{i,n}+(1-\pi_{{s\bullet}\backslash t})p^{st}_{i,n} \quad\text{for all $i,n$}  \\
p^{\bullet t}_{i,n} & = \pi_{s/{\bullet t}}p^{s/{\bullet t}}_{i,n}+(1-\pi_{s/{\bullet t}})p^{st}_{i,n} \quad\text{for all $i,n$} 
\end{align*}
as parallel to \eqref{eq:collo}. 
\end{defn}

The following claim is parallel to Claim~\ref{claim:biasbound}.
\begin{clm}
\label{claim:nearfar}
Under conditions parallel to Claim~\ref{claim:biasbound}, we have 
$$
\lim_{n\rightarrow\infty}\mathcal{B}^{st}_{n}:=
\lVert \mathbf{w}^{st}_n-\frac{1}{2}(\mathbf{w}^{s\bullet}_n + \mathbf{w}^{\bullet t}_n) \rVert
\leq
\sqrt{\frac{1}{2}(\pi_{{s\bullet}\backslash t}^2+\pi_{s/{\bullet t}}^2+\pi_{{s\bullet}\backslash t}\pi_{s/{\bullet t}})}. 
$$
\end{clm}

In Section~\ref{sec:expnearfar}, we verify Claim~\ref{claim:nearfar} experimentally, and show that in contrast, 
the error $\lVert \mathbf{w}^{ts}_n-\frac{1}{2}(\mathbf{w}^{s\bullet}_n + \mathbf{w}^{\bullet t}_n) \rVert$
for approximating the order-reversed phrase ``$t$ $s$'' can exceed this bias bound. 
Further, we demonstrate that by using the additive composition of Near-far Context 
vectors, one can indeed assess meaning similarities between ordered phrases. 

\subsection{Dimension Reduction}
\label{sec:dimensionreduction}

By far we have only discussed vector representations that have a 
high dimension equal to the lexicon size $n$. In practice, people mainly use low-dimensional 
``embeddings'' of words to represent their meaning. Many of the embeddings, including SGNS 
and GloVe, can be formalized as linear dimension reduction, which is equivalent to the 
finding of a $d$-dimensional vector $\mathbf{v}^t$ (where $d\ll n$) for each target word $t$, 
and an $(n,d)$-matrix $A$ such that $\sum_t L(A\mathbf{v}^t, \mathbf{w}^t_n)$ is minimized for some 
loss function $L(\cdot, \cdot)$. In other words, $A\mathbf{v}^t$ is trained as a good approximation 
for $\mathbf{w}^t_n$. 

Naturally, we expect the loss function $L$ to account for a crucial factor in word embeddings. 
Although there are empirical investigations on other detailed designs of embedding methods (e.g.~how to count 
co-occurrences, see \citealt{levyTACL}), the 
loss functions have not been explicitly discussed previously. In this section, 
we discuss \textit{how the loss functions would affect additive compositionality of word embeddings}, 
from a viewpoint of bounding the bias 
$\lVert \mathbf{v}^{\{st\}}-\frac{1}{2}(\mathbf{v}^{s}+\mathbf{v}^{t}) \rVert$.

\paragraph{SVD} When $L$ is the $L^2$-loss, its minimization has a closed-form solution 
given by the Singular Value Decomposition (SVD). More precisely, one considers a matrix 
whose $j$-th column is $\mathbf{w}^t_n$ where $t$ is the $j$-th target word. Then, SVD factorizes 
the matrix into 
$U\Sigma V^\top$, where $U$, $V$ are orthonormal and 
$\Sigma$ is diagonal. 
Let $\Sigma_d$ denote the truncated $\Sigma$ to the top $d$ singular values. Then, $A$ is solved as 
$U\sqrt{\Sigma_d}$ and $\mathbf{v}^{t}$ the $j$-th column of $\sqrt{\Sigma_d}V^\top$. 
SVD has been used in \citet{HellingerPCA:EACL}, \citet{stratos-collins-hsu:2015:ACL-IJCNLP} 
and \citet{levyTACL}. In this setting, we have
$$
\lVert A\mathbf{v}^{s}-\mathbf{w}^s_{n} \rVert\leq\varepsilon_1,\quad
\lVert A\mathbf{v}^{t}-\mathbf{w}^t_{n} \rVert\leq\varepsilon_2\quad\text{ and }\quad
\lVert A\mathbf{v}^{\{st\}}-\mathbf{w}^{\{st\}}_n \rVert\leq\varepsilon_3,
$$
where $\varepsilon_1$, $\varepsilon_2$ and $\varepsilon_3$ are minimized. Thus, by Triangle Inequality we have 
$$
\lVert A\cdot\bigl(\mathbf{v}^{\{st\}}-\frac{1}{2}(\mathbf{v}^{s}+\mathbf{v}^{t})\bigr) \rVert
 \leq \mathcal{B}^{\{st\}}_n + 
\frac{1}{2}(\varepsilon_1+\varepsilon_2) + \varepsilon_3. 
$$
Further, by Claim~\ref{claim:biasbound} we can bound $\mathcal{B}^{\{st\}}_n$ for sufficiently large $n$, 
so $\lVert \mathbf{v}^{\{st\}}-\frac{1}{2}(\mathbf{v}^{s}+\mathbf{v}^{t}) \rVert$ is bounded in turn because $A$ is 
a bounded operator. 
This bound suggests that word embeddings trained by SVD preserve additive compositionality. 

However, 
the same argument does not directly apply to other loss functions 
because a general loss may not satisfy a triangle inequality, and a bound for Euclidean distance may not 
always transform to a bound for the loss, or vice versa. Specifically, we describe two 
widely used alternative embeddings in the following and discuss the effects of their loss. 

\paragraph{GloVe} The GloVe model~\citep{pennington-socher-manning14} trains a dimension reduction for 
vector representations with $F(x):=\ln{x}$. Let $v^t_i$ be the $i$-th entry of $A\mathbf{v}^t$, and 
$C^t_i$ the co-occurrence count. 
Then, the loss function of GloVe is given by 
$$
L(v^t_i, w^t_{i,n}):=f\bigl( C^t_i \bigr)(v^t_i-w^t_{i,n})^2, 
$$
where $f$ is a function set to constant when $C^t_i$ is larger than a 
threshold, and decreases to $0$ when $C^t_i\rightarrow 0$. 
In words, GloVe uses a weighted $L^2$-loss and the weight is a function 
of co-occurrence count.
To minimize the loss, GloVe uses stochastic 
gradient descent methods such as AdaGrad~\citep{Duchi:2011}. 

\paragraph{SGNS} The Skip-Gram with Negative Sampling (SGNS) model \citep{word2vecNIPS} also trains a 
dimension reduction with $F(x):=\ln{x}$. The training is based on the 
Noise Contrastive Estimation (NCE) \citep{gutmann12}, so its loss function has two parameters, 
the number $k$ of noise samples per 
data point, and the noise distribution $p^{\text{noise}}_{i,n}$. 

\begin{clm}
\label{claim:NCEloss}
Let $v^t_i$ be the $i$-th entry of $A\mathbf{v}^t$. The loss function of SGNS is given by 
\begin{equation*}
L(v^t_i, w^t_{i,n}) := C(t) D_{\phi_i}\bigl(v^t_i + \ln(kp^{{\rm noise}}_{i,n}),\, w^t_{i,n} + \ln(kp^{{\rm noise}}_{i,n})\bigr),
\end{equation*}
where 
$D_{\phi}(\cdot,\cdot)$ is the Bregman divergence associated to the convex function
$$
\phi(x):=\bigl(p^t_{i,n}+kp^{{\rm noise}}_{i,n}\bigr)\ln\bigl(\exp(x)+kp^{{\rm noise}}_{i,n}\bigr).
$$
When $k\rightarrow+\infty$, $D_{\phi}$ converges to the Bregman divergence $D_\varphi$ 
associated to $\varphi(x):=\exp(x)$.
\end{clm}

Proof of Claim~\ref{claim:NCEloss} is found in Appendix~\ref{app:sgns}. We draw a graph of the SGNS loss in 
Figure~\ref{fig:loss}, where $D_{\phi}\bigl(v^t_i + \ln(kp^{\text{noise}}_{i,n}),\, w^t_{i,n} + \ln(kp^{\text{noise}}_{i,n})\bigr)$ is plotted on $y$-axis against $v^t_i-w^t_{i,n}$ on $x$-axis. 
Note that the graph grows faster at $x\rightarrow+\infty$ than 
$x\rightarrow-\infty$, suggesting that an overestimation of $w^t_{i,n}$ will be punished
more than an underestimation. In addition, the loss function weighs more on high co-occurrence probabilities, 
as indicated by the $p^t_{i,n}$ coefficient in the equation of the limit curve (Figure~\ref{fig:loss}). 
Thus, SGNS loss tends to enforce underestimation of 
$w^t_{i,n}$ for frequent context words (as overestimation is costly), and compensate
$w^t_{i,n}$ for rare ones (i.e., overestimation on rare context words is 
affordable and will be done if necessary). This is a special property of SGNS which might have some 
smoothing effect. 

\begin{figure}[t]
\centering
\includegraphics[scale=0.34,bb=0 10 1000 320,clip]{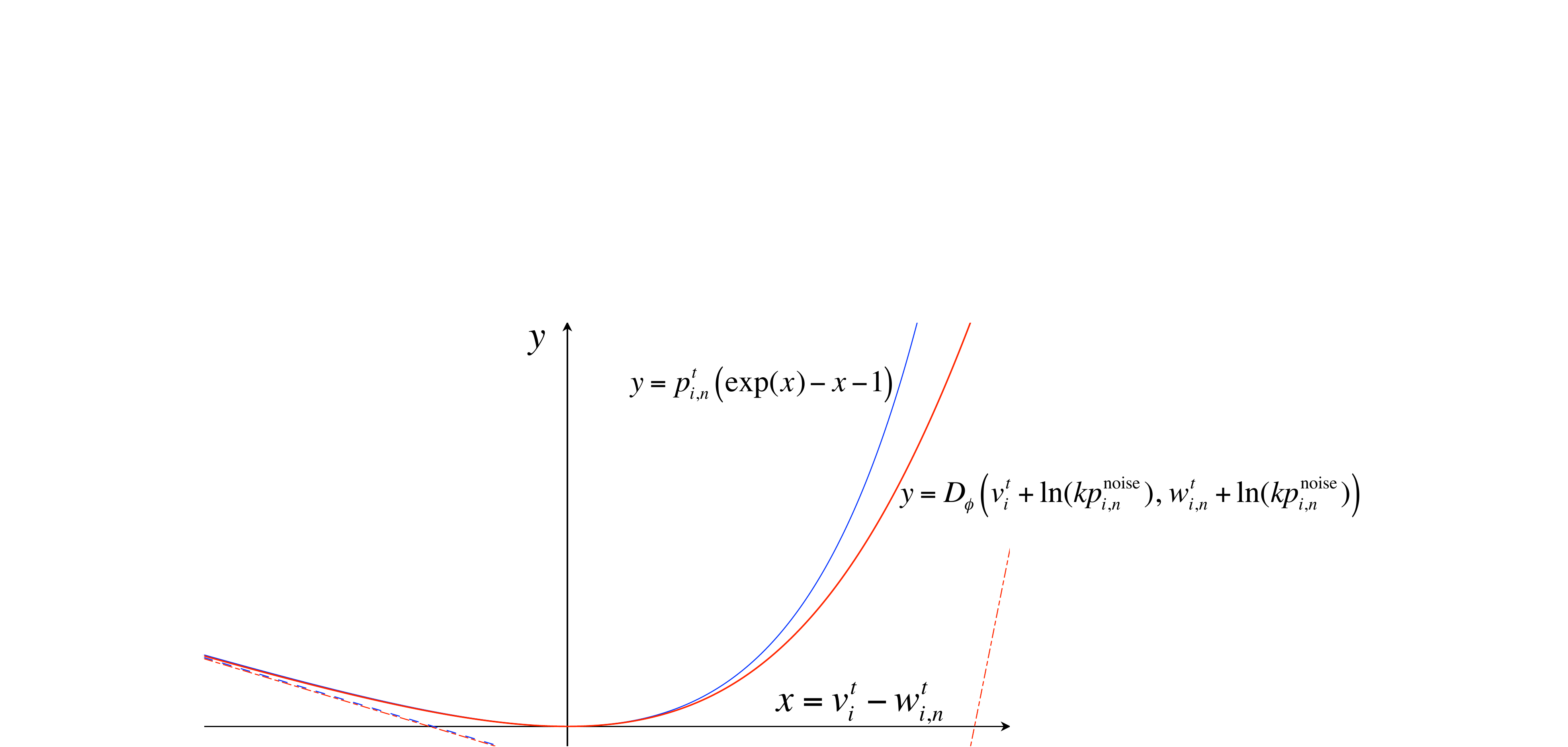}
\caption{A graph of the SGNS loss function with two asymptotes (red), and its limit curve at $k\rightarrow+\infty$ with one asymptote (blue).}
\label{fig:loss}
\end{figure}


Compared to SVD, both 
the loss functions of GloVe and SGNS 
weigh less on rare context words. As a result, 
the trained $A\mathbf{v}^t$ may fail to precisely approximate the low co-occurrence part of $\mathbf{w}^t_{n}$. 
As we discussed in Section~\ref{sec:sketch}, entries corresponding to low-frequency words 
dominate the behavior of vector representations; thus, failing to precisely approximate this part 
might hinder the inheritance of additive compositionality from high-dimensional vector representations to 
low-dimensional embeddings. 
Therefore, we conjecture that word vectors trained by GloVe or SGNS might exhibit less additive 
compositionality compared to SVD, and the composition might be less respectful to our bias bound. 

The previous discussion is only exploratory and cannot fully comply with practice because, after $\mathbf{v}^t$ 
is trained by dimension reduction, people usually re-scale the norms of all $\mathbf{v}^t$ to $1$, and then 
they use the normalized vectors in additive composition. It is not clear why this normalization step 
can usually result in better performance. 

Nevertheless, in our experiments (Section~\ref{sec:expdimred}), we find that word vectors trained by SVD 
preserve our bias bound 
well in additive composition, even after the normalization step is conducted. 
In contrast, vectors trained by GloVe or SGNS are less respectful to the bound. 
Further, in extrinsic evaluations (Section~\ref{sec:exteval}) we show that vectors trained 
by SVD can indeed be more additive compositional, as judged by human annotators. 

\section{Related Work}
\label{sec:relwork}

Additive composition is a classical approach to approximating meanings of phrases and/or 
sentences~\citep{foltz98,landauer97}. Compared to other composition operations, vector addition/average 
has either served as a strong baseline~\citep{mitchell-lapata08,takase2016}, or remained one of the 
most competitive 
methods until recently~\citep{baneaSemEval}. Additive composition has also been successfully integrated 
into several NLP systems. For example, \citet{tian14} use vector additions for assessing semantic similarities between 
paraphrase candidates in a logic-based textual entailment recognition system (e.g.~the similarity between 
``\textit{blamed for death}'' and ``\textit{cause loss of life}'' is calculated by the cosine similarity between
sums of word vectors $\mathbf{v}^{\text{blame}}+\mathbf{v}^{\text{death}}$ and 
$\mathbf{v}^{\text{cause}}+\mathbf{v}^{\text{loss}}+\mathbf{v}^{\text{life}}$); in \citet{dan}, average of 
vectors of words in a whole sentence/document is fed into a deep neural network for sentiment analysis and 
question answering, which achieves near state-of-the-art performance with minimum training time. 
There are other semantic relations handled by vector additions as well, 
such as word analogy (e.g.~the vector 
$\mathbf{v}^{\text{king}}-\mathbf{v}^{\text{man}}+\mathbf{v}^{\text{woman}}$ is close to 
$\mathbf{v}^{\text{queen}}$, suggesting ``\textit{man} is to \textit{king} as \textit{woman} is to 
\textit{queen}'', see \citealt{word2vecNAACL}), and synonymy 
(i.e.~a set of synonyms can be represented by the sum of vectors of the words in the set, see 
\citealt{rothe-schutze:2015:ACL-IJCNLP}). 
We expect all these utilities to be related to our theory of additive composition somehow, for example 
a link between additive composition and word analogy is hypothesized in 
Section~\ref{sec:wordanalogy}. Ultimately, our theory would provide new insights into previous works, 
for instance, the insights about how to construct word vectors. 

Lack of syntactic or word-order dependent effects on meaning is considered one of the most important 
issue of 
additive composition~\citep{landauer02}. Driven by 
this point of view, a number of advanced compositional frameworks have been proposed to cope with 
word order and/or syntactic information~\citep{mitchell-lapata08,zanzotto10,baroni-zamparelli10,coecke10,grefenstette-sadrzadeh:2011:EMNLP,socher12,paperno-pham-baroni14,hashimoto-EtAl:2014:EMNLP2014}. 
The usual approach is to introduce new parameters that represent different word positions or syntactic roles. 
For example, given a two-word phrase, one can first 
transform the two word vectors by different matrices and then add the results, so the two matrices are 
parameters~\citep{mitchell-lapata08}; or, regarding different syntactic roles, one can assign matrices to 
adjectives and use them to modify vectors of nouns~\citep{baroni-zamparelli10}; further, one can insert neural 
network layers between parents and children in a syntactic tree~\citep{socher12}. An empirical 
comparison of composition models 
can be found in~\citet{blacoe-lapata12}, with an accessible introduction to the literature. 
One theoretical issue of these methods, however, is the lack of learning guarantee. 
In contrast, our proposal of the Near-far Context demonstrates that word order can be handled within an 
additive compositional framework, being parameter-free and with a proven bias bound. Recently, 
\citet{tian2016} further extended additive composition to realizing a formal semantics framework. 

From a wider perspective, constructing and composing vector representations for linguistic sequences 
have become one of the central techniques in NLP, and a lot of approaches have been explored. 
Some of them, such as the vectors constructed from probability ratios and composed by 
multiplications~\citep{mitchell10}, might still be related to additive composition because by taking 
logarithm, multiplications become additions and probability ratios become PMIs. Other composition 
methods range from circular convolution~\citep{mitchell10} to neural networks such as 
recursive autoencoder~\citep{socher11} and long short-term memory~\citep{context2vec}. 
Word vectors can be trained jointly with composition 
parameters~\citep{hashimoto-EtAl:2014:EMNLP2014,pham-EtAl:2015:ACL-IJCNLP}, 
and training signals range from surrounding context words~\citep{takase2016} to 
supervised labels~\citep{cwmodel}. 
We believe it is also important to investigate the theoretical aspects of these approaches, 
which remain largely unclear. 
As for word vectors, some theoretical works have been done on explaining the errors of 
dimension reductions of PMI vectors~\citep{TACL742,TACL809}.

Error bounds in approximation schemes have been extensively studied in statistical 
learning theory~\citep{Vapnik:1995,Gnecco:2008}, and especially for neural 
networks~\citep{Niyogi:98,Burger2001}. 
Since we have formalized compositional frameworks as approximation schemes, there is a 
good chance to apply the theories of approximation error bounds to this problem, especially for 
advanced compositional frameworks that have many parameters. Though the theories are usually 
established on general settings, we see a great potential in using properties that are specific to 
natural language data, as we demonstrate in this work. 

There have been consistent efforts toward understanding stochastic behaviors of natural language. 
Zipf's Law~\citep{zipf35} and its applications~\citep{kobayashi14}, non-parametric Bayesian 
language models such as the Hierarchical Pitman-Yor 
Process~\citep{teh:2006:COLACL}, and the topic model~\citep{blei12} might further help refine our theory. 
For example, it can be fruitful to consider additive composition of topics. 

\section{Experimental Verification}
\label{sec:expveri}

In this section, we conduct experiments on the British National Corpus (BNC) \citep{bnc} to verify assumptions 
and predictions of our theory. The corpus contains about 100M word tokens, including written 
texts and utterances in British English. For constructing vector representations we use lemmatized words 
annotated in the corpus, and for counting co-occurrences we use context windows that do not cross sentence 
boundaries. 
The size of the context windows is $5$ to each side for a target word, and $4$ for a target phrase. We 
extract all unigrams, ordered and unordered bigrams occurring more than 
$200$ times as targets. This results in 16,210 unigrams, 45,793 ordered bigrams and 45,398 unordered bigrams. 
For the lexicon of context words we use the same set of unigrams. 

\subsection{Test of Independence}
\label{sec:indeptest}

\begin{figure}[t]
\centering
\includegraphics[scale=0.375,bb=0 0 980 200,clip]{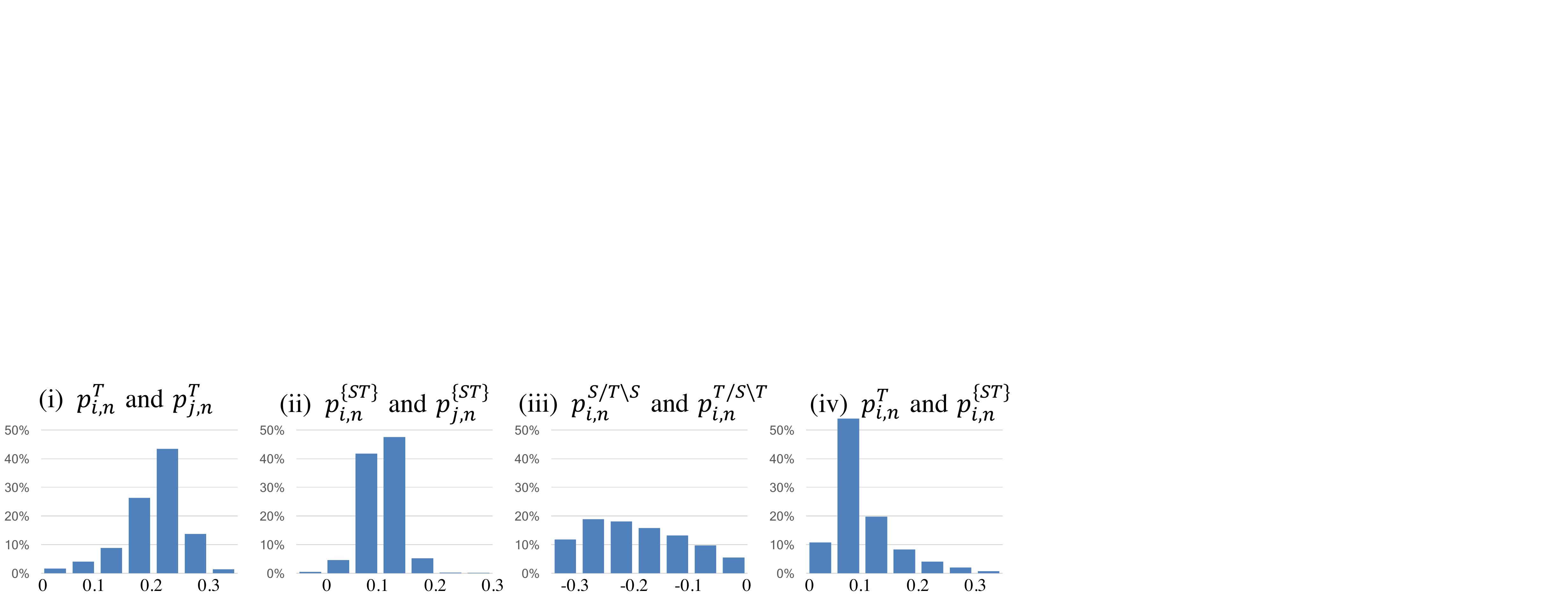}
\caption{Distributions of Spearman's $\rho$ between different pairs of random variables}
\label{fig:distrho}
\end{figure}

In order to test the independence assumptions in our theory, 
we use Spearman's $\rho$ to measure the correlations between random variables. 
Spearman's $\rho$ is the Pearson correlation between rank values, and is invariant under 
transformations of any monotonic function. One has $-1\leq\rho\leq -1$, and 
if two variables are independent, $\rho$ should be close to $0$. 

In our theory, Assumption (B1) of Theorem~\ref{thm:main} states that $p^\Upsilon_{i,n}$ and 
$p^\Upsilon_{j,n}$ are independent for each $1\leq i < j \leq n$. To test, we calculate the 
Spearman's $\rho$ between (i) $p^T_{i,n}$ and $p^T_{j,n}$, and (ii) $p^{\{ST\}}_{i,n}$ and $p^{\{ST\}}_{j,n}$, 
where $T$ and $\{ST\}$ vary on the 16,210 unigrams and 45,398 unordered bigram samples respectively. 
Further, Assumption (C) of Theorem~\ref{thm:main} states that for each $1\leq i \leq n$, the random variables 
$p^{S/T\backslash S}_{i,n}$ and $p^{T/S\backslash T}_{i,n}$ 
are independent, whereas 
$F(p^{S/T\backslash S}_{i,n}\!+\!1/n)$ and $F(p^{\{ST\}}_{i,n}\!+\!1/n)$ have positive correlation. 
Thus, we check the Spearman's $\rho$ between (iii) $p^{S/T\backslash S}_{i,n}$ and $p^{T/S\backslash T}_{i,n}$, 
and (iv) $p^{S/T\backslash S}_{i,n}$ and $p^{\{ST\}}_{i,n}$, where $\{S, T\}$ vary on the 
45,389 unordered bigrams. The results are summarized in Figure~\ref{fig:distrho}. 

For most $i$-$j$ pairs, Figure~\ref{fig:distrho}(i)(ii) suggest that the correlations between $p^\Upsilon_{i,n}$ and 
$p^\Upsilon_{j,n}$ are positive but quite weak (for 70\% of the $i$-$j$ pairs, the Spearman's $\rho$ between 
$p^T_{i,n}$ and $p^T_{j,n}$ is $0.2\pm 0.05$; and for 90\% pairs the Spearman's $\rho$ between 
$p^{\{ST\}}_{i,n}$ and $p^{\{ST\}}_{j,n}$ is $0.1\pm 0.05$). 
As a comparison, when $i$ and $j$ indicate a pair of semantically related context words such as 
\emph{black} and \emph{white}, the Spearman's $\rho$ between 
$p^T_{i,n}$ and $p^T_{j,n}$ is $0.40$ and between $p^{\{ST\}}_{i,n}$ and $p^{\{ST\}}_{j,n}$ is $0.31$. 
Such examples only contribute to a negligible portion of the whole $i$-$j$ pairs, because semantically related 
pairs are rare. 

On the other hand, Figure~\ref{fig:distrho}(iii) shows that $p^{S/T\backslash S}_{i,n}$ and 
$p^{T/S\backslash T}_{i,n}$ have \emph{negative} correlation for most $i$; the Spearman's $\rho$ 
for 66\% of the $p^{S/T\backslash S}_{i,n}$-and-$p^{T/S\backslash T}_{i,n}$ pairs is $-0.2\pm 0.1$. 
In addition, Figure~\ref{fig:distrho}(iv) confirms that $p^{S/T\backslash S}_{i,n}$ and $p^{\{ST\}}_{i,n}$ have 
positive correlation. 

Now, can the observed weak correlations support our assumptions on independence? 
To test this, one may calculate a $p$-value as the probability of a Spearman's $\rho$ being farther from 
$0$ than the observation, making use of the fact that $\rho\sqrt{\frac{N-2}{1-\rho^2}}$ approximately 
follows a Student's $t$-distribution (where $N$ is the sample size). 
If the $p$-value is small, the Spearman's $\rho$ should be considered too far from $0$ to support 
independence. However, this test turns out to be overly strict; for unordered bigrams (i.e.~$N=45398$), 
one needs $|\rho|<0.012$ to make $p>0.01$. In other words, since the sample size 
is huge, even weak correlations among samples can manifest as evidence for rejecting the independence as 
null hypothesis. 

Nevertheless, our theoretical analysis is still valid, because the Law of Large Numbers 
holds even for weakly correlated random variables, and the fact that $p^{S/T\backslash S}_{i,n}$ and 
$p^{T/S\backslash T}_{i,n}$ are \emph{negatively} correlated does not change the direction of our 
proven inequality. Therefore, our independence assumptions are oversimplifications for language modeling, 
but the theoretical conclusions and the bias bound are still likely true. 

\begin{figure}[t]
\centering
\includegraphics[scale=0.35,bb=0 0 1000 270,clip]{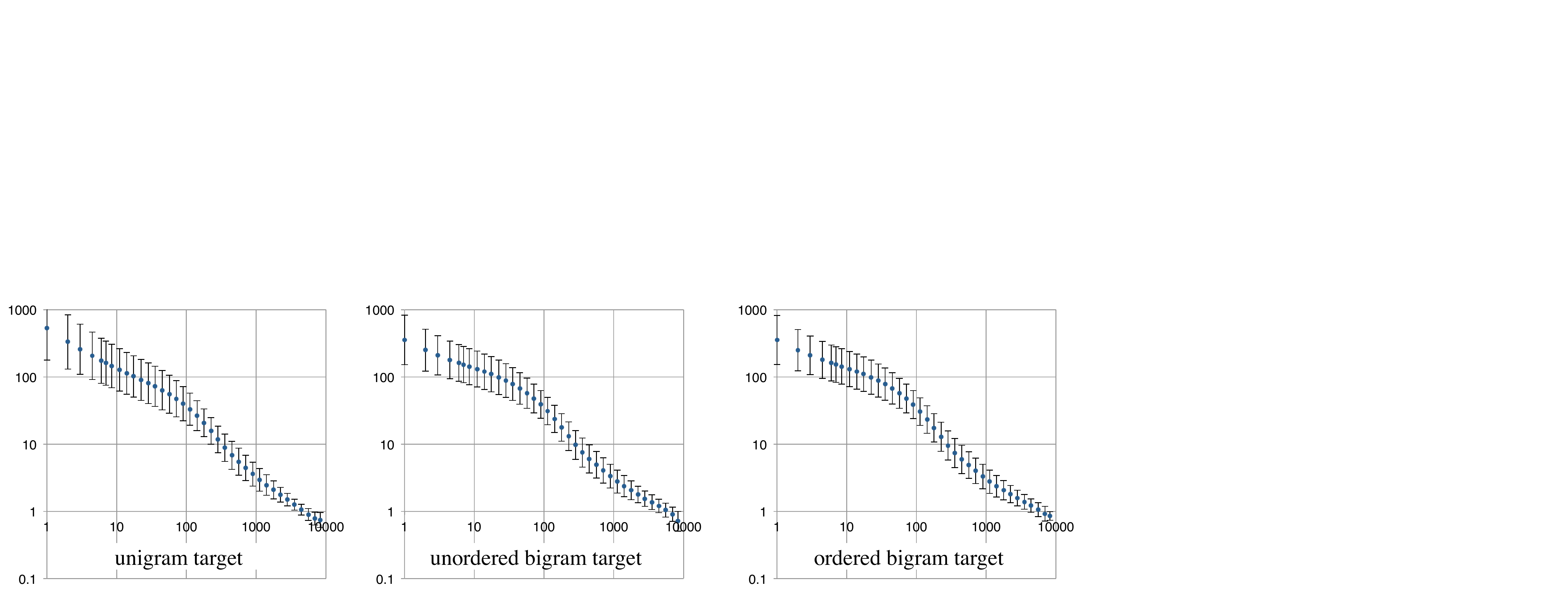}
\caption{For each $x$ coordinate, the log-log graphs show the average value of the $x$-th largest 
probability ratios $p^\Upsilon_{i,n}/p_{i,n}$ on $y$-axis. The ranking is taken among 
$1\leq i\leq n$ with $\Upsilon$ fixed, and the average is taken across different $\Upsilon$ that are 
unigrams, unordered bigrams, or ordered bigrams respectively. Standard deviation is shown as error bar.}
\label{fig:genzipf}
\end{figure}

\subsection{Generalized Zipf's Law}
\label{sec:genzipf}

Consider the probability ratio $p^\Upsilon_{i,n}/p_{i,n}$, where $\Upsilon$ can be a unigram, 
ordered bigram or unordered bigram. Assumption (B) of Theorem~\ref{thm:main} states that 
$(p^\Upsilon_{i,n}/p_{i,n})$'s ($1\leq i\leq n$, $n$ fixed) can be viewed as independent sample points drawn from 
distributions that have a same power law tail of index $1$. We verify this assumption in the following. 

A power law distribution has two parameters, the index $\alpha$ and the lower bound $m$ of the power 
law behavior. If a random variable $X$ obeys a power law, the probability of $x\leq X$ conditioned 
on $m\leq X$ is given by 
\begin{equation}
\label{eq:powerlawdef}
\mathbb{P}(x\leq X | m\leq X)=\frac{m^\alpha}{x^\alpha}. 
\end{equation}

For each fixed $\Upsilon$, we estimate $\alpha$ and $m$ from the sample 
$p^\Upsilon_{i,n}/p_{i,n}$ $(1\leq i\leq n)$, 
using the method of \citet{powerlaw}. Namely, $\alpha$ is estimated by maximizing the likelihood of the 
sample, and $m$ is sought to minimize the Kolmogorov-Smirnov 
statistic, which measures how well the theoretical distribution \eqref{eq:powerlawdef} fits the empirical 
distribution of the sample. After $m$ is estimated, we plot all $p^\Upsilon_{i,n}/p_{i,n}$ greater than 
$m$ in a log-log graph, against their ranking. If the sample points are drawn from a power law, 
the graph will be a straight line. Since Assumption (B) states that the power law tail is 
the same for all $\Upsilon$ and has index $1$, we should obtain the same straight line for all $\Upsilon$, 
and the slope of the line should be $-1$. 

In Figure~\ref{fig:genzipf}, we summarize the graphs described above for all $\Upsilon$. 
More precisely, we plot the ranked probability ratios for each fixed $\Upsilon$ into the same log-log 
graph, and then show the average and standard deviation of the $x$-th largest probability ratios across 
different $\Upsilon$. The figure shows that for each target type, most data points lie within a narrow stripe 
of roughly the same shape, suggesting that the distribution of probability ratios for each fixed $\Upsilon$ 
is approximately the same. In addition, the shape can be roughly approximated by a straight line with 
slope $-1$, which suggests that the distribution is power law of index $1$, verifying Assumption (B). 

\begin{figure}
\centering
\begin{minipage}{.38\textwidth}
  \centering
  \includegraphics[scale=0.36,bb=0 0 340 270,clip]{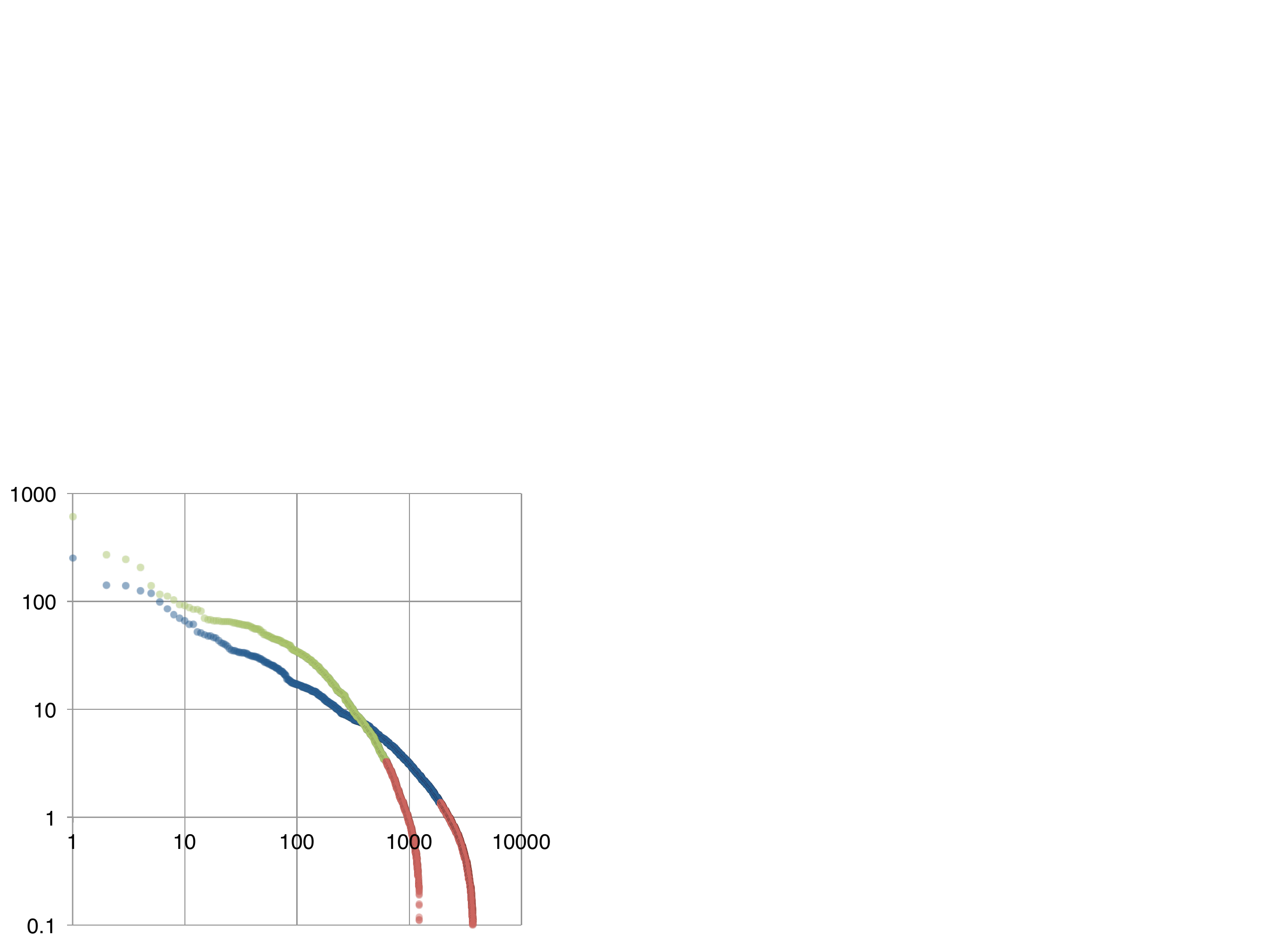}
  \captionof{figure}{Log-log plot of probability ratios against ranks, for two fixed targets.}
  \label{fig:genzipfsample}
\end{minipage}%
\begin{minipage}{.02\textwidth}
  ~
\end{minipage}%
\begin{minipage}{.60\textwidth}
  \centering

\setlength{\tabcolsep}{4pt}
\begin{tabular}{| c | c | c | c | c | c | c | c |}
\hline
 $i$ & word & \multicolumn{5}{ c |}{frequencies} & $p$-value \\
\hline
1 & \emph{the} & 16210 & 0 & 0 & 0 & 0 & $< 0.0001$ \\
2 & \emph{be} & 16210 & 0 & 0 & 0 & 0 & $< 0.0001$ \\
\hline
101 & \emph{between} & 16167 & 29 & 14 & 0 & 0 & $0.0010$ \\
121 & \emph{off} & 16173 & 29 & 6 & 2 & 0 & $0.0003$ \\
142 & \emph{under} & 16169 & 28 & 9 & 3 & 1 & $0.0059$ \\
\hline
3075 & \emph{evident} & 15992 & 194 & 23 & 1 & 0 & $< 0.0001$ \\
3076 & \emph{refugee} & 15914 & 176 & 77 & 27 & 16  & $0.0002$ \\
3077 & \emph{button} & 15969 & 167 & 43 & 18 & 13 & $< 0.0001$ \\
3078 & \emph{belt} & 15920 & 206 & 54 & 24 & 6 & $< 0.0001$ \\
3079 & \emph{vegetable} & 15859 & 194 & 93 & 39 & 25 & $0.0126$ \\
3080 & \emph{expertise} & 16053 & 124 & 29 & 4 & 0 & $< 0.0001$ \\
\hline
\end{tabular}
\captionof{table}{$\chi^2$-tests on distributions of $X:=p^T_{i,n}/p_{i,n}$} 
\label{tab:chisq}

\end{minipage}
\end{figure}

As a concrete example, in Figure~\ref{fig:genzipfsample} we show a log-log graph of the $x$-th 
largest probability ratios $p^s_{i,n}/p_{i,n}$ and $p^t_{i,n}/p_{i,n}$ $(1\leq i\leq n)$, where 
$s$ and $t$ are two individual word targets. The red points are cut off because their $y$ values are lower than 
the boundaries of power law behavior estimated from data. The blue and green points are the power law 
tails. 

Further, to document this Zipfian behavior quantitatively, we conduct a $\chi^2$-test on the distribution of 
$p^T_{i,n}/p_{i,n}$. In this test, we fix each $i$ and categorize the values of $X:=p^T_{i,n}/p_{i,n}$, where $T$ 
varies over the 16,210 unigram samples. According to Figure~\ref{fig:genzipf} and 
Figure~\ref{fig:genzipfsample}, we assume that $X$ has a power law tail starting from $X \geq 2^4$. 
Thus, we divide values of $X$ into 5 categories, being $X<2^4$, $2^{4+k}\leq X < 2^{5+k}$ $(k=0,1,2)$, and 
$X\geq 2^7$, respectively. We count frequencies in each category and choose a parameter 
$\frac{1}{2^4}\leq m\leq\frac{1}{2}$ by minimizing the $\chi^2$ statistic. The parameter $\alpha$ is fixed to 1. 
Then, the degree of freedom is calculated as $5-1-1=3$, and 
the $\chi^2$-test produces a $p$-value indicating how good 
the power law hypothesis fit to the observed frequencies. We decide that the test is passed if $p\geq 0.0001$. 

Among all indices $1\leq i \leq n$, there are 16\% distributions passing the test. A selection of examples is 
shown in Table~\ref{tab:chisq}. 
It turns out that many function words, such as ``\emph{the}'' and ``\emph{be}'' cannot pass the test 
(with all values of $X$ less than $2^4$), because 
the occurring probabilities of these words do not change much, whether or not conditioned on a target. 
An exception is that several prepositions, such as ``\emph{between}'' and ``\emph{under}'' do pass. 
On the other hand, as $i$ becomes larger (i.e.~$p_{i,n}$ becomes smaller), more of the distributions of 
$p^T_{i,n}/p_{i,n}$ become distorted, similar to the green dots in Figure~\ref{fig:genzipfsample} 
which will fail the test. As Table~\ref{tab:chisq} suggests, 
no obvious linguistic factor seems able to explain which word would pass.
However, Figure~\ref{fig:genzipf} still confirms that the averaged behavior of these distributions obeys a 
power law. 

\subsection{The Choice of Function $F$}
\label{sec:expfunctionF}

In this section we experimentally verify the effects brought by different function $F$. Recall that $F$ is 
parameterized by $\lambda$ as defined in Definition~\ref{defn:fdef}. 
In Section~\ref{sec:effectF}, we have shown that $\mathbb{E}[F(X)^2]<\infty$ is a sufficient and necessary 
condition for the norms of natural phrase vectors to converge to $1$. If $X$ has a power law tail of index 
$\alpha$, then the condition for $\mathbb{E}[F(X)^2]<\infty$ is $\lambda<\alpha/2$. So if we construct 
vector representations with different $\lambda$, only those vectors satisfying $\lambda<\alpha/2$ 
will have convergent norms. We verify this prediction first. 

\begin{figure}
\centering
\begin{minipage}{.52\textwidth}
  \centering
  \includegraphics[scale=0.4,bb=0 0 480 240,clip]{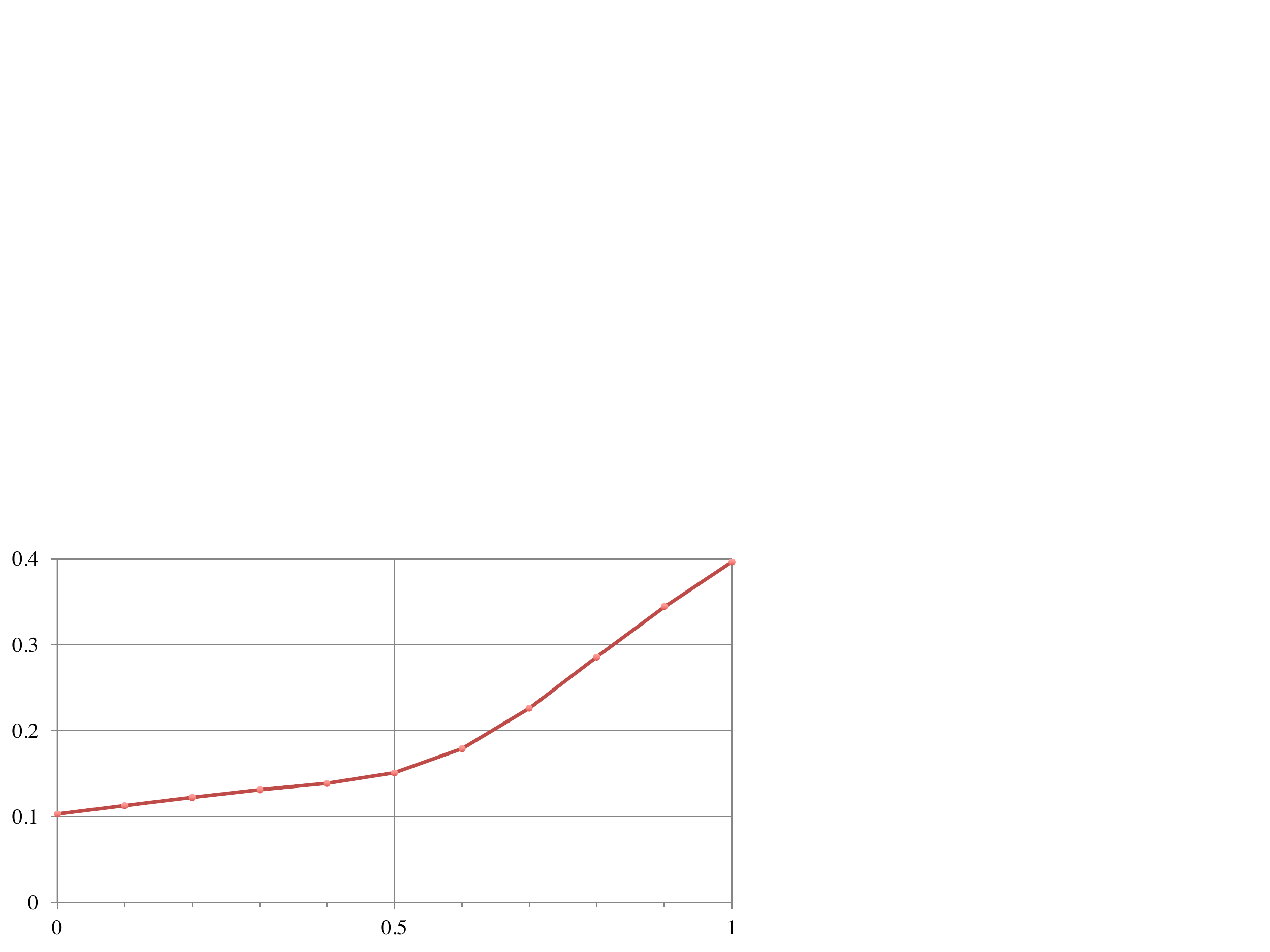}
  \captionof{figure}{Plot of standard deviation of norms of natural phrase vectors against different $\lambda$.}
  \label{fig:powergraph}
\end{minipage}%
\begin{minipage}{.02\textwidth}
  ~
\end{minipage}%
\begin{minipage}{.46\textwidth}
  \centering
  \includegraphics[scale=0.4,bb=0 0 280 240,clip]{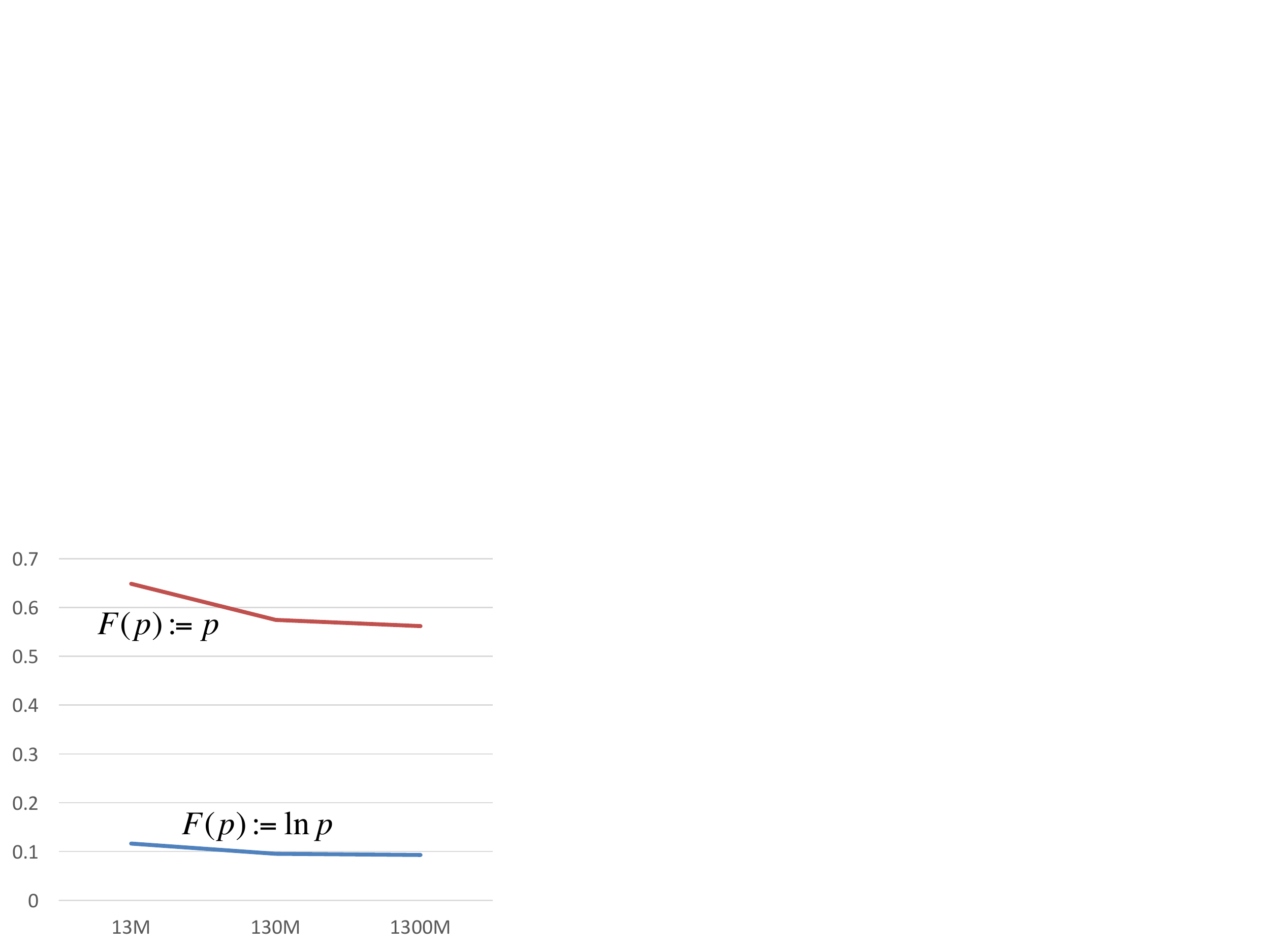}
  \captionof{figure}{Plot of standard deviation of norms against number of tokens in corpus.}
  \label{fig:wikicorpus}
\end{minipage}
\end{figure}

In Figure~\ref{fig:powergraph}, we plot the standard deviation of the norms of natural phrase vectors on 
$y$-axis, against different $\lambda$ values used for constructing the vectors. We tried 
$\lambda=0,0.1,\ldots,1$. As the graph shows, as long as $\lambda< 0.5$, most of the norms lie within the 
range of $1\pm 0.1$. In contrast, the observed standard deviation quickly explodes as $\lambda$ 
gets larger. In addition, the transition point appears to be slightly larger than $0.5$, which complies with 
the fact that the observed $\alpha$ is slightly larger than $1$ (i.e., the slope $-1/\alpha$ of the power law 
tails in Figure~\ref{fig:genzipf} and Figure~\ref{fig:genzipfsample} appear to be slightly more gradual than $-1$). 

To confirm that the above observation represents a general principle across different corpora, we also conduct 
experiments on English Wikipedia\footnote{\url{https://dumps.wikimedia.org}}. 
We use WikiExtractor\footnote{\url{https://github.com/attardi/wikiextractor}} to extract texts from a 
2015-12-01 dump, 
and Stanford CoreNLP\footnote{\url{http://stanfordnlp.github.io/CoreNLP/}} for sentence splitting. 
The corpus has 1300M word tokens (about 13 times the size of BNC), and we use words in their surface 
forms instead of lemmas. We extract words and unordered bigrams which occur more than 500 times, 
resulting in about 85K words and 264K bigrams. Then, we additionally make two smaller corpora 
by uniformly sampling 10\% and 1\% sentences in Wikipedia. For each corpus, we construct natural phrase 
vectors and 
calculate the standard deviation of their norms as previous. The results are shown in Figure~\ref{fig:wikicorpus}. 
Again, we found that when one sets $F(p):=\ln p$, the standard deviation is around 0.1; in contrast when 
$F(p):=p$, the standard deviation is above 0.5. As the corpus increases, the standard 
deviation slightly decreases; at Wikipedia's full size, the standard deviation for $F(p):=\ln p$ 
descends below 0.095. 

Next, we investigate how $F$ affects the Euclidean distance $\mathcal{B}^{\{st\}}_n$. In Figure~\ref{fig:funcF}, 
we plot 
$$
\mathcal{B}^{\{st\}}_n \text{ on $y$-axis,} \quad\text{ against }\quad
\sqrt{\frac{1}{2}(\pi_{s/t\backslash s}^2+\pi_{t/s\backslash t}^2+\pi_{s/t\backslash s}\pi_{t/s\backslash t})}
\text{ on $x$-axis}, 
$$
for every unordered bigram $\{st\}$. We tried four different choices of function $F$, as indicated above the 
graphs. For the choices (a) $F(p):=\ln{p}$ and (b) $F(p):=\sqrt{p}$, we verify the upper bound $y\leq x$ 
as suggested by 
Claim~\ref{claim:biasbound}. In contrast, the approximation errors seem no longer bounded 
when (c) $F(p):=p$ or (d) $F(p):=p\ln{p}$. 

In Section~\ref{sec:exteval} we will extrinsically evaluate the additive compositionality of vector representations, and 
find $F$ a crucial factor there; while $F(p):=\ln{p}$ and $F(p):=\sqrt{p}$ 
evaluate similarly well, $F(p):=p$ and $F(p):=p\ln{p}$ do much worse. This suggests that our bias bound indeed 
has the power of predicting additive compositionality, demonstrating the usefulness of our theory. 
In contrast, it seems that the average level of approximation errors for observed bigrams 
(shown as green dashed lines in Figure~\ref{fig:funcF}) is less predictive, as the poor choices 
$F(p):=p$ and $F(p):=p\ln{p}$ actually have lower average error levels. 
This emphasizes a particular caveat that, choosing 
composition operations by minimizing the observed average error may not always 
be justifiable. Here if we consider the function $F$ as a parameter in additive composition, 
and choose the one with the lowest average error observed, we will get the worst setting 
$F(p):=p$. Therefore, we see how important a learning theory for composition research is. 

\begin{figure}[t]
\centering
\includegraphics[scale=0.158,bb=0 0 2200 620,clip]{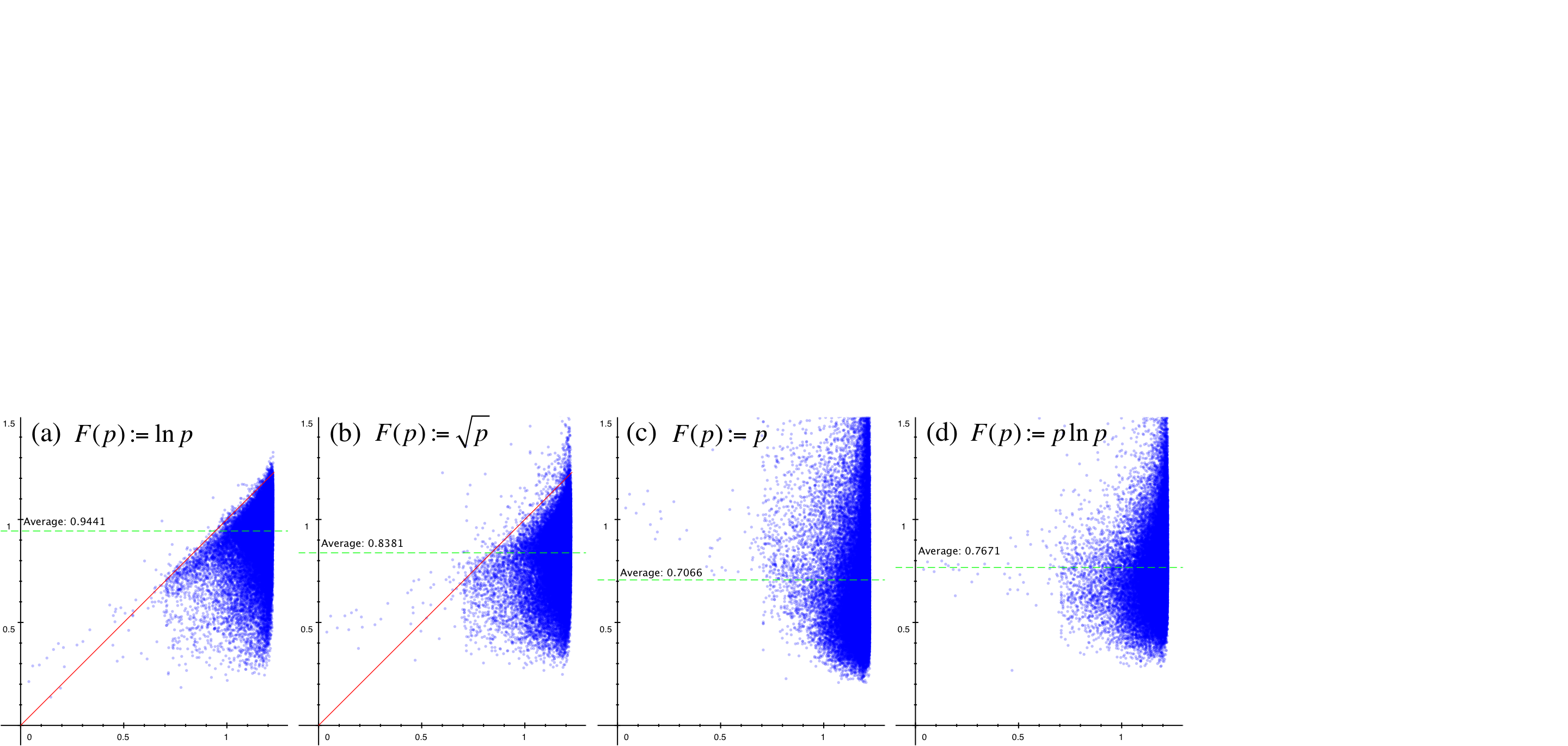}
\caption{Approximation errors for unordered bigrams observed in BNC. The choice of $F$ is shown above 
each graph. The theoretical upper bound $y\leq x$ is drawn as red solid lines in (a) and (b). The average error 
levels are drawn as green dashed lines.}
\label{fig:funcF}
\end{figure}

\subsection{Handling Word Order in Additive Composition}
\label{sec:expnearfar}

For vector representations constructed from the Near-far Contexts (Section~\ref{sec:wordorder}), we have a 
similar bias bound given by Claim~\ref{claim:nearfar}. 
In this section, we experimentally verify the bound 
and qualitatively show that the additive composition of 
Near-far Context vectors can be used for assessing semantic similarities between \emph{ordered} bigrams. 

In Figure~\ref{fig:nearfar_error_1} and Figure~\ref{fig:nearfar_error_2}, we plot 
\begin{multline*}
\quad\text{(a) } \mathcal{B}^{st}_n \quad\text{and}\quad \text{(b) } 
\lVert \mathbf{w}^{ts}_n-\frac{1}{2}(\mathbf{w}^{s\bullet}_n + \mathbf{w}^{\bullet t}_n) \rVert 
\quad\text{on $y$-axis}, \\
\text{against}\quad
\sqrt{\frac{1}{2}(\pi_{{s\bullet}\backslash t}^2+\pi_{s/{\bullet t}}^2+\pi_{{s\bullet}\backslash t}\pi_{s/{\bullet t}})}
  \text{ on $x$-axis}, 
\end{multline*}
for every ordered bigram $st$. We tried two settings of $F$, namely 
$F(p):=\ln{p}$ (Figure~\ref{fig:nearfar_error_1}) and $F(p):=\sqrt{p}$ (Figure~\ref{fig:nearfar_error_2}). 
In both cases, the approximation errors in (a) are bounded by $y\leq x$ (red solid lines) as suggested 
by Claim~\ref{claim:nearfar}. In contrast, the approximation errors for order-reversed bigrams exceed 
this bound, showing that the additive composition of Near-far Context vectors actually recognizes word order. 

In Table~\ref{tab:wordorder}, 
we show the 8 nearest word pairs for each of 8 ordered bigrams, measured by cosine similarities between 
additive compositions of Near-far Context vectors. More precisely, for word pairs ``$s_1$ $t_1$'' and 
``$s_2$ $t_2$'', we calculate the cosine similarity between 
$\frac{1}{2}(\mathbf{v}^{s_1\bullet} + \mathbf{v}^{\bullet t_1})$ and 
$\frac{1}{2}(\mathbf{v}^{s_2\bullet} + \mathbf{v}^{\bullet t_2})$, where $\mathbf{v}^{s\bullet}$ and 
$\mathbf{v}^{\bullet t}$ are normalized 
$200$-dimensional SVD reductions of $\mathbf{w}^{s\bullet}_n$ and $\mathbf{w}^{\bullet t}_n$, respectively, with 
$F(p):=\sqrt{p}$. 
The table shows that additive composition of Near-far Context vectors can indeed represent meanings of 
ordered bigrams, 
for example, ``\textit{pose problem}'' is near to ``\textit{arise dilemma}'' but not to ``\textit{dilemma arise}'', 
and ``\textit{problem pose}'' is near to ``\textit{difficulty cause}'' but not to ``\textit{cause difficulty}''. 
It is also noteworthy that ``\textit{not enough}'' is similar to ``\textit{always want}'', showing some degree 
of semantic compositionality beyond word level. 
We believe this ability of computing meanings of arbitrary ordered 
bigrams is already highly useful, because only a 
few bigrams can be directly observed from real corpora. 

\begin{figure}
\centering
\begin{minipage}{.49\textwidth}
  \centering
  \includegraphics[scale=0.1,bb=0 0 1720 960,clip]{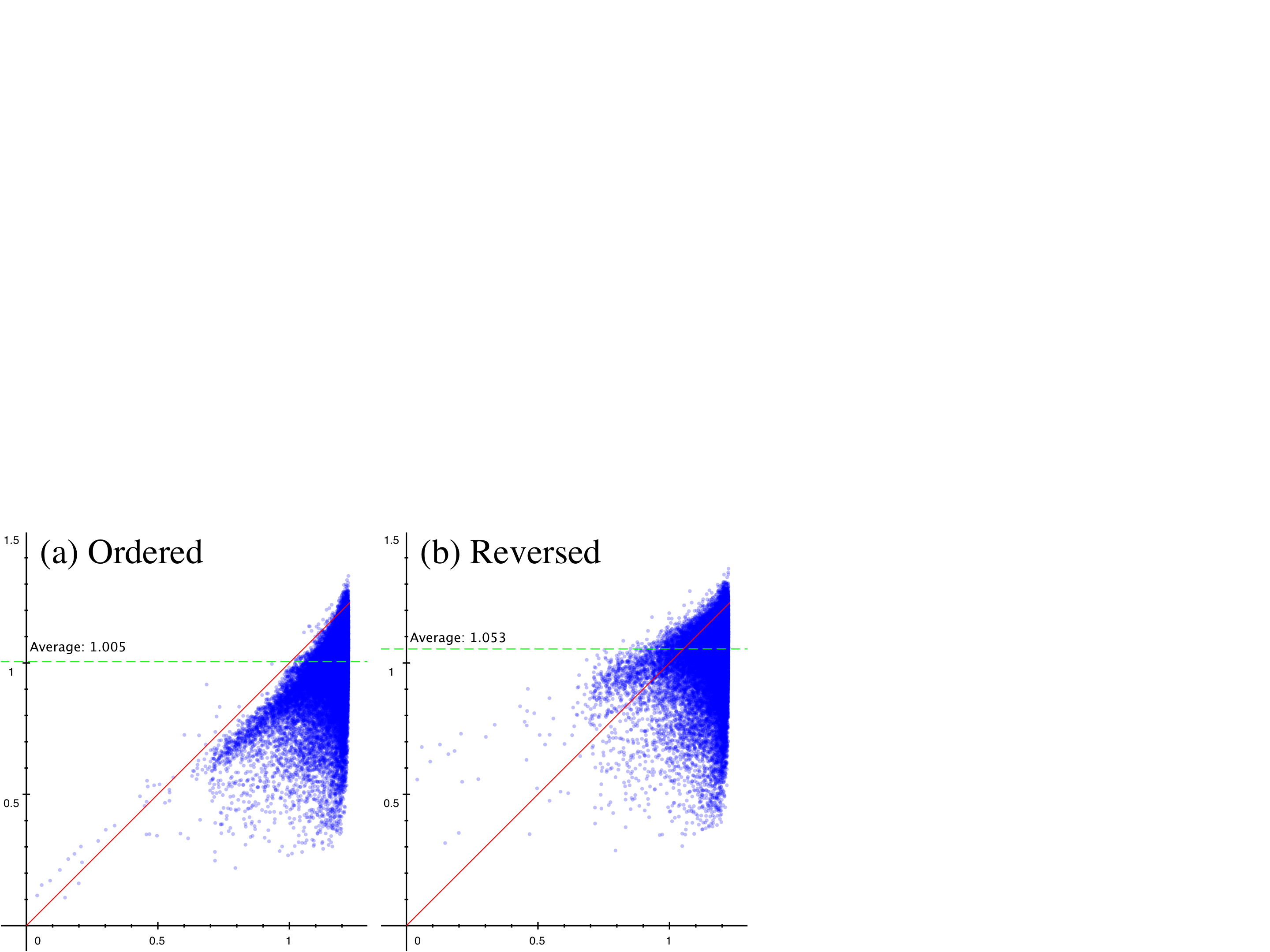}
  \captionof{figure}{Near-far Context, $F(p):=\ln{p}$.}
  \label{fig:nearfar_error_1}
\end{minipage}%
\begin{minipage}{.02\textwidth}
  ~
\end{minipage}%
\begin{minipage}{.49\textwidth}
  \centering
  \includegraphics[scale=0.1,bb=0 0 1720 960,clip]{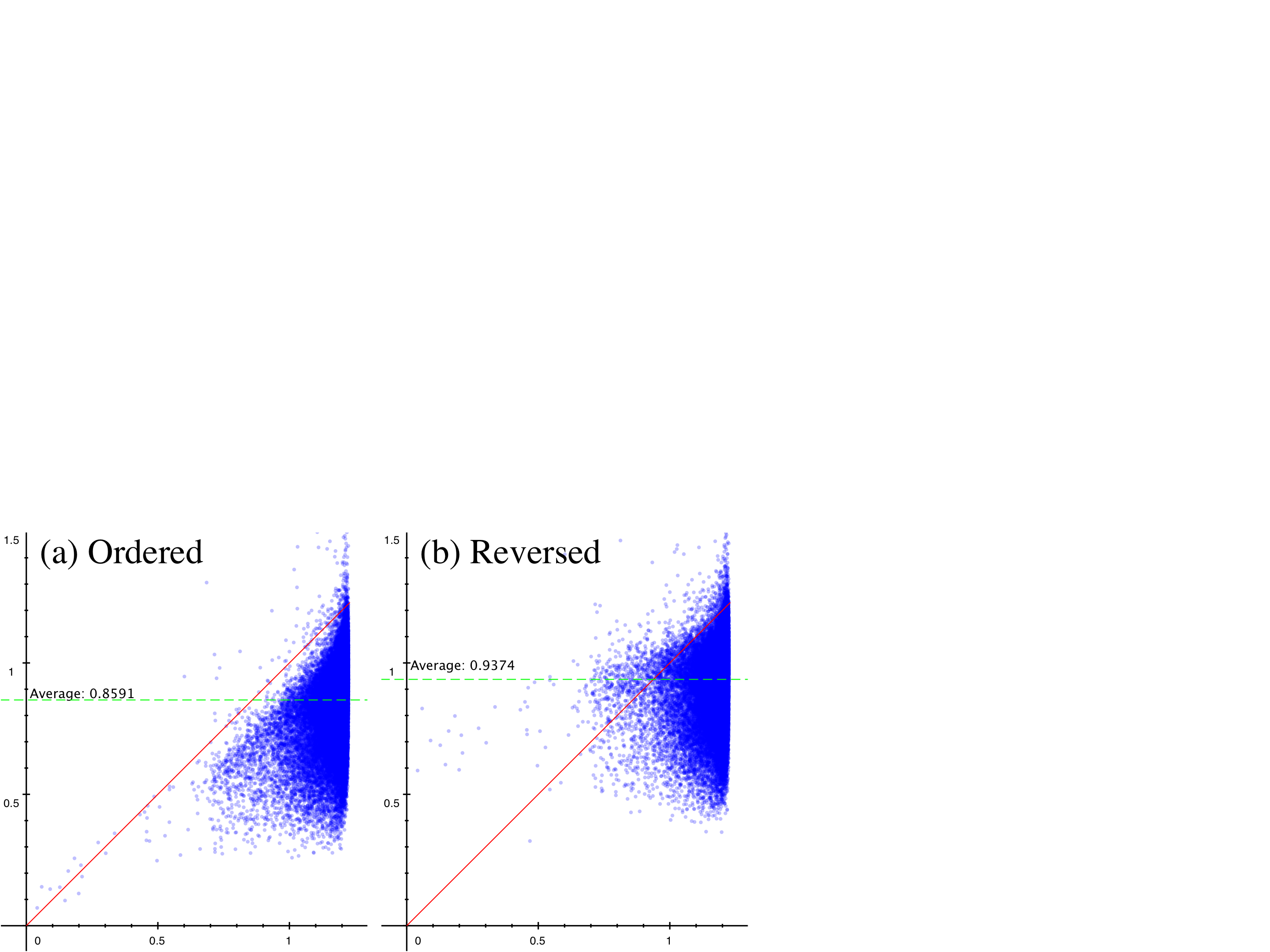}
  \captionof{figure}{Near-far Context, $F(p):=\sqrt{p}$.}
  \label{fig:nearfar_error_2}
\end{minipage}
\end{figure}

\begin{table}[t]
\centering
\renewcommand{\arraystretch}{1.2}
\setlength{\tabcolsep}{8pt}
\it
\begin{tabular}{| c | c ? c | c | }
\hline
\textbf{pose problem} & \textbf{problem pose} & \textbf{tax rate} & \textbf{rate tax} \\
\hline
solve dilemma & difficulty solve & income price & {\rm income inflation} \\
arise dilemma & difficulty cause & income inflation & premium taxation \\
solve difficulty & difficulty tackle & taxation premium & premium inflation \\
solve concern & tendency solve & income premium & price income \\
cause dilemma & solution cause & {\rm inflation income} & {\rm taxation premium} \\
tackle difficulty & dilemma cause & taxation price & inflation income \\
{\rm dilemma serious} & shortage solve & {\rm premium taxation} & earnings taxation \\
confront difficulty & consequence solve & inflation premium & premium income \\
\hline
\hline
\textbf{high price} & \textbf{price high} & \textbf{not enough} & \textbf{enough not} \\
\hline
 low rate & rate low & really sufficient & too never \\
 low premium & level low & insufficient bother & {\rm really never} \\
 low output & value low & still bother & too really \\
 low value & cost low & always want & {\rm ought too} \\
 low cost & premium low & always bother & too actually \\
 low wage & output low & really prepared & too always \\
 low level & inflation low & really unwilling & sufficient never \\
 low margin & market low & really obliged & quite never \\
\hline
\end{tabular}
\rm
\caption{Top 8 similar word pairs, assessed by cosine similarities between additive compositions of 
Near-far Context vectors.}
\label{tab:wordorder}
\end{table}

\subsection{Dimension Reduction}
\label{sec:expdimred}

In this section, we verify our prediction in Section~\ref{sec:dimensionreduction} that vectors trained by SVD 
preserve our bias bound more faithfully than GloVe and SGNS. In 
Figure~\ref{fig:dimred}, we use 
\emph{normalized} word vectors $\mathbf{v}^t$ 
that are constructed from the distributional vectors $\mathbf{w}^t_n$ by reducing to $200$ dimensions using 
different reduction methods. We use SVD in (a) and (b), with $F(p):=\ln{p}$ in (a) and $F(p):=\sqrt{p}$ in (b). 
The GloVe model is shown in (c) and SGNS in (d), both of them using $F(p):=\ln{p}$. For each unordered 
bigram $\{st\}$ we plot 
$$
\lVert \mathbf{v}^{\{st\}}-\frac{1}{2}(\mathbf{v}^{s}+\mathbf{v}^{t}) \rVert
\text{ on $y$-axis,} \quad\text{ against }\quad
\sqrt{\frac{1}{2}(\pi_{s/t\backslash s}^2+\pi_{t/s\backslash t}^2+\pi_{s/t\backslash s}\pi_{t/s\backslash t})}
\text{ on $x$-axis}. 
$$
The graphs show that 
vectors trained by SVD still largely conform to our bias bound $y\leq x$ (red solid lines), but vectors trained by 
GloVe or SGNS no longer do. Our extrinsic evaluations in Section~\ref{sec:exteval} also show that SVD might 
perform better than GloVe and SGNS. 

\begin{figure}[t]
\centering
\includegraphics[scale=0.158,bb=0 0 2200 620,clip]{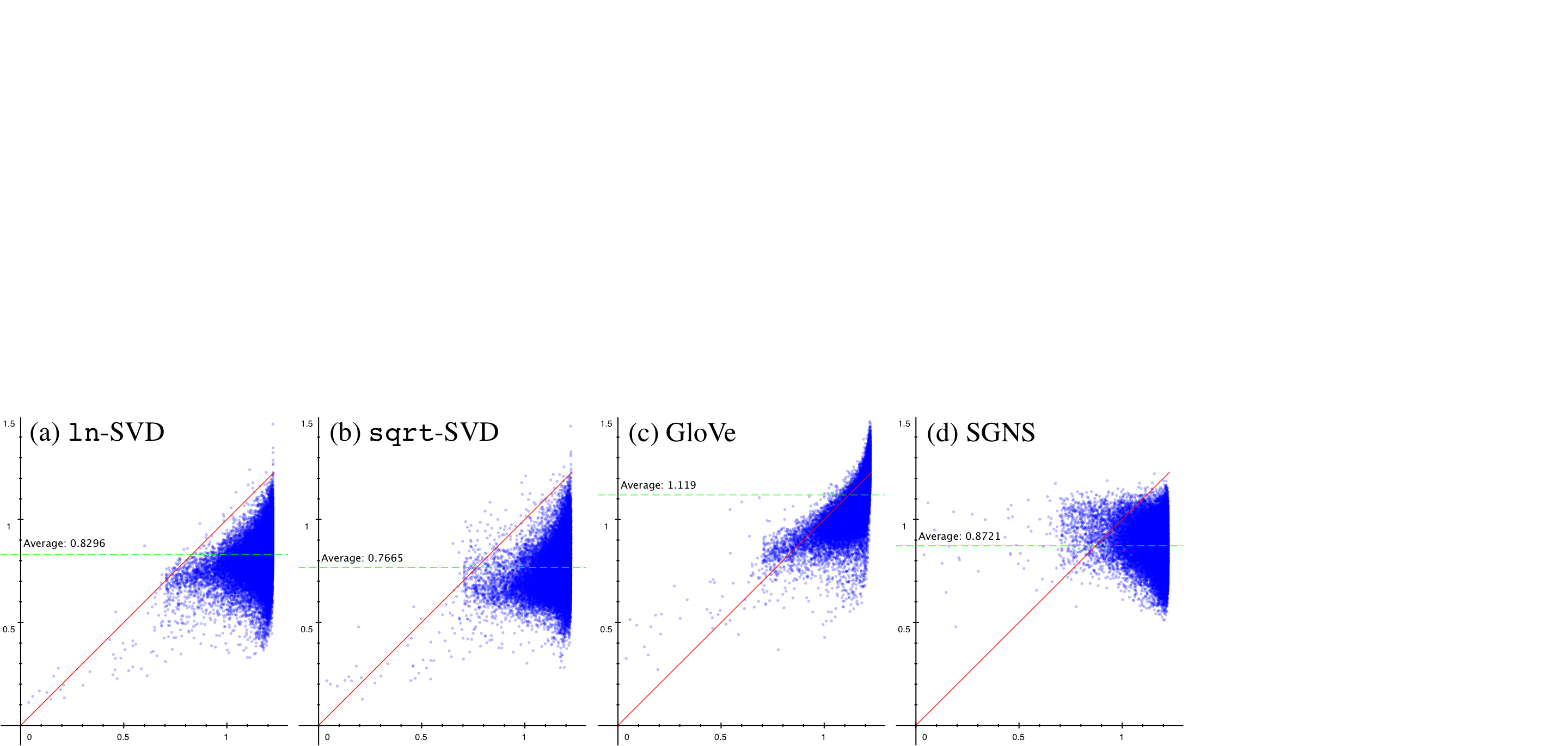}
\caption{Approximation errors for different dimension reduction methods.}
\label{fig:dimred}
\end{figure}

\section{Extrinsic Evaluation of Additive Compositionality}
\label{sec:exteval}

In this section, we test additive composition on human annotated data sets to see if our theoretical predictions
correlate with human judgments. We conduct a phrase similarity task and a word analogy task. 

\subsection{Phrase Similarity}
\label{sec:phrasesim}

In a data set\footnote{\url{http://homepages.inf.ed.ac.uk/s0453356/}} 
created by \citet{mitchell10}, phrase pairs are annotated with similarity scores. 
Each instance in the data is a (\textit{phrase1}, \textit{phrase2}, \textit{similarity})
triplet, and each phrase consists of two words. 
The similarity score is annotated by humans,
ranging from 1 to 7, indicating how similar the meanings of the two phrases are. 
For example, one annotator assessed the similarity between ``\textit{vast amount}'' and
``\textit{large quantity}'' as 7 (the highest),
and the similarity between ``\textit{hear word}'' and ``\textit{remember name}'' as 1 (the lowest). 
Phrases are divided into three categories: Verb-Object, Compound Noun, and Adjective-Noun. 
Each category has 108 phrase pairs, and they are annotated by 18 human participants 
(i.e.,~1,944 instances in each category). 
Using this data set, we can compare the human ranking of phrase similarity with the one calculated from 
cosine similarities between vector-based compositions. We use Spearman's $\rho$ to measure 
how correlated the two rankings are. 

Vector representations are constructed from BNC, with the same settings described in Section~\ref{sec:expveri}. 
We plot in Figure~\ref{fig:phrasestat} the distributions of how many times the phrases in the data set occur 
as bigrams in BNC. 
The figure indicates that a large portion of the phrases are rare or unseen as bigrams, so their 
meanings cannot be directly assessed as natural vectors from the corpus. Therefore, the data is 
suitable for testing compositions of word vectors. 

We reduce the high dimensional distributional representations into $200$-dimensional and normalize the 
vectors. The 
dimension $200$ is selected by observing the top 800 
singular values calculated by SVD. As illustrated in Figure~\ref{fig:sing}, the decrease of singular values 
flattens to a constant rate at a rank of about $200$. This suggests that the most characteristic features 
in the vector representations are projected into 200 dimensions. In our preliminary experiments, we have 
confirmed that 200-dimension performs better than 100-dimension, 500-dimension or no dimension reduction. 

\begin{figure}
\centering
\begin{minipage}{.48\textwidth}
  \centering
  \includegraphics[scale=0.46,bb=0 0 380 170,clip]{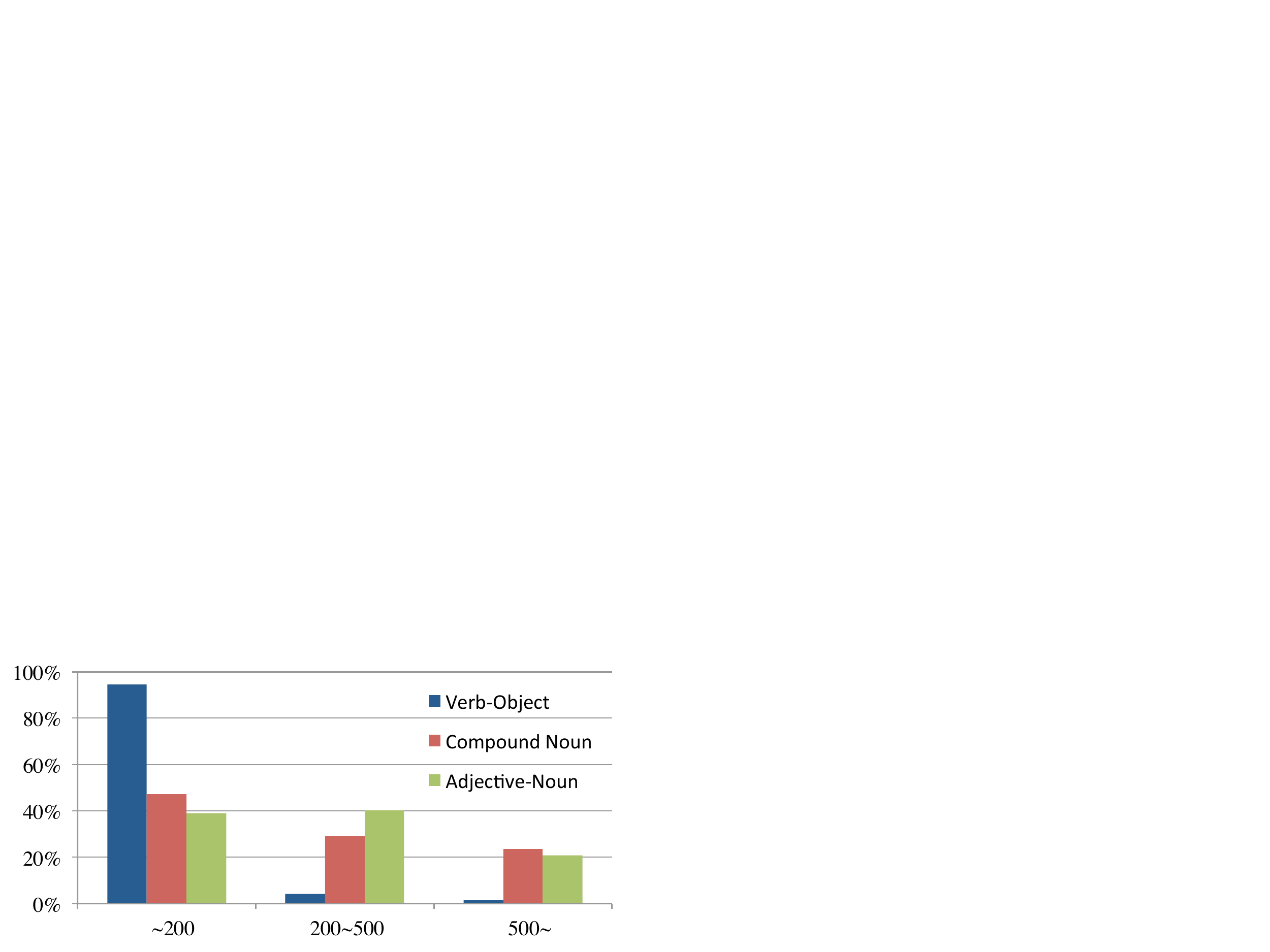}
  \captionof{figure}{Distributions of how many times the phrases in the data occur as 
  bigrams in BNC. The $y$-axis shows percentage and $x$-axis shows frequency range.}
  \label{fig:phrasestat}
\end{minipage}%
\begin{minipage}{.04\textwidth}
  ~
\end{minipage}%
\begin{minipage}{.48\textwidth}
  \centering
  \includegraphics[scale=0.46,bb=0 0 380 170,clip]{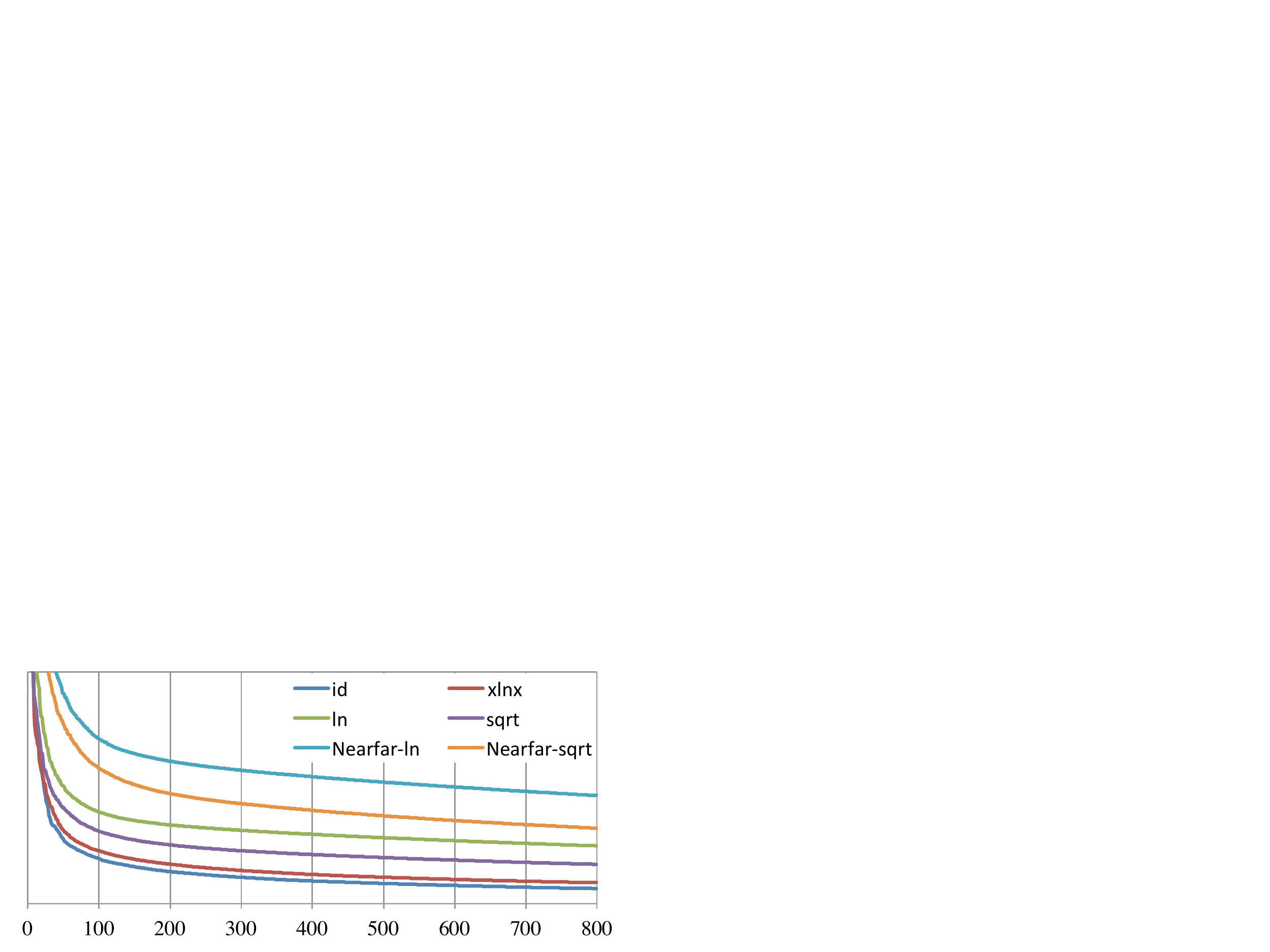}
  \captionof{figure}{Top 800 singular values calculated by SVD. The $y$-axis 
  shows singular value and $x$-axis shows rank. Different $y$-scales are used for different settings.}
  \label{fig:sing}
\end{minipage}
\end{figure}

For training word embeddings, we use the random projection algorithm~\citep{svd} for SVD, 
and Stochastic Gradient Descent (SGD)~\citep{bottou12} 
for SGNS and GloVe. Since these are randomized algorithms, we run each test $20$ 
times and report the mean performance with standard deviation. We tune SGD learning rates 
by checking convergence of the objectives, and get slightly better results 
than the default training parameters set in the software of SGNS\footnote{\url{https://code.google.com/p/word2vec/}} 
and GloVe\footnote{\url{http://nlp.stanford.edu/projects/glove/}}. 

As pointed out by \citet{levyTACL}, there are other detailed settings that can vary in SGNS and GloVe. 
We make these settings close enough to be comparable but emphasize the differences of loss 
functions. More precisely, 
we use no subsampling and set the number of negative samples to 2 in SGNS, and use the default loss 
function in GloVe with the cutoff threshold set to $10$. In addition, the default implementations of both 
SGNS and GloVe 
weigh context words by a function of distance to their targets, which we disable 
(i.e.~equal weights are used for all context words), so as to make it 
compatible with our problem setting.

The test results are shown in Table~\ref{tab:composition}. We compare 
different settings of function $F$, Ordinary and Near-far Contexts, and different dimension reductions. 
When using ordinary contexts and SVD reduction, we find that the functions \texttt{ln} ($F(p):=\ln{p}$) 
and \texttt{sqrt} ($F(p):=\sqrt{p}$) perform similarly well, whereas \texttt{id} ($F(p):=p$) and \texttt{xlnx} 
($F(p):=p\ln{p}$) are much worse, confirming our predictions in Section~\ref{sec:functionF}. 
As for Near-far Context vectors (Section~\ref{sec:wordorder}), we find that the Nearfar-\texttt{sqrt}-SVD setting 
has a high performance, 
demonstrating improvements brought by Near-far to additive composition. However, we note that 
Nearfar-\texttt{ln}-SVD is worse. One reason could be that the function \texttt{ln} emphasizes 
lower co-occurrence probabilities, which combined with Near-far labels could make the vectors more prone 
to data sparseness; or correlatively, 
some important syntactic markers might be obscured because 
they occur in high frequency. Finally, 
we note that SVD is consistently good and usually better than GloVe and SGNS, which supports our 
arguments in Section~\ref{sec:dimensionreduction}. 

We report some additional test results for reference. In Table~\ref{tab:composition}, 
the ``Tensor Product'' row shows the results of composing Ordinary-\texttt{ln}-SVD vectors by 
tensor product instead of average, which means that the similarity between two phrases 
``$s_1$ $t_1$'' and 
``$s_2$ $t_2$'' is assessed by taking product of the word cosine similarities 
$\cos(s_1,s_2)\cdot\cos(t_1,t_2)$. The numbers are worse than additive composition, suggesting that 
a similar phrase might be something more than a sequence of individually similar words. 
In the ``Upper Bound'' row, we show the best possible Spearman's $\rho$ for this task, which are less than 
$1$ because there are disagreements between human annotators. Compared to these numbers, we find 
that the performance of additive composition on compound nouns is remarkably high. 
Furthermore, in ``Muraoka et al.'' we cite the best results reported in \citet{muraoka14}, which has tested 
several compositional frameworks. In ``Deep Neural'', we also test 
additive composition of word vectors trained by deep neural networks 
(normalized $200$-dimensional vectors trained by \citealt{turian10}, using the model of \citealt{cwmodel}). 
These results cannot be directly compared to each other because they construct vector representations
from different corpora; but we can fairly say that additive composition is still a powerful method 
for assessing phrase similarity, and linear dimension reductions might be more suitable for training additively 
compositional word vectors than deep neural networks. Therefore, our theory on additive composition is 
about the state-of-the-art. 

\begin{table}[t]
\centering
\renewcommand{\arraystretch}{1.2}
\setlength{\tabcolsep}{12pt}
\begin{tabular}{| c | c | c | c | }
\hline
 & Verb-Object & Compound Noun & Adjective-Noun \\
\hline
Ordinary-\texttt{id}-SVD & $.4029 \pm .0009$ & $.4275 \pm .0009$ & $.4160 \pm .0009$ \\
Ordinary-\texttt{xlnx}-SVD & $.4204 \pm .0011$ & $.4728 \pm .0013$ & $.4511 \pm .0012$ \\
Ordinary-\texttt{ln}-SVD & $\mathbf{.4369 \pm .0022}$ & $\mathbf{.5187 \pm .0016}$ & $.4604 \pm .0033$ \\
Ordinary-\texttt{sqrt}-SVD & $.4318 \pm .0019$ & $.5051 \pm .0020$ & $.4790 \pm .0018$ \\
Nearfar-\texttt{ln}-SVD & $.4204 \pm .0018$ & $.5135 \pm .0020$ & $.4491 \pm .0028$ \\
Nearfar-\texttt{sqrt}-SVD & $\mathbf{.4359 \pm .0020}$ & $\mathbf{.5193 \pm .0024}$ & $.4873 \pm .0011$ \\
SGNS & $.4273 \pm .0035$ & $.4977 \pm .0025$ & $\mathbf{.5125 \pm .0032}$ \\
GloVe & $.4014 \pm .0046$ & $.4986 \pm .0053$ & $.4308 \pm .0062$ \\
\hline
Tensor Product & $.4092 \pm .0033$ & $.4801 \pm .0035$ & $.4348 \pm .0048$ \\
Upper Bound & $.691$ & $.693$ & $.715$ \\
Muraoka et al. & $.430$ & $.481$ & $.469$ \\
Deep Neural & $.305$ & $.385$ & $.207$ \\
\hline
\end{tabular}
\caption{Spearman's $\rho$ in the phrase similarity task.}
\label{tab:composition}
\end{table}

\subsection{Word Analogy}
\label{sec:wordanalogy}

Word analogy is the task of solving questions of the form ``\textit{a} is to \textit{b} as \textit{c} is to \_\_?'', and 
an elegant approach proposed by \citet{word2vecNAACL} is to find the word vector most similar to 
$\mathbf{v}^b - \mathbf{v}^a + \mathbf{v}^c$ . For example, in order to answer the 
question ``\textit{man} is to \textit{king} as \textit{woman} is to \_\_?'', one needs to calculate 
$\mathbf{v}^{\text{king}} - \mathbf{v}^{\text{man}} + \mathbf{v}^{\text{woman}}$ and find out its most similar 
word vector, which will probably turn out to be $\mathbf{v}^{\text{queen}}$, indicating the correct answer 
\textit{queen}. 

As pointed out by \citet{levyCoNLL}, the key to solving analogy questions is the ability to ``add'' 
(resp.~``subtract'') some aspects to (resp.~from) a concept. For example, \textit{king} is a concept of \textit{human} 
that has the aspects of being \textit{royal} and \textit{male}. If we can ``subtract'' the aspect \textit{male} 
from \textit{king} and ``add'' the aspect \textit{female} to it, then we will probably get the concept 
\textit{queen}. Thus, the vector-based solution proposed by \citet{word2vecNAACL} is 
essentially assuming that ``adding'' and ``subtracting'' aspects can be realized by adding and subtracting 
word vectors. Why is this assumption admissible?

We believe this assumption is closely related to additive compositionality. Because, if an aspect is 
represented by an adjective (e.g.~\textit{male}) and a concept is represented by a noun 
(e.g.~\textit{human}), we can usually ``add'' the aspect to the concept by simply arranging the 
adjective and the noun to form a phrase (e.g.~\textit{male human}). Therefore, as the meaning of the phrase 
can be calculated by additive composition (e.g.~$\mathbf{v}^{\text{male}}+\mathbf{v}^{\text{human}}$), 
we have indeed realized the ``addition'' of aspects by addition of word vectors. Specifically, 
since $\textit{man}\approx\textit{male human}$, $\textit{king}\approx\textit{royal male human}$, 
$\textit{woman}\approx\textit{female human}$ and $\textit{queen}\approx\textit{royal female human}$, we 
expect the following by additive composition of phrases. 
\[\begin{split}
\mathbf{v}^{\text{man}} & \approx  \mathbf{v}^{\text{male}} + \mathbf{v}^{\text{human}} \\
\mathbf{v}^{\text{king}} & \approx \mathbf{v}^{\text{royal}} + \mathbf{v}^{\text{male}} + \mathbf{v}^{\text{human}} \\
\mathbf{v}^{\text{woman}} & \approx \mathbf{v}^{\text{female}} + \mathbf{v}^{\text{human}} \\
\mathbf{v}^{\text{queen}} & \approx \mathbf{v}^{\text{royal}} + \mathbf{v}^{\text{female}} + \mathbf{v}^{\text{human}} \\
\end{split}\]
Here, ``$\approx$'' denotes proximity between vectors in the sense of cosine similarity. From these approximate 
equations, we can imply that 
$
\mathbf{v}^{\text{king}} - \mathbf{v}^{\text{man}} + \mathbf{v}^{\text{woman}} \approx 
\mathbf{v}^{\text{royal}} + \mathbf{v}^{\text{female}} + \mathbf{v}^{\text{human}} \approx \mathbf{v}^{\text{queen}} 
$, which solves the analogy question. 

Therefore, we expect word analogy task to serve as an extrinsic evaluation for additive 
compositionality as well. For this reason, we conduct word analogy task on the standard 
\textsc{Msr}\footnote{\url{http://research.microsoft.com/en-us/projects/rnn/}}~\citep{word2vecNAACL} and 
\textsc{Google}\footnote{\url{https://code.google.com/p/word2vec/}}~\citep{word2vecNIPS} data sets. 
Each instance in the data is a $4$-tuple of words subject to 
``\textit{a} is to \textit{b} as \textit{c} is to \textit{d}'', and the task is to find out \textit{d} from 
\textit{a}, \textit{b} and \textit{c}. 
We train word vectors with the same settings described in Section~\ref{sec:expveri}, but using surface forms 
instead of lemmatized words in BNC. Tuples with out-of-vocabulary words are removed from data, 
which results in 4382 tuples in \textsc{Msr} and 8906 in 
\textsc{Google}\footnote{These are about half the size of the original data sets.}. 

\begin{table}[t]
\centering
\renewcommand{\arraystretch}{1.2}
\setlength{\tabcolsep}{5pt}
\begin{tabular}{ c | c  c  c  c  c  c }
\hline
 & \texttt{id}-SVD & \texttt{xlnx}-SVD & \texttt{ln}-SVD & \texttt{sqrt}-SVD & SGNS & GloVe \\
\hline
\textsc{Google} & $19.43\pm .06$ & $32.47\pm .10$ & $\mathbf{52.04\pm .36}$ & $51.28\pm .25$ & $45.16\pm .44$ & $48.39\pm .44$ \\
\textsc{Msr} & $17.36\pm .06$ & $33.85\pm .12$ & $\mathbf{66.67\pm .26}$ & $60.93\pm .25$ & $55.56\pm .30$ & $65.05\pm .55$ \\
\hline
\end{tabular}
\caption{Accuracy (\%) in the word analogy task.}
\label{tab:analogy}
\end{table}

The test results are shown in Table~\ref{tab:analogy}. Again, we find that \texttt{ln} and \texttt{sqrt} perform similarly 
well but \texttt{id} and \texttt{xlnx} are worse, confirming that the choice of function $F$ can drastically affect 
performance on word analogy task as well, which we believe is related to additive compositionality. In addition, 
we confirm that SVD can perform better than SGNS and GloVe, which gives more support to our conjecture that 
vectors trained by SVD might be more compatible to additive composition. 

\section{Conclusion}
\label{sec:conclusion}

In this article, we have developed a theory of additive composition regarding its bias. 
The theory has explained why and how additive composition works, making useful suggestions 
about improving additive compositionality, which include the choice of a transformation function, 
the awareness of word order, and the dimension reduction methods. Predictions made by our theory 
have been verified experimentally, and shown positive correlations with human judgments. In short, 
we have revealed the mechanism of additive composition. 

However, we note that our theory is not ``proof'' of additive composition being a ``good'' compositional 
framework. As a generalization error bound usually is in machine learning theory, our bound for the bias does not 
show if additive composition is ``good''; rather, it specifies some factors that can affect the errors. If 
we have generalization error bounds for other composition operations, a comparison between such bounds 
can bring 
useful insights into the choices of compositional frameworks in specific cases. 
We expect our bias bound to inspire more results
in the research of semantic composition. 

Moreover, we believe this line of theoretical research can be pursued further. 
In computational linguistics, the idea of treating semantics and semantic relations by algebraic operations on 
distributional context vectors is relatively new~\citep{Clarke:2012}. Therefore, the relation between 
linguistic theories and 
our approximation theory of semantic composition is left largely unexplored. For example, the intuitive 
distinction between compositional (e.g.~\textit{high price}) and non-compositional (e.g.~\textit{white lie}) 
phrases is currently ignored in our theory. Our bias bound treats both cases by a single 
collocation measure. Can one improve the bound by taking account of this distinction, and/or other kinds 
of linguistic knowledge? This is an intriguing question for future work. 

\newpage

\appendix


\section{Proof of Lemmas}
\label{app:fullproof}

In this appendix, we prove Lemma~\ref{lem:ycalc} and Lemma~\ref{lem:phicalc} in Section~\ref{sec:formalization}. 

In order to prove Lemma~\ref{lem:ycalc}, we first note the following equations: 
\begin{align}
\frac{F\bigl(x+(np_{i,n})^{-1}\bigr)-F\bigl(\beta+(np_{i,n})^{-1}\bigr)}{\bigl(1+(\beta np_{i,n})^{-1}\bigr)^{-1+\lambda}} & \geq F(x)-F(\beta) 
\label{eq:aux1}
\\
\frac{F\bigl(x+(np_{i,n})^{-1}\bigr)-F\bigl(\beta+(np_{i,n})^{-1}\bigr)}{\bigl(1+(\beta np_{i,n})^{-1}\bigr)^{\lambda}} & = F\Bigl(\frac{x+(np_{i,n})^{-1}}{1+(\beta np_{i,n})^{-1}}\Bigr)-F(\beta) 
\label{eq:aux2}
\end{align}
Equation~\eqref{eq:aux1} can be obtained by analyzing the derivatives $F'(x)=x^{-1+\lambda}$, and 
Equation~\eqref{eq:aux2} immediately follows from the identity $z^\lambda(F(x)-F(y))=F(zx)-F(zy)$. 

\begin{proof}[Proof of Lemma~\ref{lem:ycalc}(a)(b)]
We calculate $e_{i,n}$ as follows. By definition, 
$$
e_{i,n}=\mathbb{E}\Bigl[ \frac{Y_{i,n}}{\sqrt{\varphi_{i,n}}}I_{X\leq\beta} \Bigr]
+\mathbb{E}\Bigl[ \frac{Y_{i,n}}{\sqrt{\varphi_{i,n}}}I_{X\geq\beta} \Bigr].
$$
Then, note that $Y_{i,n}\leq 0$ when $X\leq\beta$, and by \eqref{eq:aux1} we have 
\begin{equation}
\label{eq:last1}
0\geq\mathbb{E}\Bigl[ \frac{Y_{i,n}}{\sqrt{\varphi_{i,n}}}I_{X\leq\beta} \Bigr]\geq 
\frac{\mathbb{E}\bigl[\bigl(F(X)-F(\beta)\bigr)I_{X\leq\beta}\bigr]}{\sqrt{1+(\beta np_{i,n})^{-1}}}. 
\end{equation}
From the condition $\mathbb{E}[F(X)^2]<\infty$ we have $\mathbb{E}[\lvert F(X)\rvert ]<\infty$, so
\begin{equation}
\label{eq:res1}
\eqref{eq:last1}\geq -\mathbb{E}[\lvert F(X)\rvert ]-\lvert F(\beta) \rvert\quad\text{and}\quad
\eqref{eq:last1}\rightarrow 0 \text{ when } np_{i,n}\rightarrow 0. 
\end{equation}
Next, when $X\geq\beta$ we have $Y_{i,n}\geq 0$, and by \eqref{eq:aux2} we get 
\begin{equation}
\label{eq:last2}
0\leq\mathbb{E}\Bigl[ \frac{Y_{i,n}}{\sqrt{\varphi_{i,n}}}I_{X\geq\beta} \Bigr]= 
\sqrt{1+(\beta np_{i,n})^{-1}}\cdot\mathbb{E}\left[\left(F\Bigl(\frac{X+(np_{i,n})^{-1}}{1+(\beta np_{i,n})^{-1}}\Bigr)-F(\beta)\right)I_{X\geq\beta}\right]. 
\end{equation}
By Assumption (B2), $X$ obeys a power law at $X\geq\beta$, so if we put $Z:=X-\beta$, the probability 
density of $Z$ is given by 
$$
-\D\mathbb{P}(z\leq Z)=\frac{\xi\D z}{(z+\beta)^2}\leq\frac{\xi\D z}{\sqrt{z}(z+\beta)^{3/2}}
\quad\text{(where $z> 0$)}. 
$$
Thus, 
\begin{equation}
\label{eq:last3}
\eqref{eq:last2}\leq \int_0^\infty \frac{F\Bigl(\dfrac{z}{1+(\beta np_{i,n})^{-1}}+\beta\Bigr)-F(\beta)}{\sqrt{\dfrac{z}{1+(\beta np_{i,n})^{-1}}}}\cdot\frac{\xi\D z}{(z+\beta)^{3/2}}. 
\end{equation}
The condition $\mathbb{E}[F(X)^2]<\infty$ implies $\lambda<0.5$, so 
$\dfrac{F(u+\beta)-F(\beta)}{\sqrt{u}}\rightarrow 0$ at $u\rightarrow\infty$. In addition the 
function is smooth on $[0,\infty)$, so it can be bounded by a constant $M$. 
Therefore, 
\begin{equation}
\label{eq:res2}
\eqref{eq:last3}\leq\int_0^\infty M\cdot\frac{\xi\D z}{(z+\beta)^{3/2}}<\infty \quad\text{and}\quad 
\eqref{eq:last3}\rightarrow 0 \text{ when } np_{i,n}\rightarrow 0. 
\end{equation}
The limit above is a consequence of Lebesgue's Dominated Convergence Theorem. 

Combining \eqref{eq:res1} and \eqref{eq:res2}, we have proven Lemma~\ref{lem:ycalc}(a)(b). 
\end{proof}

\begin{proof}[Proof of Lemma~\ref{lem:ycalc}(c)(d)]
We calculate $\mathbb{E}[Y_{i,n}^2/\varphi_{i,n}]$ as follows. By definition, 
$$
\mathbb{E}[Y_{i,n}^2/\varphi_{i,n}]=\mathbb{E}\Bigl[ \frac{Y_{i,n}^2}{\varphi_{i,n}}I_{X\leq\beta} \Bigr]
+\mathbb{E}\Bigl[ \frac{Y_{i,n}^2}{\varphi_{i,n}}I_{X\geq\beta} \Bigr].
$$
By \eqref{eq:aux1} we have 
\begin{equation}
\label{eq:last4}
\mathbb{E}\Bigl[ \frac{Y_{i,n}^2}{\varphi_{i,n}}I_{X\leq\beta} \Bigr]\leq 
\frac{\mathbb{E}\bigl[\bigl(F(X)-F(\beta)\bigr)^2I_{X\leq\beta}\bigr]}{1+(\beta np_{i,n})^{-1}}. 
\end{equation}
Then, from the condition $\mathbb{E}[F(X)^2]<\infty$, we have
\begin{equation}
\label{eq:res3}
\eqref{eq:last4} \text{ is uniformly integrable}\quad\text{and}\quad
\eqref{eq:last4}\rightarrow 0 \text{ when } np_{i,n}\rightarrow 0. 
\end{equation}
Next, by \eqref{eq:aux2} we get 
\begin{equation}
\label{eq:last5}
\mathbb{E}\Bigl[ \frac{Y_{i,n}^2}{\varphi_{i,n}}I_{X\geq\beta} \Bigr]= 
\bigl(1+(\beta np_{i,n})^{-1}\bigr)\cdot\mathbb{E}\left[\left(F\Bigl(\frac{X+(np_{i,n})^{-1}}{1+(\beta np_{i,n})^{-1}}\Bigr)-F(\beta)\right)^2I_{X\geq\beta}\right]. 
\end{equation}
By Assumption (B2), $X$ obeys a power law at $X\geq\beta$, 
so if we put $Z:=\dfrac{X-\beta}{1+(\beta np_{i,n})^{-1}}$, 
the probability density of $Z$ is given by 
\[
\begin{split}
-\D\mathbb{P}(z\leq Z) & =\frac{1}{1+(\beta np_{i,n})^{-1}}\cdot 
\frac{\xi\D z}{\bigl(z+\dfrac{\beta}{1+(\beta np_{i,n})^{-1}}\bigr)^2} \\
& \leq \frac{1}{1+(\beta np_{i,n})^{-1}}\cdot\frac{\xi\D z}{z^2}
\end{split}
\quad\text{(where $z> 0$)}. 
\]
Thus, 
$$
\eqref{eq:last5}\leq \int_0^\infty\bigl(F(z+\beta)-F(\beta)\bigr)^2\cdot\frac{\xi\D z}{z^2}. 
$$
The condition $\mathbb{E}[F(X)^2]<\infty$ implies $\lambda<0.5$, so the above integral 
is finite. The integral is independent of $i$ and $n$, so \eqref{eq:last5} is uniformly integrable; 
in addition, when $np_{i,n}\rightarrow 0$ we have 
$\bigl(z+\dfrac{\beta}{1+(\beta np_{i,n})^{-1}}\bigr)^2\rightarrow z^2$, so 
\begin{equation}
\label{eq:res4}
\lim_{np_{i,n}\rightarrow 0}\eqref{eq:last5}=\int_0^\infty\bigl(F(z+\beta)-F(\beta)\bigr)^2\cdot\frac{\xi\D z}{z^2} 
\end{equation}
by Lebesgue's Dominated Convergence Theorem. 

Combining \eqref{eq:res3} and \eqref{eq:res4}, we have proven Lemma~\ref{lem:ycalc}(c)(d). 
\end{proof}

\begin{proof}[Proof of Lemma~\ref{lem:phicalc}]
Firstly, by the definition of $\varphi_{i,n}$ we have 
$$
n^{-1+2\lambda}\cdot\varphi_{i,n}=\bigl(np_{i,n}+1/\beta\bigr)^{-1+2\lambda}\cdot p_{i,n}. 
$$
So Lemma~\ref{lem:phicalc}(a) immediately follows. 
To prove Lemma~\ref{lem:phicalc}(b), we simply apply Assumption (A) to the above. 

To prove Lemma~\ref{lem:phicalc}(c), we calculate 
\[\begin{split}
n^{-1+2\lambda}\ln n\sum_{i=1}^{\frac{n}{\delta\ln n}}\varphi_{i,n} & =\sum_{i=1}^{\frac{n}{\delta\ln n}}\bigl(np_{i,n}+1/\beta\bigr)^{-1+2\lambda}\cdot p_{i,n}\ln n \\
& \leq \sum_{i=1}^{\frac{n}{\delta\ln n}}(np_{i,n})^{-1+2\lambda}\cdot p_{i,n}\ln n \\
& \approx \bigl(\frac{n}{\ln n}\bigr)^{-1+2\lambda}\sum_{i=1}^{\frac{n}{\delta\ln n}} i^{-2\lambda} 
  \quad\quad\quad\text{(by Assumption (A))} \\[4pt]
& \approx \bigl(\frac{n}{\ln n}\bigr)^{-1+2\lambda}\cdot\frac{1}{1-2\lambda}\bigl(\frac{n}{\delta\ln n}\bigr)^{1-2\lambda} \\[10pt]
& =\frac{\delta^{-1+2\lambda}}{1-2\lambda}, 
\end{split}\]
so the required $M_\delta$ exists. 

To prove Lemma~\ref{lem:phicalc}(d), we note that by Equation \eqref{eq:indexnp}, 
$$
\frac{n}{\delta\ln n}\leq i\leq n \Rightarrow np_{i,n}\leq\delta. 
$$
Hence, 
\[\begin{split}
n^{-1+2\lambda}\ln n\sum_{\frac{n}{\delta\ln n}\leq i}^{n}\varphi_{i,n} & =\sum_{\frac{n}{\delta\ln n}\leq i}^{n}\bigl(np_{i,n}+1/\beta\bigr)^{-1+2\lambda}\cdot p_{i,n}\ln n \\
& \geq \sum_{\frac{n}{\delta\ln n}\leq i}^{n}(\delta+1/\beta)^{-1+2\lambda}\cdot p_{i,n}\ln n \\
& \approx (\delta+1/\beta)^{-1+2\lambda} \sum_{\frac{n}{\delta\ln n}\leq i}^{n} i^{-1} 
  \quad\quad\quad\text{(by Assumption (A))} \\[8pt]
& \approx (\delta+1/\beta)^{-1+2\lambda}\ln\bigl(\delta\ln n\bigr)\rightarrow\infty 
\quad\text{(when $n\rightarrow\infty$)}.
\end{split}\]
Therefore, Lemma~\ref{lem:phicalc}(d) is proven. 
\end{proof}


\section{The Loss Function of SGNS}
\label{app:sgns}

In this appendix, we discuss the loss function of SGNS. The model 
is originally proposed as an ad hoc objective function using the negative sampling technique~\citep{word2vecNIPS}, 
without any explicit explanation on what is optimized and what is the loss. It is later 
shown that SGNS is a factorization of the shifted-PMI matrix~\citep{levyNIPS}, but the loss function for this 
factorization remains unspecified. Here, we give a re-explanation of the SGNS model, with the loss function 
explicitly stated. 

\subsection{Noise Contrastive Estimation}

The original objective function of SGNS is proposed as an adaptation of the Noise Contrastive 
Estimation (NCE) method, but in fact SGNS \emph{is} using NCE without any adaptation. 
NCE~\citep{gutmann12} is a method for solving the classical problem that, given a sample $(x_i)_{i=1}^N$ 
(wherein $x_i\in\mathcal{X}$) drawn from an \emph{unknown} probability distribution 
$\mathbb{P}_{\text{data}}$, and 
a function family $f(\cdot;\theta):\mathcal{X}\rightarrow\mathbb{R}_{\geq 0}$ parameterized by $\theta$, 
to find the optimal $\theta^{*}$ such that $f(x;\theta^{*})$ approximates the distribution 
$\mathbb{P}_{\text{data}}(x)$ best. An alternative to NCE is the Maximum Likelihood Estimation (MLE), 
in which $\theta^{*}$ is 
chosen as to maximize the log-likelihood of the sample $(x_i)_{i=1}^N$, with respect to the constraint that 
$f(\cdot;\theta^{*})$ should be a probability:
\[
\theta^{*}_{\text{MLE}}=\argmax_{\theta}\sum_{i=1}^N\ln f(x_i;\theta),
\quad\text{where }\sum_{x\in\mathcal{X}}f(x;\theta)=1.
\]
For MLE, the constraint $\sum_{x\in\mathcal{X}}f(x;\theta)=1$ is important, because 
$f(x;\theta)$ can tend to arbitrarily large if we maximize the log-likelihood without the constraint. NCE finds 
$\theta^{*}$ in a different way. It firstly mixes $(x_i)$ with a noise sample drawn from a \emph{known} distribution $\mathbb{P}_{\text{noise}}$, each 
data point $x_i$ mixed with $k$ noise points $y_{i,1},\ldots,y_{i,k}\sim\mathbb{P}_{\text{noise}}$. Hence, 
\begin{equation}
\label{eq:PxIsData}
\mathbb{P}(x\text{ is data}\;|\;x)=\frac{\mathbb{P}_{\text{data}}(x)}{\mathbb{P}_{\text{data}}(x)+k\mathbb{P}_{\text{noise}}(x)},
\end{equation}
which gives the probability of a given point $x\in\mathcal{X}$ being a data point. $\mathbb{P}_{\text{data}}$ is 
unknown in~\eqref{eq:PxIsData}, so we approximate $\mathbb{P}(x\text{ is data}\;|\;x)$ by $g(x;\theta)$ as below:
\begin{equation}
\label{eq:gApproxP}
g(x;\theta):=\frac{f(x;\theta)}{f(x;\theta)+k\mathbb{P}_{\text{noise}}(x)}. 
\end{equation}
Then, NCE maximizes the log-likelihood of 
``\textit{$x_i$ being data and $y_{i,1},\ldots,y_{i,k}$ being noise}'':
\begin{equation}
\label{eq:NCE}
\theta^{*}_{\text{NCE}}=\argmax_{\theta}\sum_{i=1}^N\Bigl(\ln g(x_i;\theta)
+\sum_{j=1}^k\ln(1-g(y_{i,j};\theta))\Bigr).
\end{equation}
The most important point of NCE is that, $f(x;\theta)$ will \emph{not} tend to infinity even we 
maximize~\eqref{eq:NCE} \emph{without} the constraint $\sum_{x\in\mathcal{X}}f(x;\theta)=1$. 
This is because making 
$f(x;\theta)$ large will accordingly make $1-g(y_{i,j};\theta)$ small, which will \emph{decrease} the likelihood of 
``\textit{$y_{i,1},\ldots,y_{i,k}$ being noise}''. No longer necessary to repeatedly calculate  
$\sum_{x\in\mathcal{X}}f(x;\theta)$ during parameter update, NCE usually results in efficient training 
algorithms. 

\subsection{The Skip-Gram with Negative Sampling Model}
\label{sec:SGNS}


Let $p^t_i$ be the co-occurrence probability of the $i$-th word, conditioned on it being in the 
context of a target word $t$. 
SGNS approximates  $p^t_i$ by the function family 
$$f(i,t;\textbf{u},\textbf{v}):=\exp(\textbf{u}_i\cdot\textbf{v}^t+\ln (kp^{\text{noise}}_i)),$$ 
using NCE to optimize parameters. 
Here, $\textbf{u}$ and $\textbf{v}$ are parameters of the function family, whose columns are vectors 
$\textbf{u}_i$ and $\textbf{v}^t$, corresponding to the $i$-th context word and the word target $t$, respectively. 
The training data $\mathcal{C}$ is a collection of co-occurring context-target word pairs. 
On the other hand, $k$ and $p^{\text{noise}}_i$ are constants in the definition of the function family, 
where $k$ is the number of noise points drawn for each training instance, and 
$p^{\text{noise}}_i$ is the probability value of the $i$-th context word being drawn from the noise 
distribution $\mathbb{P}_{\text{noise}}$. 

Substituting the above $f(i,t;\textbf{u},\textbf{v})$ into \eqref{eq:gApproxP},  we get 
$$
g(i,t;\textbf{u},\textbf{v})=\frac{\exp(\textbf{u}_i\cdot\textbf{v}^t+\ln (kp^{\text{noise}}_i))}{\exp(\textbf{u}_i\cdot\textbf{v}^t+
\ln (kp^{\text{noise}}_i))+kp^{\text{noise}}_i}=\sigma(\mathbf{u}_i\cdot\mathbf{v}^t),
$$
where $\sigma(x)=1/\bigl(1+\exp(-x)\bigr)$ is the sigmoid function. 

Substituting the 
obtained $g(i,t;\textbf{u},\textbf{v})$ into \eqref{eq:NCE}, we get 
$$
(\mathbf{u},\mathbf{v})^*_{\text{NCE}}=
\argmax_{\mathbf{u}, \mathbf{v}}\sum_{(i,t)\in\mathcal{C}}\Bigl(\ln\sigma(\mathbf{u}_i\cdot\mathbf{v}^t)
+\sum_{y\sim \mathbb{P}_{\text{noise}}}\ln(1-\sigma(\textbf{u}_{y}\cdot\textbf{v}^t))\Bigr) 
$$
which is exactly the objective function of SGNS proposed in \cite{word2vecNIPS}. 

Since SGNS is using $f(i,t;\textbf{u},\textbf{v})=\exp(\textbf{u}_i\cdot\textbf{v}^t+\ln (kp^{\text{noise}}_i))$ to 
approximate $p^t_i$, it is using $\textbf{u}_i\cdot\textbf{v}_t$ to approximate 
$w^t_i:=\ln p^t_i-\ln (kp^{\text{noise}}_i)$. This $w^t_i$ is in the form of vector entries of distributional 
representations as given in Definition~\ref{defn:wvec}. Thus, SGNS can be viewed as a 
dimension reduction of the distributional representations we consider in this article. 

\begin{proof}[Proof of Claim~\ref{claim:NCEloss}] 
To calculate its loss function, we consider the objective of SGNS: 
$$
O(\mathbf{u}, \mathbf{v}):=
\sum_{(i,t)\in\mathcal{C}}\Bigl(\ln\sigma(\mathbf{u}_i\cdot\mathbf{v}^t)
+\sum_{y\sim \mathbb{P}_{\text{noise}}}\ln(1-\sigma(\textbf{u}_{y}\cdot\textbf{v}^t))\Bigr). 
$$
The sum is taken across all context-target pairs in $\mathcal{C}$. We regroup the summands by each 
distinct target, and note that conditioned on an occurrence of target $t$, the probability for 
one to encounter the $i$-th context word co-occurring in the training data is $p^t_i$, 
whereas the expected times for one to draw the context word from noise is given by 
$kp^\text{noise}_i$. So we have 
$$
O(\mathbf{u}, \mathbf{v})=\sum_t C(t)\sum_i 
\Bigl(p^t_i \ln\sigma(\mathbf{u}_i\cdot\mathbf{v}^t)
+kp^\text{noise}_i\ln(1-\sigma(\textbf{u}_{n}\cdot\textbf{v}^t))\Bigr) 
$$
where $C(t)$ is the occurrence count of $t$. 
Now, we know that the optimal $O(\mathbf{u}, \mathbf{v})$ is achieved at 
$\textbf{u}_i\cdot\textbf{v}^{t}=w^t_i$, so we define 
$$
M:=\sum_t C(t)\sum_i 
\Bigl(p^t_i \ln\sigma(w^t_i)
+kp^\text{noise}_i\ln(1-\sigma(w^t_i))\Bigr). 
$$
Then, to maximize $O(\mathbf{u}, \mathbf{v})$ is to minimize $M-O(\mathbf{u}, \mathbf{v})$, 
and by some calculation we obtain 
$$
M-O(\mathbf{u}, \mathbf{v})=\sum_t C(t)\sum_i 
D_{\phi}\bigl(\textbf{u}_{i}\cdot\textbf{v}^{t} + \ln (kp^{\text{noise}}_i),\, w^t_i + \ln (kp^{\text{noise}}_i)\bigr), 
$$
where $D_{\phi}(p,\,q):=\phi(p)-\phi(q)-\phi'(q)(p-q)$ is the Bregman divergence associated with the 
convex function 
$$
\phi(x)=(p^t_i+kp^{\text{noise}}_i)\ln(\exp(x)+kp^{\text{noise}}_i).
$$
This is the loss function given in Claim~\ref{claim:NCEloss}. 
The limit of $D_{\phi}$ at $k\rightarrow+\infty$ is easily derived. 
\end{proof}


\bibliographystyle{spbasic}      
\bibliography{ref.bib}   

%
%

\end{document}